\documentclass{article} % For LaTeX2e
\usepackage{iclr2027_conference,times}
\usepackage{hyperref}
\usepackage{url}
\usepackage{censor}

\usepackage[dvipsnames]{xcolor}
\usepackage{colortbl}

\usepackage{graphicx,wrapfig}
\usepackage{longtable,booktabs,nicematrix,makecell,multirow}
\usepackage{caption,subcaption}
\usepackage{float}
\usepackage[export]{adjustbox}

\usepackage{amsmath,amssymb,amsthm,mathtools}
\usepackage{nicefrac}
\usepackage{microtype}

\theoremstyle{plain}
\newtheorem{theorem}{Theorem}
\newtheorem{proposition}[theorem]{Proposition}
\newtheorem{lemma}[theorem]{Lemma}
\newtheorem{corollary}[theorem]{Corollary}
\newtheorem{assumption}[theorem]{Assumption}

\usepackage{xstring}
\usepackage[latin1]{inputenc}
\usepackage[T1]{fontenc}
\usepackage[normalem]{ulem}

\usepackage{hyperref}
\usepackage{url}
\usepackage{cleveref}
\usepackage{etoolbox}

\crefname{assumption}{assumption}{assumptions}
\Crefname{assumption}{Assumption}{Assumptions}
\crefname{section}{\S}{\S}

\pretocmd{\appendix}{%
\crefalias{section}{appendix}
\crefalias{subsection}{appendix}
\crefalias{subsubsection}{appendix}
\counterwithin{equation}{section}

\counterwithin{figure}{section}

\counterwithin{table}{section}%

}{}{}

\usepackage{listings,verbatim}
\usepackage[most]{tcolorbox}
\usepackage{fancyvrb,fvextra}
\usepackage{tikz}

\tcbuselibrary{listings}
\fvset{fontsize=\scriptsize}
\newtcolorbox{rqbox}{
  colback=Gray!5!white,
  colframe=Gray!80!black,
  size=small,
  boxed title style={empty},
  title=,
}
\newtcolorbox{samplebox}[1]{
  colback=RoyalBlue!5!white,
  colframe=RoyalBlue!75!black,
  title=#1,
  size=small,
  before upper={\fvset{fontsize=\tiny}},
}
\newtcolorbox{tmplbox}[1]{
  colback=RedOrange!5!white,
  colframe=RedOrange!80!black,
  title=#1,
  size=small,
}
\newtcolorbox{promptbox}[1]{
  colback=LimeGreen!5!white,
  colframe=LimeGreen!80!black,
  title=#1,
  size=small,
}
\newtcolorbox[auto counter]{takeawaybox}{
  colback=Green!5!white,
  colframe=Green!80!black,
  title=Takeaways~\thetcbcounter,
  size=small,
}
\newtcolorbox{generationbox}[1]{
  breakable,
  colback=gray!5!white,
  colframe=gray!80!black,
  title=#1,
  size=small,
  before upper={\fvset{fontsize=\tiny}},
}

\definecolor{gold}{RGB}{255, 215, 0}
\definecolor{silver}{RGB}{192, 192, 192}
\definecolor{bronze}{RGB}{205, 127, 50}

\newcommand{\best}[1]{\textbf{#1}}

\newcommand{\gold}[1]{\cellcolor{RoyalBlue!8}\textbf{#1}}
\newcommand{\silver}[1]{\cellcolor{RoyalBlue!4}\underline{#1}}
\newcommand{\bronze}[1]{\cellcolor{RoyalBlue!2}\uwave{#1}}

\newcommand{\norm}[1]{\left\lVert#1\right\rVert}

\newcommand{\dns}[2]{#2\makebox[0em][l]{\color{gray}\tiny $\downarrow$}}
\newcommand{\dn}[2]{\cellcolor{Red!0}#2\makebox[0em][l]{\color{Red}\tiny$\downarrow$}}
\newcommand{\dnl}[2]{\cellcolor{Red!0}#2\makebox[0em][l]{\color{Red}${\downarrow}$}}
\newcommand{\df}[2]{%
\IfBeginWith{#1}{-0.}{\dns{#1}{#2}}{%
\IfBeginWith{#1}{-1.}{\dn{#1}{#2}}{%
\IfBeginWith{#1}{-}{\dnl{#1}{#2}}{% anything that is negative
\IfBeginWith{#1}{0.}{#2}{%
\IfBeginWith{#1}{1.}{#2}{%
\IfBeginWith{#1}{2.}{#2}{%
\IfBeginWith{#1}{3.}{#2}{%
\IfBeginWith{#1}{4.}{#2}{%
#2%
}}}}}}}}}
\newcommand{\dfcell}[1]{
\IfBeginWith{#1}{-}{\cellcolor{Red!2}\color{Red}$\downarrow\StrGobbleLeft{#1}{1}$}{%
\cellcolor{RoyalBlue!2}\color{RoyalBlue}$\uparrow#1$
}}
\newcommand{\upshow}[2]{%
\IfStrEq{#1}{}{#2}{%
#2{\color{RoyalBlue}\tiny\makebox[2em][l]{$\uparrow$#1}}}
}
\newcommand{\pct}[2]{%
\IfBeginWith{#2}{-}{#1{\color{Red}\tiny\makebox[0em][l]{$\downarrow$\StrGobbleLeft{#2}{1}\%}}}{%
#1{\color{RoyalBlue}\tiny\makebox[0em][l]{$\uparrow$#2\%}}
}}

\DeclareMathOperator{\IoU}{IoU}
\newcommand{\kth}[2]{{#1}^{(#2)}}
\newcommand{\Wz}{\kth{\mathbf{W}}{0}}
\newcommand{\Uz}{\kth{\mathbf{U}}{0}}

\newcommand{\Wk}{\kth{\mathbf{W}}{k}}
\newcommand{\Uk}{\kth{\mathbf{U}}{k}}
\newcommand{\Sk}{\kth{\mathbf{\Sigma}}{k}}
\newcommand{\Vk}{\kth{\mathbf{V}}{k}}

\newcommand*{\sval}[1]{\sigma\!\left(#1\right)}
\newcommand{\PM}{\mathbf{M}_{\mathrm{principal}}}
\newcommand{\CM}{\mathbf{M}_{\mathrm{param}}}

\newcommand{\sdpof}{\dataset{\ensuremath{\text{SDPO}_{5\%}}}}
\newcommand{\sdpoz}{\dataset{\ensuremath{\text{SDPO}_0}}}

\newcommand{\sdpocot}{\dataset{\ensuremath{\text{SDPO}^{\text{CoT}}}}}

\newcommand{\sdpozcot}{\dataset{\ensuremath{\text{SDPO}_0^{\text{CoT}}}}}
\newcommand{\sdpomask}[1]{\dataset{\ensuremath{\text{SDPO}_{5\%}^{\text{M}#1}}}}
\newcommand{\grpo}{\dataset{\ensuremath{\text{GRPO}}}}

\newcommand{\dataset}[1]{{\scalebox{0.8}[1.0]{\textsc{#1}}}}
\newcommand{\dmath}{\dataset{Math}}
\newcommand{\dchem}{\dataset{Science}}
\newcommand{\dtool}{\dataset{ToolUse}}
\newcommand{\dcode}{\dataset{Coding}}

\newcommand{\lin}{{\scalebox{0.8}[1.0]{\text{Linear}}}}
\newcommand{\spe}{{\scalebox{0.8}[1.0]{\text{Spectral}}}}
\usepackage{amsmath,amsfonts,bm}

\def\eqref#1{equation~\ref{#1}}
\def\1{\bm{1}}

\DeclareMathAlphabet{\mathsfit}{\encodingdefault}{\sfdefault}{m}{sl}
\SetMathAlphabet{\mathsfit}{bold}{\encodingdefault}{\sfdefault}{bx}{n}

\DeclareMathOperator*{\argmin}{arg\,min}

\title{Denser $\neq$ Better: Limits of On-Policy Self-Distillation for Continual Post-Training}

\makeatletter
\renewcommand{\And}{%
  \end{tabular}\hfil\linebreak[0]\hfil
  \begin{tabular}[t]{l}\bf\rule{\z@}{0pt}\ignorespaces
}
\makeatother

\author{%
Meng Wang\textsuperscript{1},
Haohan Zhao\textsuperscript{1},
Wenzhuo Liu\textsuperscript{1},
Lu Yang\textsuperscript{1},
Geng Liu\textsuperscript{1},
Haiyang Guo\textsuperscript{2,3},\And
Guo-Sen Xie\textsuperscript{4},
Gaofeng Meng\textsuperscript{1,2,3},
Hongbin Liu\textsuperscript{1,2,3},
Fei Zhu\textsuperscript{1}\thanks{Corresponding author. Email: \texttt{fei.zhu@cair-cas.org.hk}} \\\\
\textsuperscript{1}Centre for Artificial Intelligence and Robotics, HKISI, CAS\quad%
\textsuperscript{2}Institute of Automation, CAS\\%
\textsuperscript{3}University of Chinese Academy of Sciences\quad%
\textsuperscript{4}Nanjing University of Science and Technology
}

\iclrfinalcopy % Uncomment for camera-ready version, but NOT for submission.
\begin{document}

\maketitle

\begin{abstract}
Continual post-training enables foundation models to acquire new knowledge
while preserving existing capabilities. Recent work suggests that on-policy
learning can mitigate forgetting, with self-distillation as a particularly
attractive approach. We revisit this optimistic claim through self-distillation
policy optimization (SDPO). Our experiments show that SDPO accelerates 
in-domain specialization when teacher signals are stable and well aligned,
but struggles to generalize out of distribution. In continual post-training,
SDPO exhibits greater forgetting and can even collapse, whereas GRPO, the more
established on-policy reinforcement learning method, adapts more conservatively
and better preserves prior capabilities. Further analyses link these failures to
increased drift in parameter and response space, and to amplification of 
high-frequency artifacts through a self-reinforcing teacher--student loop. 
Thus, on-policy data alone is insufficient for continual learning. 
Self-distillation is effective when teacher targets are stable and token-level
supervision is reliable, but should not be treated as a default stabilizer for 
continual post-training.
Our code is available at \url{https://github.com/Moenupa/SDPO-CL}.
\end{abstract}

\section{Introduction}

Reinforcement learning with verifiable rewards has emerged as an effective post-training paradigm for large language models (LLMs) \citep{lambert2024tulu,guo2025deepseek}. However, practical systems often require continual post-training across domains and skills \citep{qwen3technicalreport,nvidia_nemotron_3_2025}, raising a central challenge: models must acquire new capabilities without forgetting prior ones.
Existing work suggests an optimistic picture for continual post-training: on-policy reinforcement fine-tuning (RFT) forgets less than supervised fine-tuning (SFT) \citep{lai2025reinforcement, chen2025retaining, shenfeld2025rl, zhang2025reinforcement}, and this benefit has often been attributed to the use of on-policy data \citep{chen2025retaining, shenfeld2025rl, zhu2025path}. 

Recently, on-policy distillation has been claimed to be an attractive alternative to
reward-based reinforcement learning. By converting model-generated
rollouts into dense token-level supervision, these methods promise better credit
assignment, higher sample efficiency, and reduced dependence on external
teachers or explicit reward models
\citep{agarwal2024policy,gu2024minillm,thinkingmachines_on_policy_distillation}.
A particularly appealing variant is on-policy self-distillation, where the model
itself acts as a teacher by conditioning on demonstrations, feedback, solutions,
or other auxiliary context, and the resulting behavior is distilled back into the
same model. For example, self-distillation fine-tuning (SDFT) uses a
demonstration-conditioned model as its own teacher and demonstrates that such
on-policy distillation can acquire new skills while preserving prior capabilities
\citep{shenfeld2026self}. Self-distillation policy optimization (SDPO) similarly
uses a better-informed, feedback-conditioned teacher to provide dense token-level
supervision to the student policy \citep{hubotter2026reinforcement}.
Together, these works suggest that on-policy self-distillation may be a practical recipe for continual learning, combining the on-policy nature of RFT that mitigates forgetting with the efficiency of dense supervised updates for learning new knowledge.

\begin{figure}
\centering
\includegraphics[width=\textwidth]{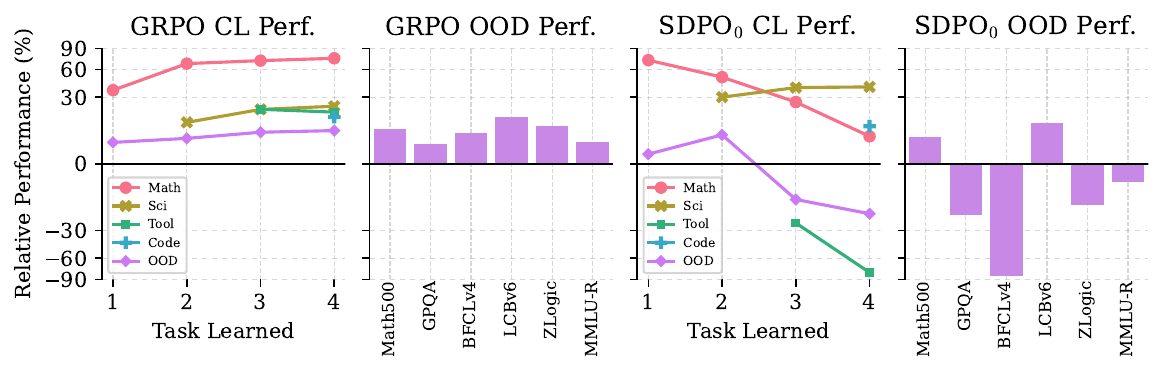}
\caption{GRPO versus SDPO relative performance to base model (\%) across sequentially learned tasks (1--4) and out-of-distribution (OOD) benchmarks. The line plots track performance retention across domains (\dmath{}, \dchem{}, \dtool{}, \dcode{}, and macro OOD average) as new tasks are acquired. The bar charts break down final relative performance across specific OOD benchmarks.}
\label{fig:1}
\end{figure}

However, this view conflates two distinct ingredients: where the data come from and which objective is used to update the model. On-policy sampling determines the histories on which the model is trained, but the training objective determines which continuation distribution is reinforced at each history. This distinction is especially important for dense self-distillation. When every token becomes a training target, teacher drift, noisy rationales, formatting conventions, and domain-specific artifacts can be reinforced as repeatedly as useful behavior \citep{li2026rethinking,kim2026does}. Thus, on-policy self-distillation may
avoid the off-policy mismatch of SFT, but it does not automatically inherit the
conservative update bias of sequence-level on-policy RFT.

% contribution & findings
This paper studies whether on-policy self-distillation enables reliable continual post-training. 
We focus on self-distillation policy optimization (SDPO, \citealp{hubotter2026reinforcement}) and compare it with sequence-level reward optimization (GRPO, \citealp{shao2024deepseekmath}) from two perspectives:

\begin{rqbox}
\begin{itemize}
\item {What are the key factors that make on-policy self-distillation beneficial? (\Cref{sec:sdpo})}
\item {Does SDPO outperform GRPO in generalization or continual learning? (\Cref{sec:sdpo_vs_grpo})}
\end{itemize}
\end{rqbox}

To answer these questions, we conduct post-training experiments in single-domain and multi-domain continual settings. 
We vary the supervision density by comparing GRPO, standard SDPO, and SDPO variants with additional teacher strategies (\Cref{sec:sdpo:ema}) or supervision tweaks (\Cref{sec:sdpo:think}), and evaluate on in-distribution and general benchmarks to assess specialization, retention, and transfer (\Cref{sec:sdpo_vs_grpo:single,sec:sdpo_vs_grpo:cl}). 
Finally, we connect behavior to diagnostics---parameter and response drift (\Cref{sec:sdpo_vs_grpo:analysis}), fixes to mitigate forgetting (\Cref{sec:sdpo_vs_grpo:merging}), and theoretical analysis (\Cref{sec:sdpo_vs_grpo:theory}) to interpret our observations.

Our results show that on-policy self-distillation offers a powerful yet fragile training signal (\Cref{fig:1}). 
SDPO can substantially improve performance on the current training domain, but it also increases the risk of drift, interference, and even collapse. 
Our observations suggest that self-distillation introduces a trade-off: it strengthens the local learning signal, but also increases sensitivity, domain mismatch, and accumulated artifacts. Our main contributions are:
\begin{itemize}
\item
We separate two factors often conflated in continual post-training: on-policy data and the objective used to turn on-policy rollouts into updates. This lets us ask whether on-policy self-distillation inherits the retention benefits attributed to on-policy RL.

\item 
We show that SDPO is a strong specialization method but a fragile continual-learning signal. Its benefits depend on teacher stability and token reliability; its forgetting concentrates on neighboring-yet-misaligned tasks and accumulates across domains. 

\item We diagnose how SDPO's forgetting behavior arises via Parameter drift, response drift, collapse analysis, and an excess-KL theoretical view.
\end{itemize}

\begin{figure}[t]
\centering
\begin{subfigure}[b]{\textwidth}
\centering
\includegraphics[width=.8\textwidth]{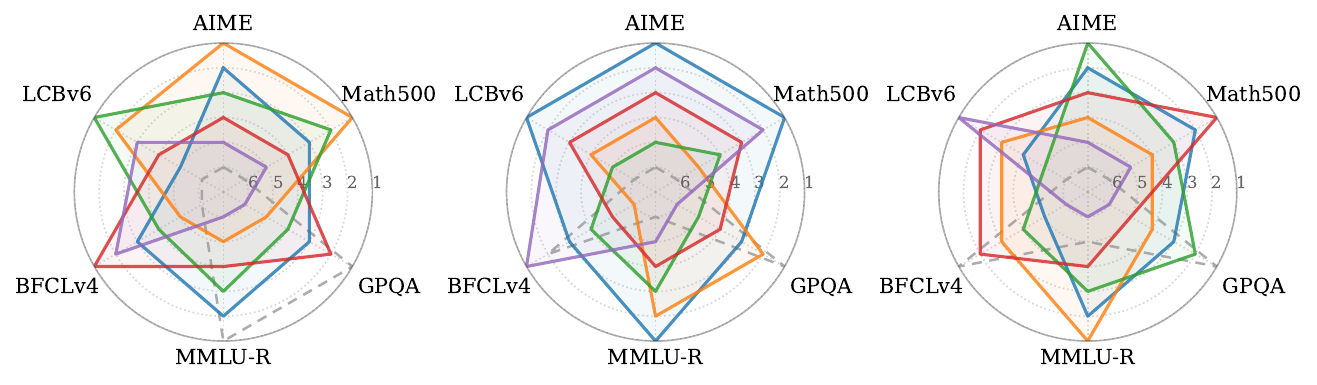}
\caption{Performance ranking across six benchmarks after \dmath{}, \dtool{}, and \dchem{} training.}
\label{fig:sdpoema:result}
\end{subfigure}\hfill
\begin{subfigure}[b]{\textwidth}
\includegraphics[width=\textwidth]{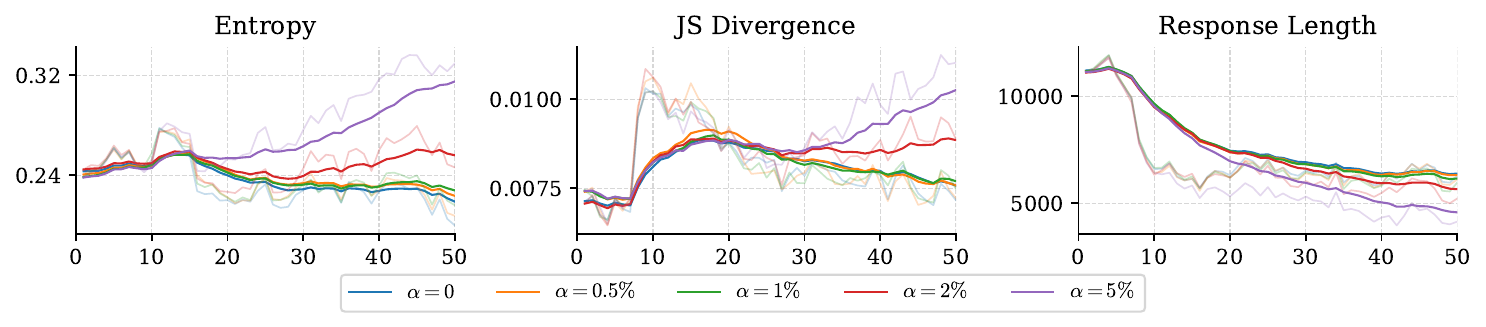}
\caption{Training dynamics of \dmath{}-SDPO under different teacher EMA rates
$\alpha\in\{0,0.5\%,1\%,2\%,5\%\}$.}
\label{fig:sdpoema:training}
\end{subfigure}\\
\caption{Effect of SDPO teacher EMA update rate $\alpha$: no single $\alpha$ is globally optimal.
\textbf{(a)} Ranking-based radar plots; a larger area is better. 
\textbf{(b)} Training dynamics during \dmath{}-SDPO. }
\label{fig:sdpoema}
\end{figure}

\section{Token-Level Dense Supervision is Powerful but Sensitive}
\label{sec:sdpo}

First, we argue that SDPO's benefits depend on the quality of the dense signal. 
SDPO distills token logits from the teacher model, which improves sample efficiency because every generated token contributes supervision. 
However, the same mechanism reinforces teacher noise, temporal mismatch, and formatting artifacts many times within a single response. We therefore begin by asking: \emph{how do token-level reward stability (\Cref{sec:sdpo:ema}) and density (\Cref{sec:sdpo:think}) affect self-distillation?}

\subsection{Token-Level Supervision Requires Fresh but Temporally Stable Teachers}
\label{sec:sdpo:ema}

A natural intuition in SDPO is that the teacher should be updated frequently:
a fresher teacher should better reflect the student's latest improvements and
therefore provide denser and stronger token supervision. To test this intuition, we
sweep the EMA update rate $\alpha \in \{0,0.5\%,1\%,2\%,5\%\}$:
\[
\Theta_{\text{teacher}}^{(k)}
=
(1-\alpha)\Theta_{\text{teacher}}^{(k-1)}
+
\alpha\Theta_{\text{student}}^{(k)} .
\]
Surprisingly, freshness alone does not predict performance. As shown in
\Cref{fig:sdpoema:result}, no single $\alpha$ is uniformly optimal: in
\dmath{}-SDPO, $\alpha=1\%$ gives the best overall profile, while $\alpha=5\%$
performs worst; in \dtool{} training, the frozen teacher is best, although
$\alpha=5\%$ is competitive. This raises a question: \emph{if a larger EMA rate makes
the teacher fresher, why can it hurt SDPO so severely?}

We argue that EMA entangles two different properties of the teacher:
\emph{freshness} and \emph{temporal stability}. A larger $\alpha$ makes the
teacher weights closer to the current student, but it also makes the teacher a
more rapidly moving target. This is particularly risky for distillation:
because SDPO applies supervision at every generated token, small teacher
fluctuations can be reinforced repeatedly within a single response. The training
curves in \Cref{fig:sdpoema:training} support this view. Large EMA rates,
especially $\alpha=2\%$ and $\alpha=5\%$, observe late-stage increases in entropy
and JS divergence, indicating growing teacher--student mismatch and less
reliable token targets. Smaller or frozen teachers avoid this instability, but
they introduce the opposite problem: the teacher can become stale and provide
a weakening optimization signal as the student evolves.

To separate freshness from per-step teacher dynamics, we introduce a simple
\emph{restart-and-freeze} strategy. Instead
of updating the teacher continuously, we periodically refresh it from the current
student and then keep it frozen within the next interval:
\[
\Theta_{\text{teacher}}
\leftarrow
(1-\gamma)\Theta_{\text{teacher}}^{\text{old}}
+
\gamma\Theta_{\text{student}}^{(k)},
\quad 0 < \gamma \leq 1 .
\]
This design preserves teacher freshness at refresh points while removing the
step-to-step volatility of EMA updates. For comparison with EMA, we match the
effective $k$-step refresh amount by
$\gamma \approx 1-(1-\alpha)^k$; for example, $\alpha=5\%$ over $k=10$ steps
corresponds to $\gamma \approx 0.40$.

\begin{table}[t]
\centering\small \setlength{\tabcolsep}{10pt}\renewcommand{\arraystretch}{0.8}
\caption{\dmath{}-SDPO performance (Acc@8). 
Restart-and-freeze offers a middle ground: it periodically re-initializes the teacher but prevents per-step noise from being distilled repeatedly.}
\label{tab:sdpoema:restart}
\begin{tabular}{lcccccl}
\toprule
Teacher Strategy & AIME & Math500 & GPQA  & ZLogic & MMLU-R& ~Avg \\
\midrule
--           & 32.71 & 73.85   & 62.94 & 81.62  & 86.02 & 67.43 \\
\midrule
EMA $\alpha=0$ & {56.42} & 84.29   & 58.72 & {74.50}  & 85.47 & {71.88} \\
Restart $\gamma=1$ & \best{57.50} & \best{85.05} & \best{60.86} & \best{74.68}  & \best{86.60} & \upshow{1.06}{\best{72.94}} \\
% \midrule
% EMA $\alpha=1\%$ & {55.90} & 84.29   & 58.44 & {74.03}  & 85.51 & {71.63}\\
% Restart@10, $\gamma=10\%$&53.34&{85.15}&{58.71}&73.18&{86.28}&71.33 \\
\midrule
EMA $\alpha=5\%$ & 34.38 & 78.97 & 55.87&\best{73.68}&83.82&65.34\\
Restart $\gamma=40\%$&\best{55.00}&\best{84.62}&\best{59.03}&73.50&\best{87.05}&\upshow{6.50}{\best{71.84}} \\
\bottomrule
\end{tabular}
\end{table}

\begin{figure}[t]
\centering
\includegraphics[width=\textwidth]{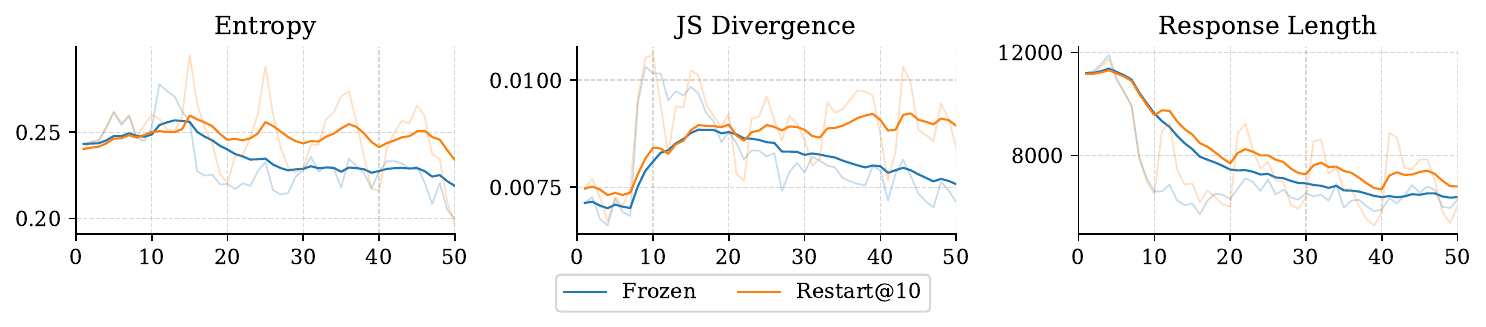}
\caption{Training dynamics during \dmath{}-SDPO with teacher \textit{frozen} versus 10-step-hard-restart.}
\label{fig:sdpoema:restart}
\end{figure}

The results in \Cref{tab:sdpoema:restart} support this decomposition.
Replacing continuous EMA with matched restart-and-freeze substantially repairs
the destructive behavior of $\alpha=5\%$: on AIME, performance improves from
34.38\% to 55.00\%, and the average score improves from 65.34\% to 71.84\%, a gain of
6.50 points. Importantly, the restart strategy also improves over the frozen
teacher: hard restart with $\gamma=1$ outperforms the frozen teacher on all
five benchmarks, improving the average score from 71.88\% to 72.94\%. The problem is not freshness per se: freshness is useful only when delivered through a stable teacher signal.
The training dynamics in \Cref{fig:sdpoema:restart} further illustrate
this point. The frozen teacher produces smooth curves, but its entropy and JS
divergence gradually decline, suggesting that the teacher becomes increasingly
outdated. In contrast, restart-and-freeze produces bounded oscillations aligned
with refresh events: each restart renews the teacher signal, and avoids noisy per-step target changes. The JS divergence remains bounded, with neither late-stage blow-up nor collapse.

\begin{takeawaybox}
Teacher design is central to on-policy self-distillation, with a \textit{stability--freshness} dilemma: a frozen teacher is stable but stale in the long term; a fast EMA teacher is fresh but can amplify online noise. Periodic refresh-and-freeze provides more reliable targets.
\end{takeawaybox}

\subsection{CoT Distillation: More Supervision Is Not Always Better}
\label{sec:sdpo:think}
We next consider the supervision density of SDPO, by additionally distilling Chain-of-Thought tokens.  This is a stronger test of the dense-supervision hypothesis: if
more token-level targets were uniformly beneficial, CoT distillation should
consistently improve performance. \Cref{fig:sdpothink:results} shows the
opposite: CoT helps \dtool{} training but hurts \dmath{} and \dchem{} training.

We attribute this contrast to the reliability of intermediate tokens. CoT
distillation does not merely add more supervision; it changes the target from
answer-producing behavior to a particular reasoning trajectory. In long-form
\dmath{} and \dchem{} tasks, such trajectories are underdetermined and often contain
redundant self-checks, uncertain interpretations, formatting habits, or incorrect
intermediate claims. These tokens are weakly tied to final correctness, so
distilling them can dilute credit assignment and encourage imitation of reasoning
artifacts. This is reflected in \dmath{} training by longer responses, higher entropy,
and larger JS divergence (\Cref{fig:sdpothink:math}).

By contrast, \dtool{} traces are short and schema-constrained. CoT therefore acts as
useful process supervision in \dtool{} training, yielding stable dynamics and
improved BFCLv4 performance (\Cref{fig:sdpothink:tool}). Additional EMA
ablations further show that long-CoT supervision becomes especially harmful
when the teacher is updated rapidly (\Cref{appdx:tab:sdpothink}), consistent with a confirmation-bias loop in which noisy reasoning patterns are copied into the teacher and reinforced again.

\begin{takeawaybox}
More supervision is useful only when reliable.
CoT distillation benefits short and structured traces, but can harm long, noisy, or artifact-prone rationales.
\end{takeawaybox}

\begin{figure}[t]
\centering
\begin{subfigure}[b]{\textwidth}
\includegraphics[width=\textwidth]{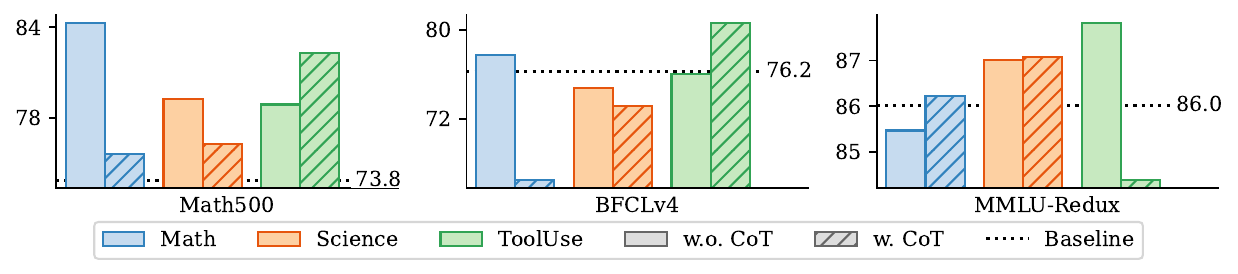}
\caption{SDPO Performance (Acc@8) trained on \dmath{}, \dchem{}, and \dtool{} use against baseline (before training).}
\label{fig:sdpothink:results}
\end{subfigure}\\
\begin{subfigure}[b]{\textwidth}
\includegraphics[width=\textwidth]{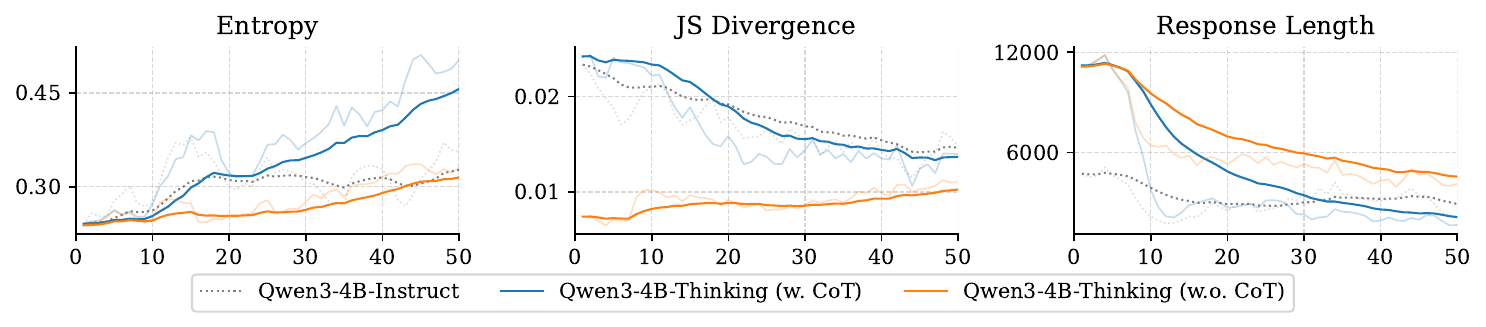}
\vskip -.5em
\caption{SDPO training dynamics on \dmath{} (which requires long CoT).}
\label{fig:sdpothink:math}
\end{subfigure}\\
\begin{subfigure}[b]{\textwidth}
\includegraphics[width=\textwidth]{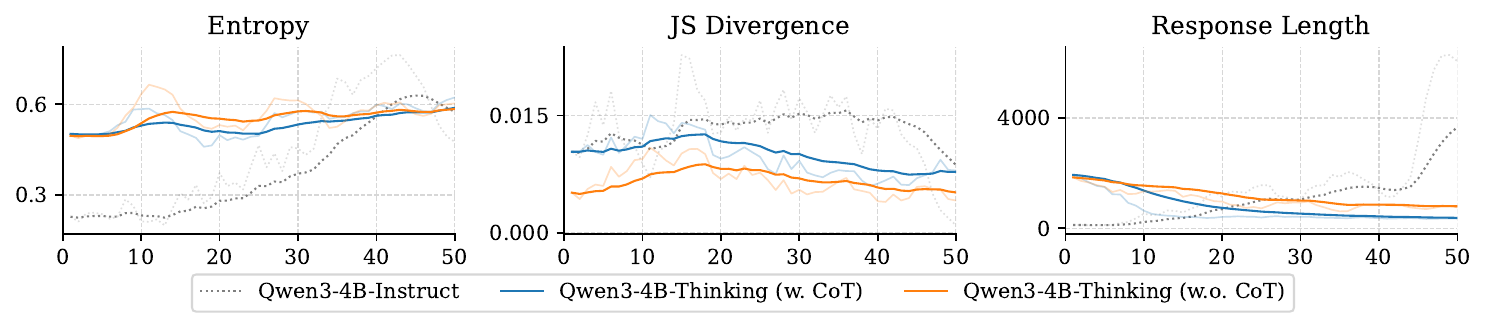}
\vskip -.5em
\caption{SDPO training dynamics on \dtool{} (which requires short CoT).}
\label{fig:sdpothink:tool}
\end{subfigure}\\
\caption{With and without CoT SDPO Performance (Acc@8) and training dynamics.}
\label{fig:sdpothink}
\end{figure}

\section{SDPO Trades Retention for Specialization}
\label{sec:sdpo_vs_grpo}

The previous section shows that SDPO's dense signal is conditional: it helps only when the teacher is stable and the distilled tokens are reliable. 
We now ask whether these local benefits translate into broader training advantages. \emph{Does token supervision improve out-of-distribution (OOD) generalization (\Cref{sec:sdpo_vs_grpo:single}), or mitigate forgetting in continual post-training (\Cref{sec:sdpo_vs_grpo:cl})?}

\subsection{Rethinking OOD Generalization: Intermediate-Distance Interference}
\label{sec:sdpo_vs_grpo:single}

\begin{table}[t]
\centering\small
\setlength{\tabcolsep}{5.5pt}\renewcommand{\arraystretch}{0.8}
\caption{Performance (Acc@8) of models trained on a single domain. 
Rows indicate training domain and method; columns indicate evaluation benchmarks.
\sdpoz{} and \sdpof{} denote SDPO with teacher EMA $0$ and $5\%$, respectively.
\sdpozcot{} denotes \sdpoz{} with CoT distilled.}
\label{tab:singledomain:result}
% \documentclass[]{article}
% \input{section/preamble}

% \begin{document}

% \begin{table}[ht]
% \small\setlength{\tabcolsep}{3pt}

\begin{NiceTabular}{llcccccccccc}
\CodeBefore
  % \cellcolor{gray!15}{2-2} % Colors cell in row 2, column 2
\Body
\toprule
&&\multicolumn{2}{c}{\dmath}&\multicolumn{2}{c}{\dchem{}}&\multicolumn{2}{c}{\dtool{}}&\multicolumn{2}{c}{\dcode{}}&\dataset{Logic}&\dataset{Know.}\\
\cmidrule(lr){3-4}\cmidrule(lr){5-6}\cmidrule(lr){7-8}\cmidrule(lr){9-10}\cmidrule(lr){11-11}\cmidrule(lr){12-12}
&Method&{\scriptsize AIME}&{\scriptsize Math500}&{\scriptsize \dchem{}}&{\scriptsize  GPQA  }&{\scriptsize  \dtool{} }&{\scriptsize  BFCLv4}&{\scriptsize  \dcode{} }&{\scriptsize  LCBv6 }&{\scriptsize  ZLogic }&{\scriptsize  MMLU-R} \\
\midrule
&--&32.71&73.85&50.54&62.94&48.90&76.25&56.39&46.73&81.62&86.02  \\
\midrule\multirow{3}{*}{\rotatebox[origin=c]{90}{{\scriptsize\dmath}}}
&\grpo&\df{11.96}{44.67}&\df{2.99}{76.84}&\df{0.79}{51.33}&\df{-0.11}{62.83}&\df{1.58}{50.48}&\df{2.08}{78.33}&\df{4.34}{60.73}&\df{6.20}{52.93}&\df{-0.80}{80.82}&\df{-0.28}{85.74} \\
&\sdpof&\df{1.67}{34.38}&\df{5.12}{78.97}&\df{-5.54}{45.00}&\df{-7.07}{55.87}&\df{4.41}{53.31}&\df{1.87}{78.12}&\df{-4.02}{52.37}&\df{4.13}{50.86}&\df{-7.94}{73.68}&\df{-2.20}{83.82} \\
&\sdpoz&\df{23.71}{56.42}&\df{10.44}{84.29}&\df{-3.18}{47.36}&\df{-4.22}{58.72}&\df{3.16}{52.06}&\df{1.46}{77.71}&\df{1.19}{57.58}&\df{1.79}{48.52}&\df{-7.12}{74.50}&\df{-0.55}{85.47} \\
\midrule\multirow{3}{*}{\rotatebox[origin=c]{90}{{\scriptsize\dchem{}}}}
&\grpo&\df{2.29}{35.00}&\df{-0.38}{73.47}&\df{8.51}{59.05}&\df{-0.44}{62.50}&\df{-0.19}{48.71}&\df{-1.25}{75.00}&\df{-0.86}{55.53}&\df{-0.83}{45.90}&\df{5.78}{87.40}&\df{1.48}{87.50} \\
&\sdpof&\df{20.63}{53.34}&\df{4.57}{78.42}&\df{21.72}{72.26}&\df{-5.55}{57.39}&\df{2.39}{51.29}&\df{-3.96}{72.29}&\df{4.27}{60.66}&\df{9.38}{56.11}&\df{9.06}{90.68}&\df{-1.92}{84.10} \\
&\sdpoz&\df{22.92}{55.63}&\df{5.40}{79.25}&\df{18.92}{69.46}&\df{-1.32}{61.62}&\df{1.65}{50.55}&\df{-1.46}{74.79}&\df{3.62}{60.01}&\df{5.21}{51.94}&\df{8.18}{89.80}&\df{0.98}{87.00} \\
\midrule\multirow{3}{*}{\rotatebox[origin=c]{90}{{\scriptsize\dtool{}}}}
&\grpo&\df{5.84}{38.55}&\df{2.75}{76.60}&\df{-1.43}{49.11}&\df{-0.69}{62.25}&\df{9.00}{57.90}&\df{3.33}{79.58}&\df{1.49}{57.88}&\df{0.96}{47.69}&\df{-4.27}{77.35}&\df{-1.07}{84.95} \\
&\sdpof&\df{15.00}{47.71}&\df{4.72}{78.57}&\df{-3.99}{46.55}&\df{-5.55}{57.39}&\df{3.86}{52.76}&\df{1.46}{77.71}&\df{-1.45}{54.94}&\df{9.17}{55.90}&\df{2.40}{84.02}&\df{0.23}{86.25} \\
&\sdpoz&\df{15.00}{47.71}&\df{5.03}{78.88}&\df{0.11}{50.65}&\df{-2.21}{60.73}&\df{6.61}{55.51}&\df{-0.21}{76.04}&\df{4.61}{61.00}&\df{11.75}{58.48}&\df{4.86}{86.48}&\df{1.80}{87.82} \\
\midrule\multirow{3}{*}{\rotatebox[origin=c]{90}{{\scriptsize\dcode{}}}}
&\grpo&\df{0.83}{33.54}&\df{-0.15}{73.70}&\df{1.42}{51.96}&\df{-0.31}{62.63}&\df{1.47}{50.37}&\df{-0.83}{75.42}&\df{1.45}{57.84}&\df{0.08}{46.81}&\df{0.28}{81.90}&\df{-0.50}{85.52} \\
&\sdpof&\df{16.46}{49.17}&\df{4.13}{77.98}&\df{-2.62}{47.92}&\df{-4.61}{58.33}&\df{0.73}{49.63}&\df{-2.29}{73.96}&\df{5.71}{62.10}&\df{11.74}{58.47}&\df{-1.09}{80.53}&\df{1.23}{87.25} \\
&\sdpoz&\df{16.88}{49.59}&\df{3.65}{77.50}&\df{-0.54}{50.00}&\df{-2.65}{60.29}&\df{2.57}{51.47}&\df{2.92}{79.17}&\df{7.79}{64.18}&\df{13.39}{60.12}&\df{0.96}{82.58}&\df{0.63}{86.65} \\
\bottomrule
\end{NiceTabular}

% \caption{Caption}
% \label{tab:placeholder}
% \end{table}

% \end{document}
\end{table}

\begin{figure}[t]
\centering
\includegraphics[width=\textwidth]{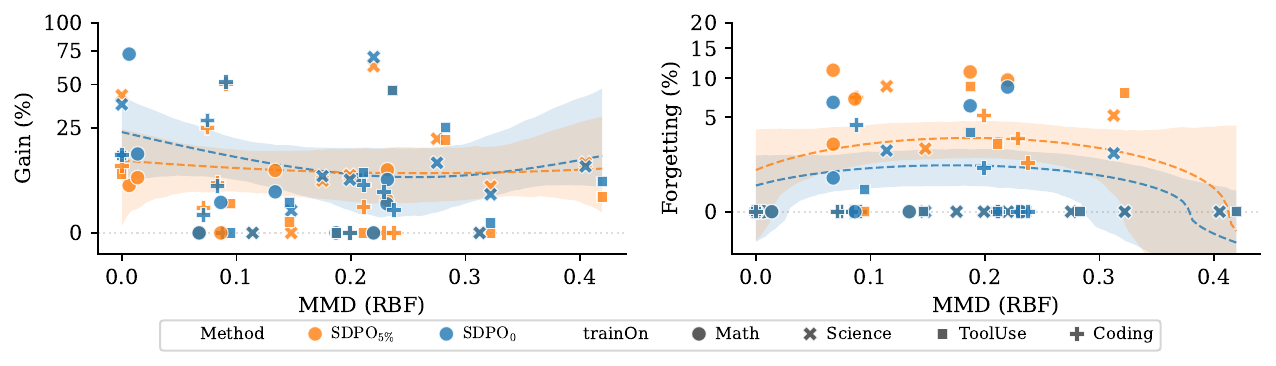}
\caption{Relative performance change in terms of gain $G$ and forgetting $F$ versus Maximum Mean Discrepancy with a Radial Basis Function kernel.}
\label{fig:mmd}
\end{figure}

SDPO's main advantage is specialization, rather than uniformly better generalization (\Cref{tab:singledomain:result}).
Under \dmath{} training, \sdpoz{} is the strongest math specialist, raising AIME from 32.71\% to 56.42\% and Math500 from 73.85\% to 84.29\%. Similarly, \dcode{}-\sdpoz{} improves LCBv6 by 13.39\%. These results indicate that SDPO can strongly exploit source-like evaluation distributions.

However, this specialization does not translate into uniform retention. 
We decompose relative performance change into positive \textit{gain} and negative \textit{forgetting} (\Cref{eq:gain,eq:forgetting}) and show that the SDPO-induced forgetting pattern is non-monotonic, as \Cref{fig:mmd} plots gain and forgetting against embedding distance measured by Maximum Mean Discrepancy (MMD, \citealp{gretton2012kernel}, \Cref{eq:mmd}). 
Source-like tasks (e.g., $\textrm{MMD}\lesssim0.05$) and very distant tasks (e.g., $\textrm{MMD}\gtrsim0.4$) mostly improve, while the largest drops concentrate in intermediate-distance tasks.

GPQA is the clearest example of intermediate-forgetting: it is not the farthest, yet it consistently degrades after SDPO training from multiple source domains. 
We interpret this pattern as evidence that GPQA may lie in an interference region: it is close enough to be affected by source-domain updates, yet insufficiently aligned to benefit from them. 
More broadly, we observe SDPO degradation on \dmath{}-to-\dchem{}, \dmath{}-to-ZLogic, and \dchem{}-to-BFCLv4 (additional cases listed in \Cref{tab:intermediate_cases}). 

These results suggest that SDPO's OOD failures are consistent with an intermediate-distance interference pattern, rather than a simple monotonic function of distributional distance.
One explanation is that distillation specializes the model toward the source-domain response manifold.
When an OOD task is very close, this specialization may transfer; when it is very far, the perturbation may not strongly affect it.
The most vulnerable tasks are partially aligned OOD tasks: close enough to be changed, but misaligned enough not to benefit. This non-monotonic pattern is consistent with prior theory of continual learning with linear models \citep{lin2023theory,goldfarbjoint}, and our results suggest that a similar interference regime emerges in LLM continual post-training.

\begin{takeawaybox}
For OOD generalization, SDPO improves source-aligned tasks and leaves distant tasks relatively unaffected or even improved, but interferes most strongly with partially aligned tasks that are close enough to be shifted and yet misaligned enough not to benefit.
\end{takeawaybox}

\subsection{Continual Learning Exposes the Accumulated Cost of Dense Supervision}
\label{sec:sdpo_vs_grpo:cl}

\begin{figure}[t]
\centering
\includegraphics[width=\textwidth]{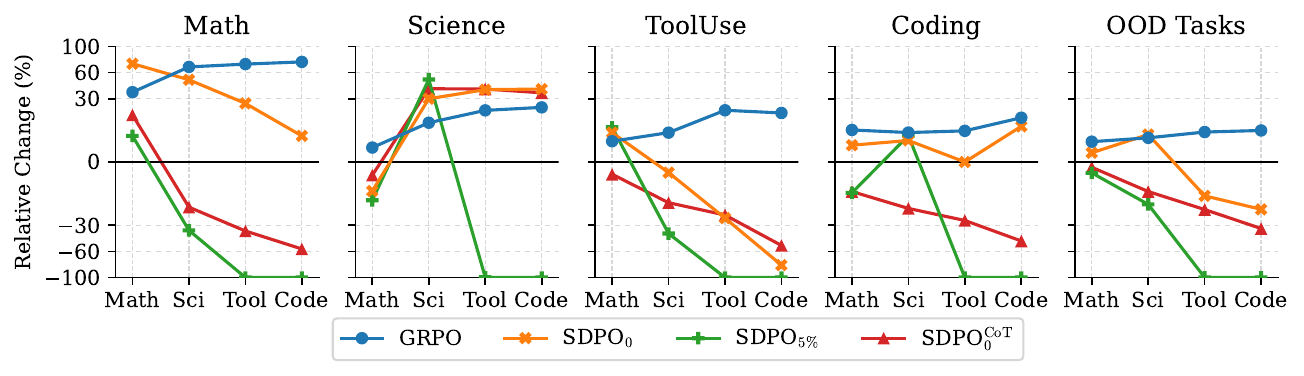}
\vskip -1em
\caption{Normalized accuracy percentage change ($y$-axis) across continual training stages ($x$-axis).
GRPO improves almost monotonically, whereas SDPO degrades more often.
Each subplot denotes one evaluation benchmark: \dmath{}, \dchem{}, \dtool{}, \dcode{}, and out-of-domain tasks.}
\label{fig:cl}
\end{figure}

% \begin{table}[t]
% \centering\small \setlength{\tabcolsep}{10pt}\renewcommand{\arraystretch}{0.8}
% \caption{Continual learning performance (Acc@8) of SDPO versus GRPO. SDPO often wins early (better current-stage), but GRPO wins over time (better last-stage). \emph{Current}: performance on the domain immediately after training it. \emph{Last}: performance after all training stages.}
% \label{tab:cl}
% \input{tab/sdpo_cl_qwen}
% \end{table}

Single-domain interference becomes more severe when updates are composed across
domains. We therefore train models sequentially on \dmath{} $\rightarrow$ \dchem{}
$\rightarrow$ \dtool{} $\rightarrow$ \dcode{} and evaluate after each stage. 

The continual setting reverses the single-domain picture. \sdpoz{} is a strong early
specializer: it reaches 56.42\% on \dmath{} and 65.70\% on \dchem{}, exceeding
GRPO's 44.67\% and 56.39\%, respectively. However, these gains do not survive later domain shifts:
after the full sequence, \sdpoz{} drops to 34.38\% on \dmath{} and 9.93\% on \dtool{}, and
also falls below the base model on \dataset{Logic} and \dataset{Knowledge}. \sdpozcot{} degrades even
earlier, consistent with the instability of long-CoT token targets. The most severe
case is \sdpof{}, which collapses after the third stage and remains near zero. By
contrast, GRPO adapts more conservatively but retains capabilities more reliably,
achieving final improvements over the base model across all evaluated categories
(\Cref{fig:cl}).

% These results argue that single-domain specialization is not evidence of continual
% learning. SDPO's dense token-level updates can accelerate local adaptation, but the
% same updates also accumulate drift, interference, and artifacts across domain shifts.
% Sequence-level GRPO is less aggressive, but its more selective reward signal better
% supports long-horizon retention.

\begin{takeawaybox}
For continual post-training, dense token-level supervision improves local adaptation but is not inherently retention-preserving; when composed across tasks, it can accumulate drift and interference that result in forgetting.
\end{takeawaybox}

\section{Understanding SDPO in Continual Post-training}
\subsection{Dense Supervision Causes Larger Parameter and Response Drift}
\label{sec:sdpo_vs_grpo:analysis}

In our benchmarks, SDPO adapts more aggressively than GRPO. 
To understand why, we look inside the model and ask whether SDPO's greater behavioral changes correspond to larger parameter shifts.
For each LLM weight matrix, we compute the compact singular value decomposition (\Cref{eq:svd}) to inspect the input-output directions and their spectral strengths. We quantify how training changes the geometry, spectral magnitude, subspace orientation, and dominant parameter regions of the model.
Specifically, we compute the following metrics between the initial $\Wz$ and step $k$ $\Wk$:

\begin{itemize}
\item \emph{Normalized spectral shift}:
$\ell_2$-normalized drift of the singular-value profile per \Cref{eq:nss}.
Larger values indicate substantial strength changes in dominant directions.

\item \emph{Principal-angle rotation}:
maximum rotation within top-$k$ left singular subspaces, as shown in \Cref{eq:maxpa}.
A small value implies better preservation of dominant subspaces.

\item \emph{Principal-mask overlap}:
intersection-over-union of principal masks before and after training as defined in \Cref{eq:pmoverlap}.
High values mean dominant entries remain important.

\item \emph{Principal-mask update sparsity}:
intersection-over-union of the initial principal mask and the changed-parameter mask as defined in \Cref{eq:cmoverlap}.
A high value indicates that training directly modifies structurally important parameters.
\end{itemize}

\begin{figure}[t]\centering

\begin{minipage}[c]{0.49\textwidth}
\begin{subfigure}{\textwidth}
\includegraphics[width=\textwidth]{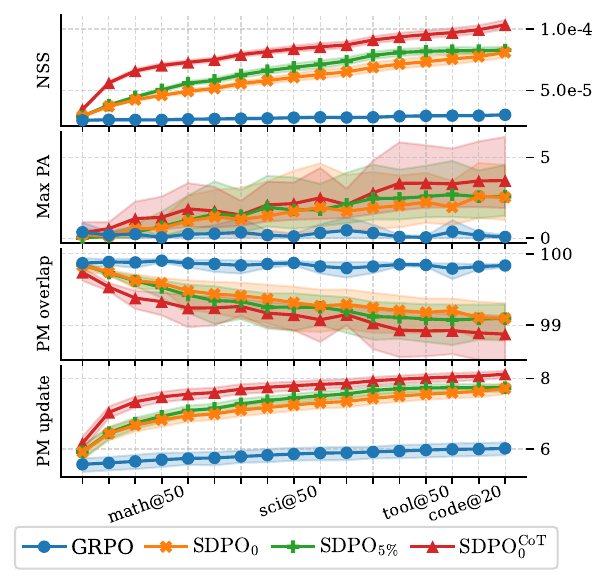}
\caption{Shift of self-Attention $V$ projection during continual learning.}\label{fig:svd:cl}
\end{subfigure}\end{minipage}\hfill%
\begin{minipage}[c]{0.49\textwidth}
\begin{subfigure}[t]{\textwidth}\centering
\includegraphics[width=.9\textwidth]{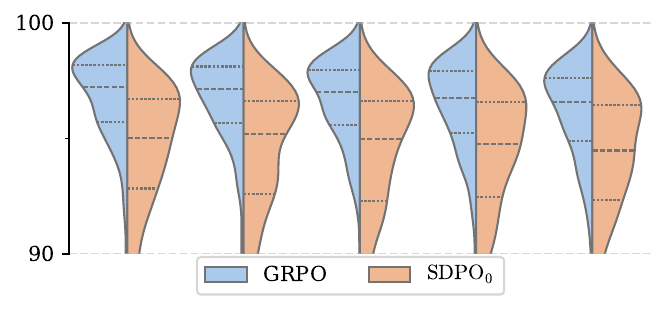}
\vskip -.7em
\caption{Embedding similarity shift during \dmath{} training; $x$-axis denotes checkpoints at steps $k=10,20,\dots,50$.}
\label{fig:rollout:math}
\end{subfigure}\\%
\begin{subfigure}{\textwidth}\centering
\includegraphics[width=.9\textwidth]{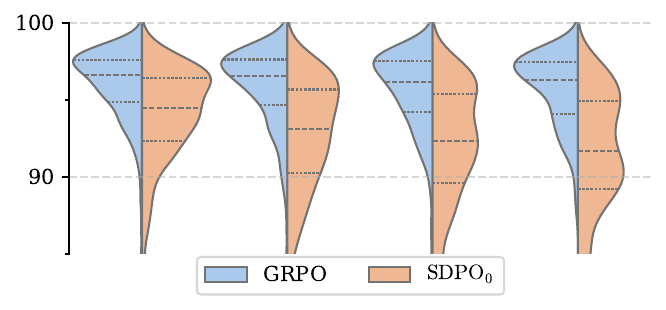}
\vskip -.7em
\caption{Embedding similarity shift during continual learning; $x$-axis denotes checkpoints after tasks \dmath{}, \dchem{}, \dtool{}, and \dcode{}.}
\label{fig:rollout:cl}
\end{subfigure}%
\end{minipage}
\caption{\textbf{(a)} 
Model parameter drift ($y$-axis) against training steps ($x$-axis).
Metrics are averaged across all layers with 95\% confidence intervals.
NSS, PA, and PM denote normalized spectral shift, principal angle, and principal mask, respectively.
\textbf{(b, c)} Cosine similarity (\%) between AIME rollout(@8) embeddings before and after
{\color{RoyalBlue}GRPO (left half)} versus {\color{Orange}\sdpoz{} (right half)} training.}
\label{fig:analysis}
\end{figure}

\Cref{fig:analysis} reveals a consistent parameter and response drift hierarchy. 
GRPO changes the model least: its spectral shift and subspace rotation remain small, and its principal-mask overlap stays near $99.9\%$.
SDPO variants drift farther, with \sdpozcot{} producing the largest subspace change.
This pattern is consistent with an accumulated-drift explanation of continual-learning failure: 
SDPO produces notably larger spectral shifts, more principal-parameter modifications, and response embedding distribution changes than GRPO.
The ordering $\sdpozcot>\sdpof>\sdpoz>\grpo$ matches the aggressiveness of the supervision signal: CoT increases token coverage, whereas fast EMA increases target non-stationarity. This suggests that denser token targets drive stronger internal updates.

\begin{takeawaybox}
Faster adaptation is accompanied by larger drift.
SDPO reshapes parameters and responses more than GRPO, which aids specialization but undermines continual-learning stability.
\end{takeawaybox}

\subsection{Mitigating SDPO Forgetting with Weight Edits}
\label{sec:sdpo_vs_grpo:merging}

\begin{table}[t]
\centering \small\setlength{\tabcolsep}{3pt} \renewcommand{\arraystretch}{0.8}
\caption{Comparison of linear and spectral weight editing for mitigating forgetting in \sdpoz{} (Acc@8). Both methods merge \dmath{}-\sdpoz with the pre-update weights at \(50\%\) scale; spectral merging additionally retains only the top \(r=16\) singular components. $\Delta'$ denotes changes on non-\dmath{} tasks.}
\label{tab:merging}
\begin{NiceTabular}{lcccccccccccc}
\CodeBefore
  \cellcolor{gray!15}{4-2}\cellcolor{gray!15}{4-3}
  \cellcolor{gray!15}{5-2}\cellcolor{gray!15}{5-3}
  \cellcolor{gray!15}{6-2}\cellcolor{gray!15}{6-3}
  \cellcolor{gray!15}{7-2}\cellcolor{gray!15}{7-3}
\Body
\toprule
&\multicolumn{2}{c}{\dmath{}}&\multicolumn{2}{c}{\dchem{}}&\multicolumn{2}{c}{\dtool{}}&\multicolumn{2}{c}{\dcode{}}&\dataset{Logic}&\dataset{Know.}&\multirow{2}{*}{Avg}&\multirow{2}{*}{$\Delta'$}\\
\cmidrule(lr){2-3}\cmidrule(lr){4-5}\cmidrule(lr){6-7}\cmidrule(lr){8-9}\cmidrule(lr){10-10}\cmidrule(lr){11-11}
{Method}&{\scriptsize AIME}&{\scriptsize Math500}&{\scriptsize \dchem{}}&{\scriptsize  GPQA  }&{\scriptsize  \dtool{} }&{\scriptsize  BFCLv4}&{\scriptsize \dcode{} }&{\scriptsize  LCBv6 }&{\scriptsize  ZLogic }&{\scriptsize  MMLU-R} & & \\
\midrule
Baseline&32.71&73.85&50.54&62.94&48.90&76.25&56.39&46.73&81.62&86.02&61.60&--\\
\midrule
\dmath{}-\grpo&44.67&76.84&51.33&62.83&50.48&78.33&60.73&52.93&80.82&85.74&64.47&\dfcell{1.73}\\
\midrule
\dmath{}-\sdpoz&56.42&84.29&47.36&58.72&52.06&77.71&57.58&48.52&74.50&85.47&64.26&\dfcell{-0.93}\\
% $+\lin{}^{s=30\%}$& 48.34&81.42&50.71&63.38&49.45&75.00&56.71&44.01&79.20&85.65\\
$+\lin{}^{50\%}$& 51.67&84.25&48.81&62.25&50.92&77.29&56.47&46.65&76.42&85.65&64.04&\dfcell{-0.62}\\
% $+\lin{}^{s=70\%}$& 55.00&84.82&46.07&59.41&52.57&76.67&57.59&48.36&72.10&84.12\\
% $+\spe{}^{s=30\%}$&41.46&79.10&50.77&62.18&48.71&74.58&55.79&45.25&80.88&85.95\\
$+\spe{}^{50\%}_{r=16}$&50.21&82.05&52.26&61.62&48.90&75.21&57.95&46.61&79.38&86.40&64.06&\dfcell{-0.13}\\
% $+\spe{}^{50\%}_{r=64}$&49.59&81.88&48.81&60.54&49.63&76.25&57.44&45.35&78.72&85.65&63.39\\
% $+\spe{}^{50\%}_{r=256}$&50.42&82.65&49.64&61.87&49.45&77.50&56.24&44.25&78.70&85.98&63.67\\
% $+\spe{}^{s=70\%}$&54.17&83.15&49.70&60.10&50.74&76.67&57.45&47.72&76.60&86.18\\
\bottomrule
\end{NiceTabular}
\end{table}

The preceding analysis suggests that continual self-distillation may accumulate both beneficial task adaptation and undesirable behavioral shifts, which motivates weight-space editing to reduce forgetting. We investigate two merging strategies between the pre-update weights \(W_0\) and the post-trained weights \(W_1\), with task vector \(\Delta=W_1-W_0\). \emph{Linear merging} interpolates the full task vector as \(W=W_0+s\Delta=(1-s)W_0+sW_1\). \emph{Spectral merging} instead retains only the \(r\) largest singular components of each matrix-valued task vector, yielding \(W=W_0+s\,\operatorname{trim}_r(\Delta)\). This design tests whether dominant low-rank directions preserve task-specific gains while suppressing weaker directions that may contribute to interference and out-of-domain forgetting.

Results in \Cref{tab:merging} (full results in \Cref{appdx:tab:merging}) show that both strategies mitigate several behavioral shifts induced by \(\sdpoz\) while largely preserving its overall performance. At a \(50\%\) merge ratio, linear merging improves GPQA from \(58.72\) to \(62.25\) and ZLogic from \(74.50\) to \(76.42\), bringing them closer to the \(\grpo\) results of \(62.83\) and \(80.82\), respectively. Spectral merging with \(r=16\) provides stronger recovery on ZLogic (\(79.38\)) and also improves \dchem{} to \(52.26\) and MMLU-R to \(86.40\), surpassing both the unmerged \(\sdpoz\) and \(\grpo\) variants on the latter two benchmarks. These improvements involve some reduction in \(\sdpoz\)'s large mathematics gains, but the average remains nearly unchanged. Thus, merging produces a more balanced model, with spectral trimming in particular moving \(\sdpoz\) toward the lower-forgetting behavior of \(\grpo\) without materially sacrificing aggregate performance.

\subsection{Theoretical View on Forgetting: Excess KL Drift}
\label{sec:sdpo_vs_grpo:theory}

Existing studies have shown that on-policy RL is biased toward solutions that solve the
new task while staying close to the base policy in KL \citep{shenfeld2025rl}. However, SDPO does
not inherit this bias: its rollouts are on-policy, but its local target is a
teacher distribution conditioned on augmented context, rather than the
KL-minimal distribution induced by a sequence-level reward.

Fix a prompt $x$ with base sequence distribution $P_0(\cdot\mid x)$, and let
$S_x$ be the set of successful responses under a binary sequence reward.
Among all sequence distributions $Q$ with $Q(S_x)=\alpha$, the KL-minimal
one---the \emph{matched-success Razor policy} $Q_{\alpha,x}^{\dagger}$, with
token kernel $q_{\alpha,x}^{\dagger}(\cdot\mid h_t)$ at history
$h_t=(x,y_{<t})$---reweights the base policy only as needed to change the
probability of eventual success. In contrast, SDPO distills a teacher
$T_c(\cdot\mid h_t)=\pi_T(\cdot\mid x,c,y_{<t})$ that observes augmented
context $c$, such as a solution or rationale. For a strict distillation
divergence $D$, its population token target is the teacher projection
$B(\cdot\mid h_t)\in\argmin_{u\in\Delta(\mathcal{V})}
\mathbb{E}_{c\mid h_t}D\bigl(u,T_c(\cdot\mid h_t)\bigr)$,
which equals $T_c$ when $c$ is deterministic. Let $P_B(\cdot\mid x)$ be the
sequence distribution induced by $B$ and set $\alpha_x=P_B(S_x\mid x)$.

\begin{theorem}[Teacher projection induces excess KL drift]
\label{thm:sdpo_excess_kl}
Assume $0<\alpha_x<1$. Then
\begin{equation}
\begin{aligned}
&
D_{\mathrm{KL}}\bigl(P_B(\cdot\mid x)\parallel P_0(\cdot\mid x)\bigr)
-
D_{\mathrm{KL}}\bigl(Q_{\alpha_x,x}^{\dagger}\parallel P_0(\cdot\mid x)\bigr)\\
&\qquad =
\mathbb{E}_{y\sim P_B(\cdot\mid x)}
\sum_{t=1}^{|y|}
D_{\mathrm{KL}}
\bigl(
B(\cdot\mid h_t)
\parallel
q_{\alpha_x,x}^{\dagger}(\cdot\mid h_t)
\bigr)
\ge 0 .
\end{aligned}
\label{eq:sdpo_excess_kl}
\end{equation}
The inequality is strict whenever the SDPO teacher target differs from the
sequence-reward Razor kernel on a set of histories with positive
$P_B$-probability.
\end{theorem}

% The no-excess-drift case is therefore restrictive: SDPO must satisfy
% $B(\cdot\mid h_t)=q_{\alpha_x,x}^{\dagger}(\cdot\mid h_t)$ at every visited
% history. This is not implied by teacher--student parameter sharing. Because the
% teacher conditions on augmented context $c$, it may select a particular
% reasoning trace, formatting pattern, or artifact even when many continuations
% are equally successful. Thus, SDPO's on-policy sampling controls where
% supervision is applied, but the distribution selected at each token is the
% teacher projection, not the KL-minimal sequence-reward projection. Combined
% with the KL-forgetting law, the excess drift in
% \Cref{eq:sdpo_excess_kl} predicts additional forgetting. Full proofs are in \Cref{appdx:proofs}.

Equality is therefore restrictive: the teacher target must coincide with the
Razor kernel at every visited history, which is not implied by
teacher--student parameter sharing. Because the teacher conditions on $c$, it
may select a particular reasoning trace, formatting pattern, or artifact even
when many continuations are equally successful. Thus, SDPO's on-policy
sampling controls where supervision is applied, but the distribution selected
at each token is the teacher projection, not the KL-minimal sequence-reward
projection. Combined with the KL-forgetting law, the excess drift in
\Cref{eq:sdpo_excess_kl} predicts additional forgetting; \Cref{appdx:proofs}
provides full proofs, a matched forward-KL analogue, and a corollary
formalizing this prediction.

\begin{takeawaybox}
On-policy data alone does not guarantee forgetting mitigation. Although SDPO trains on on-policy histories, each token is projected toward a privileged teacher rather than a KL-minimal successful policy; this can amplify artifacts and induce excess drift across continual post-training.
\end{takeawaybox}

% \section{SDPO vs OPD}
% \section{SDPO vs SFT (Instruct)}
\section{Conclusion}

This paper studied whether on-policy self-distillation improves continual post-training. 
Our results indicate that SDPO can effectively specialize when teacher signals are stable and aligned with the target domain. 
However, SDPO is not uniformly beneficial for continual learning.
Its token supervision mechanism can amplify both useful signals and harmful artifacts, making it sensitive to teacher quality, update schedules, and token structure. Compared with GRPO, SDPO variants exhibit greater forgetting, instability, parameter drift, and response-manifold shifts.

The main lesson is not that on-policy self-distillation methods are ineffective, but that they are not automatically safe. Token supervision can be a rapid specialization signal, but it can also be dangerous: each token can reinforce useful signals and harmful artifacts. Sequence-level methods such as GRPO learn more slowly, but in our experiments, they preserve capabilities more reliably. Thus, the future of on-policy self-distillation for continual learning likely depends not on making the signal denser, but on safeguarding it through token weighting, teacher controls, or data balancing strategies, thereby making it more selective, stable, and artifact-aware.

\section*{AI Use Statement}

In this work, we used generative AI tools for proving mathematical claims, namely to help formulate and verify mathematical claims and to assist in the writing of proofs. Additionally, we used generative AI tools to aid and polish writing, in particular to edit the paper to improve readability. We have not used generative AI tools for the other tasks with required disclosure (such as generating synthetic data sets, implementing methods, or interpreting results), and the remaining required disclosure tasks are not applicable to this work. 

We have reviewed all AI-assisted work: LLM-assisted mathematical claims and proofs were checked and verified for correctness by the authors, and all AI-polished text was reviewed and edited by the authors. We take responsibility for the final content of this work, including text, claims, or artifacts produced with the aid of generative AI.

\bibliography{citations}
\bibliographystyle{iclr2027_conference}

\clearpage
\appendix
\section{Related Work}

\subsection{Post-Training of Large Language Models}

\paragraph{Reinforcement Learning for Post-Training.}
Reinforcement learning is widely used to post-train large language models with preference signals or reward functions derived from human feedback, automated evaluators, or verifiable correctness criteria \citep{stiennon2020learning,ouyang2022training,ziegler2019fine,bai2022constitutional}.
Recent methods, such as direct preference optimization \citep{rafailov2023dpo}, recast preference optimization as a supervised objective, while on-policy algorithms like group relative policy optimization \citep{shao2024deepseekmath} optimize relative advantages over sampled response groups, often improving stability and scalability in large-scale settings.

\paragraph{On-Policy Distillation.}
On-policy distillation (OPD) methods construct training targets from samples generated by the current policy or its improved variants, thereby converting reinforcement signals into supervised learning updates \citep{agarwal2024policy,gu2024minillm}. 
Recent self-distillation methods remove the need for a separate teacher by running
the same model under different information conditions: the student observes only
the original prompt, while the teacher is augmented with privileged context \(c\).
The privileged signal can come from demonstrations or examples
\citep{snell2022learning,shenfeld2026self}, correct solutions or textual feedback
\citep{hubotter2026reinforcement}, or revision contexts, hints, and task-level
feedback transformations \citep{zhao2026self,ye2026policy}.

\paragraph{Chain-of-Thought Reasoning.}
Chain-of-thought (CoT) prompting and training encourage LLMs to generate intermediate reasoning before final answers, often improving performance on arithmetic, symbolic, and commonsense tasks \citep{wei2022chain,kojima2022large}. 
Under final-answer supervision, this lengthens the credit-assignment horizon; with step-level or process rewards, it offers denser feedback and improves interpretability \citep{uesato2022solving,lightman2023lets}. CoT also has drawbacks: it increases inference cost, may propagate errors across steps, and the generated rationales need not faithfully reflect the model's actual decision process \citep{turpin2023language}.

\paragraph{Continual Post-Training.}
Continual learning \citep{guo2025comprehensive} studies how models incorporate new knowledge without catastrophic forgetting of those already acquired.
Common approaches use regularization, replay, or reference-model constraints to mitigate catastrophic forgetting across training stages \citep{kirkpatrick2017overcoming, zhu2021prototype, liu2025c}.
A growing line of work shows that RFT can substantially mitigate forgetting in continual post-training
\citep{lai2025reinforcement,shenfeld2025rl,chen2025retaining}.
For example, \citet{lai2025reinforcement} find that this benefit is not mainly due to explicit KL penalties or chain-of-thought prompting, but rather to the implicit regularization induced by reward-guided policy updates.
Complementarily, \citet{shenfeld2025rl} argue that online RL is biased toward KL-minimal solutions among the policies that solve a new task, thereby reducing the distributional shift from the base model that is predictive of forgetting.
Similarly, \citet{chen2025retaining} identify the use of on-policy data as a key factor behind RL's robustness.

\subsection{Group Relative Policy Optimization}

Group relative policy optimization (GRPO, \citealp{shao2024deepseekmath}) assigns rewards at the sequence level.
Given a prompt $x \sim \mathcal{D}$, GRPO samples a group of $G$ responses $\{y_i\}_{i=1}^G$ from the current policy, i.e., $y_i \sim \pi_{\theta_{\mathrm{old}}}(\cdot \mid x)$, evaluates each response with a reward model to obtain rewards $\{r_i\}_{i=1}^G$, and computes group-relative advantages by normalizing rewards within the group, e.g., $A_i = (r_i-\mathrm{mean}(\{r_j\}_{j=1}^G))/(\mathrm{std}(\{r_j\}_{j=1}^G)+\epsilon)$. Training optimizes a PPO-style clipped objective:
\begin{equation*}
\mathcal{J}_{\mathrm{GRPO}}(\theta)
=
\frac{1}{G}\sum_{i=1}^G
\frac{1}{|y_i|}\sum_{t=1}^{|y_i|}
\min\bigl(
\rho_{i,t} A_i,
\mathrm{clip}(\rho_{i,t},1-\varepsilon,1+\varepsilon)A_i
\bigr)
-
\lambda D_{\mathrm{KL}}\bigl(\pi_\theta(\cdot\mid x)\parallel\pi_{\mathrm{ref}}(\cdot\mid x)\bigr),
\label{eq:grpo}
\end{equation*}
where $\rho_{i,t}=\nicefrac{\pi_\theta(y_{i,t}\mid x,y_{i,<t})}{\pi_{\theta_{\mathrm{old}}}(y_{i,t}\mid x,y_{i,<t})}$ and $\pi_{\mathrm{ref}}$ is a reference policy used for KL regularization. In essence, GRPO is a critic-free policy optimization method that stabilizes reinforcement learning by using relative rewards among multiple sampled responses to the same prompt.

\subsection{On-Policy Self-Distillation}

Self-distillation policy optimization (SDPO, \citealp{hubotter2026reinforcement}) instead provides dense token-level rewards using the same model as both teacher and student. 
For each student-generated response $y\sim\pi_{\theta}(\cdot\mid x)$, the student is trained to align with the more-informed teacher $\pi_{\theta}(\cdot\mid x,c)$ having richer context $c$ (including solutions or other auxiliary signals).
Training minimizes the generalized Jensen--Shannon divergence (JSD, \Cref{eq:jsd}) between student and teacher distributions:
\begin{equation*}
\mathcal{L}_{\mathrm{SDPO}}(\theta) = 
\sum_{t} D_\mathrm{JSD}(\beta) 
\left(
\pi_{\theta}(\cdot \mid x, y_{<t})
\parallel 
\mathrm{stopgrad}\left(\pi_{\theta}(\cdot \mid x, c, y_{<t})\right)
\right).
\end{equation*}
This yields dense supervision: instead of a single reward for the whole sequence, SDPO encourages the student distribution at each generated token to align with the more informed teacher distribution.

\section{Math Notations}

\paragraph{Kullback--Leibler (KL) divergence.}
For two discrete distributions $P(\mathcal{X})$ and $Q(\mathcal{X})$,
\begin{equation}
D_{\mathrm{KL}}(P \parallel Q)
=
\sum_{x\in \mathcal{X}} P(x)\log\!\tfrac{P(x)}{Q(x)}.
\label{eq:kl}
\end{equation}

\paragraph{Generalized Jensen--Shannon (JS) divergence.} A weighted average of two KL divergences from $P$ and $Q$ to their mixture with $\beta\in(0,1)$ ($\beta=0.5$ recovers the standard symmetric JS divergence):
\begin{equation}
D_\mathrm{JSD}(\beta)(P\parallel Q) =
\beta {D_\mathrm{KL}(P \parallel\beta P + (1 - \beta)Q)}+
(1-\beta) D_\mathrm{KL}(Q \parallel\beta P + (1 - \beta)Q).
\label{eq:jsd}
\end{equation}

\paragraph{Maximum Mean Discrepancy.} It measures the difference between distributions \(P\) and \(Q\) by comparing their kernel mean embeddings in a reproducing kernel Hilbert space (RKHS): 
\begin{align}\operatorname{MMD}^2(P,Q)&=\left\|\mathbb{E}_{x\sim P}[\phi(x)]-\mathbb{E}_{y\sim Q}[\phi(y)]\right\|_{\mathcal{H}}^2\\&=\mathbb{E}_{x,x'\sim P}[k(x,x')]+\mathbb{E}_{y,y'\sim Q}[k(y,y')]-2\mathbb{E}_{x\sim P,y\sim Q}[k(x,y)].\label{eq:mmd}\end{align}

\paragraph{Singular Value Decomposition (SVD).}
For each LLM weight matrix at training step $k$: $\Wk \in \mathbb{R}^{m \times n}$, we compute the compact singular value decomposition:
\begin{equation}
\Wk = \Uk \Sk {\Vk}^{\top},\label{eq:svd}
\end{equation}
where $\Uk \in \mathbb{R}^{m \times d}$, $\Vk \in \mathbb{R}^{n \times d}$, and
$d=\min(m,n)$.

\paragraph{Performance Change.}
We measure absolute ($\Delta$) and relative ($\delta$) change against the base model:
\begin{align}
\Delta_{m,\mathrm{train},\mathrm{eval}}={P_{m,\mathrm{train},\mathrm{eval}}}-{P_{\mathrm{base},\mathrm{eval}}},\label{eq:performanceDelta}\\
\delta_{m,\mathrm{train},\mathrm{eval}}=\nicefrac{\Delta_{m,\mathrm{train},\mathrm{eval}}}{P_{\mathrm{base},\mathrm{eval}}}\times 100\%,\label{eq:performancedelta}
\end{align}
where $m$ is the training method, $\mathrm{train}$ is the training domain, 
and $\mathrm{eval}$ is the evaluation benchmark. We define \emph{gain} $G$ and \emph{forgetting} $F$ as the positive and negative part of the relative change:
\begin{align}
G_{m,\mathrm{train},\mathrm{eval}}=\max(0,\delta_{m,\mathrm{train},\mathrm{eval}}),\label{eq:gain}\\
F_{m,\mathrm{train},\mathrm{eval}}=\max(0,-\delta_{m,\mathrm{train},\mathrm{eval}}).\label{eq:forgetting}
\end{align}

\paragraph{Intersection-over-Union.} For two masks $A$ and $B$ (e.g., predicted mask and ground-truth mask),
\begin{equation}
\label{eq:iou}
\IoU(A,B) = \nicefrac{|A\cap B|}{|A\cup B|},
\end{equation}
where $|\cdot|$ denotes cardinality. The quantity lies in $[0,1]$, with $1$ for a perfect overlap.

\paragraph{Normalized spectral shift.}
It is defined as the $\ell_2$-normalized drift of the singular-value profile:
\begin{equation}
\label{eq:nss}
\mathrm{NSS}\left(\Wz,\Wk\right) = 
\tfrac{\norm{\sval{\Wz}-\sval{\Wk}}_2}
{\norm{\sval{\Wz}}_2}.
\end{equation}

\paragraph{Principal-mask (PM) and Principal-angles (PAs).}
Given $\Wk_r = \Uk_r \Sk_r {\Vk_r}^{\top}$ as the rank-$r$ approximation, PM is
\begin{equation}
\label{eq:pm}
\kth{\PM}{k}[i,j]=
\mathbb{I}
\left(
\left| \Wk_r[i,j] \right| \geq \kth{\tau}{k}_{\eta}
\right),
\end{equation}
where $\kth{\tau}{k}_{\eta}$ is a threshold controlled by the hyperparameter $\eta$.
Principal angles are
\begin{equation}
\label{eq:pa}
\theta_i(\Wz,\Wk)
=
\arccos \sigma_i\left({\Uz_r}^{\top}\Uk_r\right),
\qquad i=1,\dots,k,
\end{equation}
where $\sigma_i(\cdot)$ denotes the $i$-th singular value.

\paragraph{Maximum principal-angle rotation.}
Max PA of the top-$r$ left singular vectors:
\begin{equation}
\label{eq:maxpa}
\mathrm{maxPA}(\Wz,\Wk)
=
\max_{i \leq r} \theta_i(\Wz,\Wk).
\end{equation}

\paragraph{Principal-mask overlap.}
Intersection-over-Union between the initial and current PM masks:
\begin{equation}
\label{eq:pmoverlap}
\mathrm{PMOverlap}(\Wz,\Wk)
=\IoU(\kth{\PM}{0},\kth{\PM}{k}).
\end{equation}

\paragraph{Changed-parameter-mask.} 
Per-parameter mask with change threshold $\eta>0$:
\begin{equation}
\label{eq:cm}
\kth{\CM}{k}[i,j]
=
\mathbb{I}
\left(
\left| (\Wk-\Wz)[i,j] \right| \geq \eta
\right).
\end{equation}

\paragraph{Principal-mask update sparsity.}
Intersection-over-Union between the initial
principal mask and the changed-parameter mask:
\begin{equation}
\label{eq:cmoverlap}
\mathrm{PMUpdateSparsity}(\Wz,\Wk)
=
\IoU{
(\kth{\PM}{0}, \kth{\CM}{k})
}.
\end{equation}

\section{Proofs}
\label{appdx:proofs}

We provide the proofs for results in \Cref{sec:sdpo_vs_grpo:theory}. All distributions are over
finite spaces. We assume that the base sequence distribution
$P_0(\cdot\mid x)$ has full support. When the prompt $x$ is fixed, 
we suppress it from notation.

\paragraph{Matched-success Razor policy.}
Let $S=S_x$ be the set of successful responses and let
$z=P_0(S\mid x)\in(0,1)$. For a target success level
$\alpha\in(0,1)$, define
\[
\mathcal{Q}_{\alpha}
=
\{Q\in\Delta(\mathcal{Y}): Q(S)=\alpha\}.
\]

\begin{lemma}[Closed form of the matched-success Razor policy]
\label{lem:matched_razor_closed_form}
The solution of
\[
Q_{\alpha,x}^{\dagger}
=
\arg\min_{Q\in\mathcal{Q}_{\alpha}}
D_{\mathrm{KL}}\bigl(Q\parallel P_0(\cdot\mid x)\bigr)
\]
is
\begin{equation}
Q_{\alpha,x}^{\dagger}(y)
=
P_0(y\mid x)
\left[
\tfrac{\alpha}{z}\mathbf{1}\{y\in S_x\}
+
\tfrac{1-\alpha}{1-z}\mathbf{1}\{y\notin S_x\}
\right].
\label{eq:app_matched_razor}
\end{equation}
Equivalently,
\[
Q_{\alpha,x}^{\dagger}
=
\alpha P_0(\cdot\mid x,S_x)
+
(1-\alpha)P_0(\cdot\mid x,S_x^c).
\]
\end{lemma}

\begin{proof}
Fix $x$ and write $p=P_0(\cdot\mid x)$. Let
$p^+=p(\cdot\mid S)$ and $p^-=p(\cdot\mid S^c)$. Any
$Q\in\mathcal{Q}_{\alpha}$ can be decomposed as
\[
Q=\alpha Q^+ +(1-\alpha)Q^-,
\]
where $Q^+=Q(\cdot\mid S)$ and $Q^-=Q(\cdot\mid S^c)$. Then
\[
\begin{aligned}
D_{\mathrm{KL}}(Q\parallel p)
&=
\alpha\log\frac{\alpha}{z}
+
(1-\alpha)\log\tfrac{1-\alpha}{1-z}+
\alpha D_{\mathrm{KL}}(Q^+\parallel p^+)
+
(1-\alpha)D_{\mathrm{KL}}(Q^-\parallel p^-).
\end{aligned}
\]
The first two terms are constant over $\mathcal{Q}_{\alpha}$. The remaining
terms are nonnegative and vanish if and only if $Q^+=p^+$ and $Q^-=p^-$.
Thus the unique minimizer is
\[
Q_{\alpha,x}^{\dagger}
=
\alpha p^+ +(1-\alpha)p^-,
\]
which is exactly \Cref{eq:app_matched_razor}.
\end{proof}

\begin{lemma}[Token kernel induced by the Razor policy]
\label{lem:razor_token_kernel}
Let $q_{\alpha,x}^{\dagger}(\cdot\mid h_t)$ be the autoregressive token kernel
induced by $Q_{\alpha,x}^{\dagger}$ at history $h_t=(x,y_{<t})$. Define
\begin{equation}
C_\alpha(h)
=
\mathbb{E}_{y\sim P_0(\cdot\mid h)}
\left[
\tfrac{\alpha}{z}\mathbf{1}\{y\in S_x\}
+
\tfrac{1-\alpha}{1-z}\mathbf{1}\{y\notin S_x\}
\right].
\label{eq:app_C_alpha}
\end{equation}
Then
\begin{equation}
q_{\alpha,x}^{\dagger}(a\mid h_t)
=
\pi_0(a\mid h_t)
\frac{C_\alpha(h_ta)}{C_\alpha(h_t)}.
\label{eq:app_razor_token_kernel}
\end{equation}
In the exact-success limit $\alpha\to1$,
\begin{equation}
q_{1,x}^{\dagger}(a\mid h_t)
=
\pi_0(a\mid h_t)
\frac{
\Pr_{P_0}(y\in S_x\mid h_ta)
}{
\Pr_{P_0}(y\in S_x\mid h_t)
}.
\label{eq:app_exact_success_kernel}
\end{equation}
\end{lemma}

\begin{proof}
By \Cref{lem:matched_razor_closed_form},
\[
Q_{\alpha,x}^{\dagger}(y)=P_0(y\mid x)w_\alpha(y),
\]
where
\[
w_\alpha(y)
=
\frac{\alpha}{z}\mathbf{1}\{y\in S_x\}
+
\frac{1-\alpha}{1-z}\mathbf{1}\{y\notin S_x\}.
\]
For a history $h_t=(x,y_{<t})$,
\[
\begin{aligned}
q_{\alpha,x}^{\dagger}(a\mid h_t)
&=
\Pr_{Q_{\alpha,x}^{\dagger}}(y_t=a\mid h_t) =
\frac{
\sum_{y_{>t}}P_0(a,y_{>t}\mid h_t)w_\alpha(y)
}{
\sum_{\tilde y_{\ge t}}P_0(\tilde y_{\ge t}\mid h_t)w_\alpha(\tilde y)
}.
\end{aligned}
\]
The numerator equals
\[
\pi_0(a\mid h_t)
\mathbb{E}_{y\sim P_0(\cdot\mid h_ta)}[w_\alpha(y)]
=
\pi_0(a\mid h_t)C_\alpha(h_ta),
\]
and the denominator equals $C_\alpha(h_t)$. This proves
\Cref{eq:app_razor_token_kernel}. Letting $\alpha\to1$ gives
\Cref{eq:app_exact_success_kernel}.
\end{proof}

\paragraph{Teacher projection and excess reverse KL.}

For a strict distillation divergence $D$, the population SDPO token target is
\[
B(\cdot\mid h_t)
\in
\arg\min_{u\in\Delta(\mathcal{V})}
\mathbb{E}_{c\mid h_t}
D\bigl(u,T_c(\cdot\mid h_t)\bigr).
\]
When $c$ is deterministic and $D$ is generalized JSD, the minimizer is
$B(\cdot\mid h_t)=T_c(\cdot\mid h_t)$ because generalized JSD is nonnegative
and equals zero if and only if its two arguments are identical.

Let $P_B(\cdot\mid x)$ be the sequence distribution induced by
$B(\cdot\mid h_t)$, and let
\[
\alpha=P_B(S_x\mid x).
\]
The corresponding matched-success Razor policy is
$Q_{\alpha,x}^{\dagger}$.

\begin{proof}[Proof of \Cref{thm:sdpo_excess_kl}]
Fix $x$ and suppress $x$ from notation. Let $p=P_0(\cdot\mid x)$,
$S=S_x$, $z=p(S)$, and $Q^\dagger=Q_{\alpha,x}^{\dagger}$. 
By \Cref{lem:matched_razor_closed_form},
\[
Q^\dagger(y)=p(y)w_\alpha(y),
\]
where $w_\alpha(y)=\alpha/z$ on $S$ and
$w_\alpha(y)=(1-\alpha)/(1-z)$ on $S^c$.

Since $P_B(S)=Q^\dagger(S)=\alpha$,
\begin{equation}
\sum_y P_B(y)\log w_\alpha(y)
=
\sum_y Q^\dagger(y)\log w_\alpha(y).
\label{eq:app_same_weight_expectation}
\end{equation}
Therefore,
\begin{equation}
\begin{aligned}
D_{\mathrm{KL}}(P_B\parallel p)
-
D_{\mathrm{KL}}(Q^\dagger\parallel p)
&=
\sum_y P_B(y)\log\frac{P_B(y)}{p(y)}
-
\sum_y Q^\dagger(y)\log\frac{Q^\dagger(y)}{p(y)} \\
&=
\sum_y P_B(y)\log\frac{P_B(y)}{p(y)}
-
\sum_y Q^\dagger(y)\log w_\alpha(y) \\
&=
\sum_y P_B(y)\log\frac{P_B(y)}{p(y)w_\alpha(y)} \\
&=
D_{\mathrm{KL}}(P_B\parallel Q^\dagger),
\end{aligned}
\label{eq:app_excess_reverse}
\end{equation}
where we used \Cref{eq:app_same_weight_expectation} in the third line.

Now write both sequence distributions autoregressively:
\[
P_B(y\mid x)=\prod_{t=1}^H B(y_t\mid h_t),
\qquad
Q^\dagger(y\mid x)=\prod_{t=1}^H q_{\alpha,x}^{\dagger}(y_t\mid h_t).
\]
Then the autoregressive KL chain rule gives
\begin{equation}
\begin{aligned}
D_{\mathrm{KL}}(P_B\parallel Q^\dagger)
&=
\mathbb{E}_{y\sim P_B}
\log
\frac{P_B(y\mid x)}{Q^\dagger(y\mid x)} \\
&=
\mathbb{E}_{y\sim P_B}
\sum_{t=1}^H
\log
\frac{B(y_t\mid h_t)}
{q_{\alpha,x}^{\dagger}(y_t\mid h_t)} \\
&=
\mathbb{E}_{y\sim P_B}
\sum_{t=1}^H
D_{\mathrm{KL}}
\bigl(
B(\cdot\mid h_t)
\parallel
q_{\alpha,x}^{\dagger}(\cdot\mid h_t)
\bigr).
\end{aligned}
\label{eq:app_token_chain}
\end{equation}
Combining \Cref{eq:app_excess_reverse} and \Cref{eq:app_token_chain} proves
the theorem. The right-hand side is nonnegative, and it is strictly positive
whenever the two token kernels differ on a set of histories with positive
$P_B$-probability.
\end{proof}

\paragraph{Forward KL counterpart.}

The previous theorem uses reverse KL because RL's Razor is naturally stated as
a reverse-KL projection. The empirical forgetting law is often measured by
forward KL from the base policy to the fine-tuned policy, so we also give the
matched-success forward-KL analogue.

\begin{proposition}[Forward-KL excess at matched success]
\label{prop:forward_excess}
Fix $x$ and let $p=P_0(\cdot\mid x)$. Let $S=S_x$, $z=p(S)$, and let
$P_B(S)=\alpha\in(0,1)$. Define
\[
p^+=p(\cdot\mid S),
\qquad
p^-=p(\cdot\mid S^c),
\]
and
\[
B^+=P_B(\cdot\mid S),
\qquad
B^-=P_B(\cdot\mid S^c).
\]
Then
\begin{equation}
\begin{aligned}
&
D_{\mathrm{KL}}(p\parallel P_B)
-
D_{\mathrm{KL}}(p\parallel Q_{\alpha,x}^{\dagger})
 =
zD_{\mathrm{KL}}(p^+\parallel B^+)
+
(1-z)D_{\mathrm{KL}}(p^-\parallel B^-)
\ge0 .
\end{aligned}
\label{eq:app_forward_excess}
\end{equation}
\end{proposition}

\begin{proof}
Decompose $P_B$ as
\[
P_B=\alpha B^+ +(1-\alpha)B^-.
\]
Then
\[
\begin{aligned}
D_{\mathrm{KL}}(p\parallel P_B)
&=
z\log\frac{z}{\alpha}
+
(1-z)\log\frac{1-z}{1-\alpha} +
zD_{\mathrm{KL}}(p^+\parallel B^+)
+
(1-z)D_{\mathrm{KL}}(p^-\parallel B^-).
\end{aligned}
\]
For $Q_{\alpha,x}^{\dagger}=\alpha p^+ +(1-\alpha)p^-$, the two conditional KL
terms vanish:
\[
D_{\mathrm{KL}}(p\parallel Q_{\alpha,x}^{\dagger})
=
z\log\tfrac{z}{\alpha}
+
(1-z)\log\tfrac{1-z}{1-\alpha}.
\]
Subtracting gives \Cref{eq:app_forward_excess}.
\end{proof}

\paragraph{Connection to forgetting.}

We use the KL-forgetting law as an empirical bridge between distributional
drift and prior-task forgetting.

\begin{assumption}[KL-forgetting law]
\label{assump:kl_forgetting}
For policies matched in new-task success, let $\mathsf{F}(\pi)$ denote the
drop in prior-task performance. Suppose there exists a strictly increasing
function $\omega$ with $\omega(0)=0$ such that
\[
\mathcal{K}_{\mathrm{fwd}}(\pi_1)
-
\mathcal{K}_{\mathrm{fwd}}(\pi_2)
\ge\varepsilon
\quad\Longrightarrow\quad
\mathsf{F}(\pi_1)-\mathsf{F}(\pi_2)
\ge
\omega(\varepsilon),
\]
where
\[
\mathcal{K}_{\mathrm{fwd}}(\pi)
=
\mathbb{E}_{x\sim\mathcal{D}}
D_{\mathrm{KL}}
\bigl(
P_0(\cdot\mid x)
\parallel
P_\pi(\cdot\mid x)
\bigr).
\]
\end{assumption}

\begin{corollary}[Teacher mismatch implies additional forgetting]
\label{cor:teacher_mismatch_forgetting}
Let $P_B$ be the SDPO teacher-projected sequence distribution and let
$Q_{\alpha,x}^{\dagger}$ be the matched-success Razor policy with
$\alpha=P_B(S_x\mid x)$. Define
\[
\begin{aligned}
\Delta_{\mathrm{fwd}}
=
\mathbb{E}_{x\sim\mathcal{D}}
\Bigl[
&z_xD_{\mathrm{KL}}
\bigl(
P_0(\cdot\mid x,S_x)
\parallel
P_B(\cdot\mid x,S_x)
\bigr)\\
&+(1-z_x)D_{\mathrm{KL}}
\bigl(
P_0(\cdot\mid x,S_x^c)
\parallel
P_B(\cdot\mid x,S_x^c)
\bigr)
\Bigr].
\end{aligned}
\]
If $\Delta_{\mathrm{fwd}}>0$, then under \Cref{assump:kl_forgetting},
\[
\mathsf{F}(P_B)
-
\mathsf{F}(Q_{\alpha}^{\dagger})
\ge
\omega(\Delta_{\mathrm{fwd}})
>
0.
\]
\end{corollary}

\begin{proof}
Taking expectation over $x\sim\mathcal{D}$ in \Cref{prop:forward_excess} gives
\[
\mathcal{K}_{\mathrm{fwd}}(P_B)
-
\mathcal{K}_{\mathrm{fwd}}(Q_{\alpha}^{\dagger})
=
\Delta_{\mathrm{fwd}}.
\]
The result follows directly from \Cref{assump:kl_forgetting}.
\end{proof}

\section{Training Details}

\subsection{Models and Datasets}

We adopt Qwen3 series (Qwen3-4B-Instruct-2507 and Qwen3-4B-Thinking-2507, \citealp{qwen3technicalreport}), Olmo-3 series (Olmo-3-7B-Instruct-DPO and Olmo-3-7B-Think-DPO, \citealp{olmo2025olmo3}) as base models\footnote{For Instruct models, SDPO and \sdpocot{} are equivalent due to absence of think tags.} and  train\footnote{Training framework from SDPO official implementation \url{https://github.com/lasgroup/SDPO}.}
on 4 diverse tasks---\dmath, \dchem, \dtool, and \dcode, with 
in-distribution and out-of-distribution\footnote{In-distribution and out-of-distribution are pragmatic labels, not strict measures of inter-dataset distance.}
test sets to evaluate model performance
measured in averaged accuracy (i.e., Acc@k). 
Data sources and training hyperparameters (learning rates according to \citealp{hubotter2026reinforcement}) are disclosed in 
\Cref{appdx:tab:datasets,appdx:tab:hparams}, respectively.
We visualize the embeddings of each benchmark using Qwen3-Embedding-4B
\citep{qwen3embedding} in \Cref{appdx:fig:data}.

\begin{table}[htb]
\centering \setlength{\tabcolsep}{10pt}
\caption{Data sources for training or evaluation purposes across domains. 
\textit{AIME} denotes AIME24 and AIME25 combined.
Abbreviations \textit{ZLogic}, \textit{MMLU-R} denote ZebraLogicBench, MMLU-Redux 2.0, respectively.}
\label{appdx:tab:datasets}
\begin{tabular}{l|c|c|c}
\toprule
\multicolumn{1}{c}{} & \multicolumn{1}{c}{} & \multicolumn{2}{c}{\textbf{Evaluation}} \\
       % & Training & \multicolumn{2}{c}{Evaluation} \\
\cmidrule(lr){3-4}
\multicolumn{1}{l}{\textbf{Domain}} & \multicolumn{1}{c}{\textbf{Training}} & \multicolumn{1}{c}{In-Distribution} & Out-Of-Distribution \\
\midrule
\dmath{} & \makecell{DAPO-Math-17k\\\citep{yu2026dapo}}        & \makecell{AIME\\\citep{maa2025aime}} & 
\makecell{Math500\\\citep{lightman2023lets}} \\
\midrule
\dchem{}  & \multicolumn{2}{c|}{SciKnowEval~\citep{feng2024sciknoweval}} & \makecell{GPQA\\\citep{rein2023gpqa}} \\
\midrule
\dtool{}  & \multicolumn{2}{c|}{ToolAlpaca~\citep{tang2023toolalpaca}}   & \makecell{BFCLv4\\\citep{patil2025bfcl}}\\
\midrule
\dcode{}   & \multicolumn{2}{c|}{Dolci-Think-RL-7B~\citep{olmo2025olmo3}} & \makecell{LCBv6\\\citep{jain2025livecodebench}} \\
\midrule
\dataset{Logic}  & \multicolumn{2}{c|}{--} & \makecell{ZLogic\\\citep{lin2025zebralogic}} \\
\midrule
\dataset{Knowledge}  & \multicolumn{2}{c|}{--} & \makecell{MMLU-R\\\citep{gema2024mmlu}} \\
\bottomrule
\end{tabular}
\end{table}

\begin{figure}[htb]
\centering
\begin{subfigure}[b]{.54\textwidth}
\includegraphics[width=\textwidth]{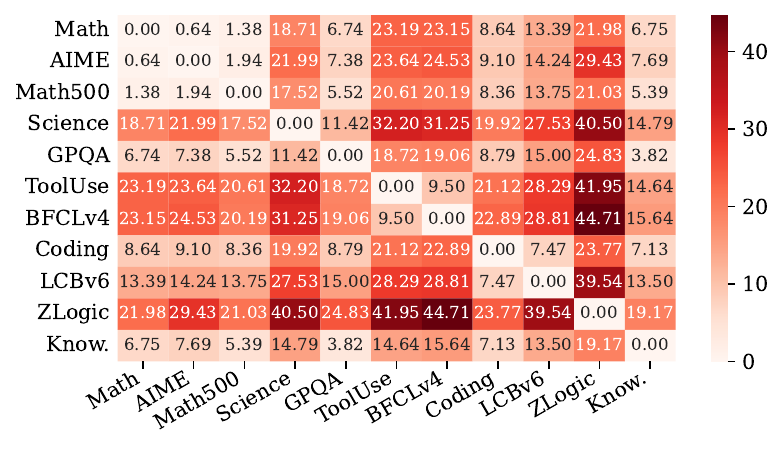}
\caption{Pairwise MMD (multiplier $10^{-2}$) between benchmarks.}
\label{appdx:fig:data:mmd}
\end{subfigure}
\begin{subfigure}[b]{.45\textwidth}
\includegraphics[width=\textwidth]{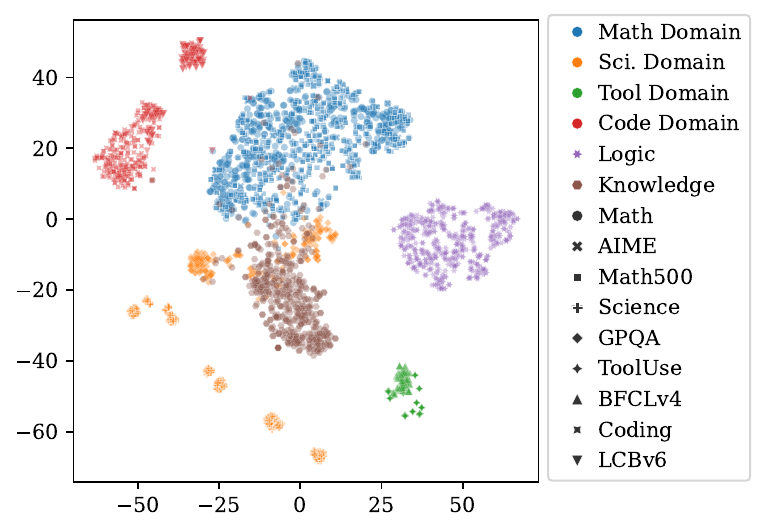}
\caption{Visualization using t-SNE.}
\label{appdx:fig:data:tsne}
\end{subfigure}
\caption{Embedding visualization of benchmarks by domain.}
\label{appdx:fig:data}
\end{figure}

\begin{table}[htb]
\centering \renewcommand{\arraystretch}{1.1}
\caption{Hyperparameters for SDPO and GRPO training with Qwen3-4B-Thinking-2507.}
\label{appdx:tab:hparams}
\begin{tabular}{cl|l|l}
\toprule
&\textbf{Parameters} & \textbf{SDPO} & \textbf{GRPO} \\
\midrule
\multirow{2}{*}{\rotatebox[origin=c]{90}{\textbf{\small Data}}}
&Max.\ prompt length   & $2048$ & $2048$ \\
&Max.\ response length & $14336$ & $14336$ \\
\midrule
\multirow{3}{*}{\rotatebox[origin=c]{90}{\textbf{\small Batching}}}
&Question batch size & $32$ & $8$ \\
&Mini batch size     & $32$ & $8$ \\
&Number of rollouts  & $8 $ & $8$ \\
\midrule
\multirow{4}{*}{\rotatebox[origin=c]{90}{\textbf{\small Rollout}}}
&Inference engine   & vLLM & vLLM \\
&Temperature        & $1.0$  & $1.0$ \\
&Top-$p$            & $-1 $  & $-1 $ \\
&Top-$k$            & $-1 $  & $-1 $ \\
% \midrule
% \multirow{4}{*}{\rotatebox[origin=c]{90}{\textbf{\small Validation}}}
% &Number of rollouts & $8   $ & $8   $ \\
% &Temperature        & $0.6 $ & $0.6 $ \\
% &Top-$p$            & $0.95$ & $0.95$ \\
% &Top-$k$            & $20  $ & $20  $ \\
\midrule
\multirow{4}{*}{\rotatebox[origin=c]{90}{\textbf{SDPO}}}
&Distillation Top-$K$             & $100$             & -- \\
&Distillation divergence          & Jensen--Shannon & -- \\
&Teacher-EMA update rate (\%)     & $\alpha\in\{0,0.5,1,2,5\}$& -- \\
&Rollout importance-sampling clip & $2$             & --  \\
\midrule
\multirow{3}{*}{\rotatebox[origin=c]{90}{\textbf{PPO}}}
&KL coefficient ($\lambda$)       & --              & 0.0 \\
&$\epsilon$-low                   & $0.20$ & $0.20$ \\
&$\epsilon$-high                  & $0.28$ & $0.28$ \\
\midrule
\multirow{6}{*}{\rotatebox[origin=c]{90}{\textbf{\small Training}}}
&Optimizer          & AdamW & AdamW \\
&Learning rate      & $10^{-5}$ & $10^{-6}$ \\
&Training steps     & $50$ ($20$ for \dcode{})  & $50$ ($20$ for \dcode{})   \\
&Warmup steps       & $10  $ & $10  $ \\
&Weight decay       & $0.01$ & $0.01$ \\
&Gradient Clip Norm & $1.0 $ & $1.0 $ \\
\bottomrule
\end{tabular}
\end{table}

\subsection{Data Samples}
\label{appdx:sec:data}
\begin{samplebox}{AIME}
\begin{Verbatim}[breaklines=true, breakanywhere=true]
Every morning Aya goes for a $9$-kilometer-long walk and stops at a coffee shop afterwards. When she walks at a constant speed of $s$ kilometers per hour, the walk takes her 4 hours, including $t$ minutes spent in the coffee shop. When she walks $s+2$ kilometers per hour, the walk takes her 2 hours and 24 minutes, including $t$ minutes spent in the coffee shop. Suppose Aya walks at $s+\frac{1}{2}$ kilometers per hour. Find the number of minutes the walk takes her, including the $t$ minutes spent in the coffee shop.
Please reason step by step, and put your final answer within \boxed{}.
\end{Verbatim}
\end{samplebox}

\begin{samplebox}{BFCLv4}
\begin{Verbatim}[breaklines=true, breakanywhere=true]
Your task is to answer the user's question using available tools.
You have access to the following tools:
Name: Tools
Description: A set of functions/tools that you can use to answer the user's question.
Documentation:
math.triangle_area_heron: Calculates the area of a triangle using Heron's formula, given the lengths of its three sides.
Parameters: {"side1": "integer. Length of the first side of the triangle.", "side2": "integer. Length of the second side of the triangle.", "side3": "integer. Length of the third side of the triangle."}
Output: Successful response.
 - Format: application/json
math.circle_area: Calculates the area of a circle given its radius.
Parameters: {"radius": "float. The radius of the circle."}
Output: Successful response.
 - Format: application/json
math.triangle_area_base_height: Calculates the area of a triangle using the formula (1/2)base*height.
Parameters: {"base": "float. The base length of the triangle.", "height": "float. The height of the triangle."}
Output: Successful response.
 - Format: application/json

Use the following format:
Thought: you should always think about what to do
Action: the action to take, should be one of the tool names.
Action Input: the input to the action, must be in JSON format. All of the action input must be realistic and from the user.

Begin!
Question: Calculate the area of a triangle, given the lengths of its three sides: 3, 4, and 5.
\end{Verbatim}
\end{samplebox}

\begin{samplebox}{\dcode{}}
\begin{Verbatim}[breaklines=true, breakanywhere=true]
You are tasked with creating a function that retrieves the details of a specific assignment based on its ID. The function should take a string representing the assignment ID as input and return a dictionary containing the assignment details. For the purpose of this question, you can assume that the assignment details will be represented as a fixed dictionary for demonstration. Implement a function named `get_assignment_details` that takes an assignment ID and returns the corresponding assignment details as a dictionary. If the assignment ID does not exist, return an empty dictionary. The assignment details for valid IDs are predefined as follows: '101': {'title': 'Math Homework', 'due_date': '2023-11-01'}, '102': {'title': 'Science Project', 'due_date': '2023-12-15'}, '103': {'title': 'History Essay', 'due_date': '2023-10-30'}.
\end{Verbatim}
\end{samplebox}

\begin{samplebox}{GPQA}
\begin{Verbatim}[breaklines=true, breakanywhere=true]
Two quantum states with energies E1 and E2 have a lifetime of 10^-9 sec and 10^-8 sec, respectively. We want to clearly distinguish these two energy levels. Which one of the following options could be their energy difference so that they can be clearly resolved?

A. 10^-8 eV
B. 10^-4 eV
C. 10^-11 eV
D. 10^-9 eV
Please reason step by step, and put your final answer within \boxed{}.
\end{Verbatim}
\end{samplebox}

\begin{samplebox}{\dmath{}}
\begin{Verbatim}[breaklines=true, breakanywhere=true]
In triangle $ABC$, $\sin \angle A = \frac{4}{5}$ and $\angle A < 90^\circ$. Let $D$ be a point outside triangle $ABC$ such that $\angle BAD = \angle DAC$ and $\angle BDC = 90^\circ$. Suppose that $AD = 1$ and that $\frac{BD}{CD} = \frac{3}{2}$. If $AB + AC$ can be expressed in the form $\frac{a\sqrt{b}}{c}$ where $a, b, c$ are pairwise relatively prime integers, find $a + b + c$.
Please reason step by step, and put your final answer within \boxed{}.
\end{Verbatim}
\end{samplebox}

\begin{samplebox}{Math500}
\begin{Verbatim}[breaklines=true, breakanywhere=true]
Convert the point $(0,3)$ in rectangular coordinates to polar coordinates.  Enter your answer in the form $(r,\theta),$ where $r > 0$ and $0 \le \theta < 2 \pi.$
Please reason step by step, and put your final answer within \boxed{}.
\end{Verbatim}
\end{samplebox}

\begin{samplebox}{MMLU-R}
\begin{Verbatim}[breaklines=true, breakanywhere=true]
Determine whether the polynomial in Z[x] satisfies an Eisenstein criterion for irreducibility over Q. x^2 - 12
A. Yes, with p=2.
B. Yes, with p=3.
C. Yes, with p=5.
D. No.
Please reason step by step, and put your final answer within \boxed{}.
\end{Verbatim}
\end{samplebox}

\begin{samplebox}{\dchem}
\begin{Verbatim}[breaklines=true, breakanywhere=true]
Which of the following molecules is the only correct reactant used in the synthesis of the product "Nc1ccn(Cc2ccc(Cl)cc2C(F)(F)F)n1" ?

A: O=C1c2ccccc2C(=O)N1c1ccn(Cc2ccc(Cl)cc2C(F)(F)F)n1
B: O=C1c2ccccc2C(Cl)N1c1ccc2ccc(Cl)nc2n1
C: O=C1c2ccccc2C(=O)N1Cc1ccc(Oc2ccc(Cl)c(C(F)(F)F)c2)cc1
D: O=C1c2ccccc2C(=O)N1CC=C(Cc1ccccc1)C(F)(F)F
Please reason step by step, and put your final answer within \boxed{}.
\end{Verbatim}
\end{samplebox}

\begin{samplebox}{\dtool}
\begin{Verbatim}[breaklines=true, breakanywhere=true]
Your task is to answer the user's question using available tools. 
You have access to the following tools:
Name: Axolotl
Description: Collection of axolotl pictures and facts
Documentation:
getRandomAxolotlImage: Retrieve a random axolotl image with information on the image source.
Parameters: {}
Output: Successful response.
 - Format: application/json
 - Structure: Object{url, source, description}
searchAxolotlImages: Search for axolotl images based on specific criteria such as color, gender, and size.
Parameters: {"color": "string. One of: [wild, leucistic, albino]. The color of the axolotl (e.g., 'wild', 'leucistic', 'albino', etc.).", "gender": "string. One of: [male, female]. The gender of the axolotl ('male', 'female').", "size": "string. One of: [small, medium, large]. The size of the axolotl ('small', 'medium', 'large').", "page": "integer. The page number for pagination purposes."}
Output: Successful response.
 - Format: application/json
 - Structure: Object{results: Array[Object{url, source, description}], pagination: Object{current_page, total_pages, total_results}}
getAxolotlFacts: Retrieve interesting facts about axolotls such as their habits, habitats, and physical characteristics.
Parameters: {"category": "string. One of: [habits, habitat, physical characteristics]. The category of facts to retrieve (e.g., 'habits', 'habitat', 'physical characteristics').", "limit": "integer. The maximum number of facts to return."}
Output: Successful response.
 - Format: application/json
 - Structure: Array[Object{fact, source}]

Use the following format:
Thought: you should always think about what to do
Action: the action to take, should be one of the tool names.
Action Input: the input to the action, must be in JSON format. All of the action input must be realistic and from the user.

Begin!
Question: Hey, can you show me a random picture of an axolotl?
\end{Verbatim}
\end{samplebox}

\begin{samplebox}{ZLogic}
\begin{Verbatim}[breaklines=true, breakanywhere=true]
There are 6 houses, numbered 1 to 6 from left to right, as seen from across the street. Each house is occupied by a different person. Each house has a unique attribute for each of the following characteristics:
 - Each person has a unique name: `Arnold`, `Peter`, `Eric`, `Alice`, `Bob`, `Carol`
 - People have unique favorite book genres: `biography`, `science fiction`, `fantasy`, `mystery`, `romance`, `historical fiction`
 - People have unique favorite sports: `baseball`, `basketball`, `swimming`, `volleyball`, `tennis`, `soccer`
 - People own unique car models: `honda civic`, `ford f150`, `tesla model 3`, `chevrolet silverado`, `bmw 3 series`, `toyota camry`

## Clues:
1. Eric is the person who loves mystery books.
2. The person who loves tennis is the person who loves fantasy books.
3. The person who loves soccer is directly left of the person who loves science fiction books.
4. There is one house between the person who owns a Honda Civic and the person who loves biography books.
5. Peter is somewhere to the right of Carol.
6. The person who loves tennis is in the first house.
7. The person who owns a Tesla Model 3 is somewhere to the right of the person who loves baseball.
8. Eric is somewhere to the left of the person who loves romance books.
9. The person who owns a Toyota Camry is somewhere to the right of the person who loves romance books.
10. The person who owns a BMW 3 Series is Peter.
11. The person who owns a BMW 3 Series is the person who loves basketball.
12. The person who owns a Tesla Model 3 is directly left of Arnold.
13. Alice and the person who loves volleyball are next to each other.
14. The person who loves historical fiction books is the person who loves soccer.
15. The person who owns a Chevrolet Silverado is not in the first house.
16. The person who loves science fiction books is directly left of the person who loves swimming.

What is Name of the person who lives in House 5?
A. Eric
B. Bob
C. Alice
D. Peter
E. Carol
F. Arnold
Please reason step by step, and put your final answer within \boxed{}.
\end{Verbatim}
\end{samplebox}

\subsection{Model Chat Templates}

\begin{tmplbox}{Qwen3-4B-Thinking-2507}
\begin{Verbatim}[breaklines=true, breakanywhere=true]
<|im_start|>user\n{prompt}<|im_end|>
<|im_start|>assistant\n<think>
\end{Verbatim}
\end{tmplbox}

\begin{tmplbox}{Qwen3-4B-Instruct-2507}
\begin{Verbatim}[breaklines=true, breakanywhere=true]
<|im_start|>user\n{prompt}<|im_end|>
<|im_start|>assistant\n
\end{Verbatim}
\end{tmplbox}

\begin{tmplbox}{OLMo-3-7B-Instruct}
\begin{Verbatim}[breaklines=true, breakanywhere=true]
<|im_start|>system\nYou are a helpful function-calling AI assistant. You do not currently have access to any functions. <functions></functions><|im_end|>
<|im_start|>user\n{prompt}<|im_end|>
<|im_start|>assistant\n
\end{Verbatim}
\end{tmplbox}

\begin{tmplbox}{OLMo-3-7B-Think}
\begin{Verbatim}[breaklines=true, breakanywhere=true]
<|im_start|>system\nYou are OLMo, a helpful function-calling AI assistant built by Ai2. Your date cutoff is November 2024, and your model weights are available at https://huggingface.co/allenai. You do not currently have access to any functions. <functions></functions><|im_end|>
<|im_start|>user\n{prompt}<|im_end|>
<|im_start|>assistant\n<think>
\end{Verbatim}
\end{tmplbox}

\subsection{Prompt Templates}
\label{appdx:sec:prompt}
\begin{promptbox}{Prompt for \dmath{} \& AIME \& Math500 \& \dchem{} \& GPQA}
\begin{Verbatim}[breaklines=true, breakanywhere=true]
{question}
Please reason step by step, and put your final answer within \boxed{}.
\end{Verbatim}
\end{promptbox}

\begin{promptbox}{Prompt for \dtool{} \& BFCLv4}
\begin{Verbatim}[breaklines=true, breakanywhere=true]
Your task is to answer the user's question using available tools. 
You have access to the following tools:
{tool}

Use the following format:
Thought: you should always think about what to do
Action: the action to take, should be one of the tool names.
Action Input: the input to the action, must be in JSON format. All of the action input must be realistic and from the user.

Begin!
Question: {question}
\end{Verbatim}
\end{promptbox}

\begin{promptbox}{SDPO Regeneration Prompt}
\begin{Verbatim}[breaklines=true, breakanywhere=true]
{question}
Please reason step by step, and put your final answer within \boxed{}.
Correct solution: {previously correct solution}
Correctly solve the original question.
\end{Verbatim}
\end{promptbox}

% \begin{promptbox}{Regeneration prompt with encouraging uncertainty}
% \begin{Verbatim}[breaklines=true, breakanywhere=true]
% {question}
% Please reason step by step, and put your final answer within \boxed{}.
% Correct solution: {previously correct solution}
% Solve the original question from scratch with a full, detailed reasoning trace, without skipping any steps, and include as much uncertainty or ambiguity in the explanation.
% \end{Verbatim}
% \end{promptbox}

\section{Supplementary Results}

\begin{table}[t]
\centering\small
\caption{SDPO performance (Acc@8) against teacher EMA $\alpha$. No $\alpha$ is globally best.}
\label{appdx:tab:sdpoema}
\begin{tabular}{cccccccccccccc}
\toprule
Trained on & $\alpha (\%)$ & AIME         & Math500        & GPQA           & MMLU-R         & BFCLv4         & LCBv6          & Avg        \\
\midrule
--            & --         & 32.71          & 73.85          & 62.94          & 86.02          & 76.25          & 46.73          & 63.08 \\
\midrule
\multirow{5}{*}{\dmath}
& $0$           & \bronze{56.42} & \bronze{84.29} & \silver{58.72} & \gold{85.47}   & \bronze{77.71} & 48.52          & \silver{68.61} \\
& $0.5$         & \gold{57.92}   & \gold{84.52}   & 57.01          & 84.65          & 76.46          & \silver{51.10} & \silver{68.61} \\
& $1$           & \silver{56.04} & \silver{84.35} & \bronze{58.21} & \silver{85.31} & \bronze{77.71} & \gold{51.79}   & \gold{69.06} \\
& $2$           & 52.71          & 84.02          & \gold{59.53}   & \bronze{85.00} & \gold{78.75}   & 50.35          & 68.39 \\
& $5$           & 34.38          & 78.97          & 55.87          & 83.82          & \silver{78.12} & \bronze{50.86} & 63.67 \\
\midrule
\multirow{5}{*}{\dtool{}}
& $0$           & \gold{47.71}   & \gold{78.88}   & \silver{60.73} & \gold{87.82}   & \silver{76.04} & \gold{58.02}   & \gold{68.20} \\
& $0.5$         & 43.13          & 77.08          & \gold{61.49}   & \silver{87.40} & 72.08          & 51.68          & 65.48 \\
& $1$           & 41.04          & 77.18          & 58.78          & \bronze{87.18} & \silver{76.04} & 48.62          & 64.81 \\
& $2$           & \bronze{45.00} & \bronze{77.70} & \bronze{58.90} & 86.75          & 74.79          & \bronze{53.51} & \bronze{66.11} \\
& $5$           & \gold{47.71}   & \silver{78.57} & 57.39          & 86.25          & \gold{77.71}   & \silver{55.55} & \silver{67.20} \\
\midrule
\multirow{5}{*}{\dchem{}}
& $0$           & \silver{55.63} & \silver{79.25} & \silver{61.62} & \silver{87.00} & 74.79          & 51.94          & \bronze{68.37} \\
& $0.5$         & 53.96          & 79.10          & \bronze{61.05} & \gold{87.42}   & \silver{75.21} & \bronze{52.93} & 68.28 \\
& $1$           & \gold{55.84}   & \bronze{79.20} & \gold{62.31}   & \bronze{86.92} & \silver{75.21} & 51.68          & \silver{68.53} \\
& $2$           & \bronze{55.00} & \gold{80.97}   & 59.47          & 86.40          & \gold{76.04}   & \silver{54.92} & \gold{68.80} \\
& $5$           & 53.34          & 78.42          & 57.39          & 84.10          & 72.29          & \gold{56.11}   & 66.94 \\
\bottomrule
\end{tabular}
\end{table}
\begin{table}[t]
\centering\small \setlength{\tabcolsep}{7.5pt}
\caption{SDPO with and without CoT Performance (Acc@8).
CoT enhances \dtool{} training while impairing \dmath{} and \dchem{}.
Larger teacher EMA $\alpha$ generally worsens CoT performance.}
\label{appdx:tab:sdpothink}
\begin{tabular}{cccccccccc}
\toprule
$\alpha (\%)$            & CoT            & Trained on & AIME & Math500 & GPQA  & MMLU-R & BFCLv4 & LCBv6 \\
\midrule
--                       & --             & --         & 32.71 & 73.85   & 62.94 & 86.02  & 76.25  & 46.73 \\
\midrule
$0$                      & \texttimes & \dmath{}& 56.42 & 84.29   & 58.72 & 85.47  & 77.71  & 48.52 \\
\rowcolor{silver!50} $0$ & \checkmark & \dmath{}& 38.13 & 75.58   & 58.59 & 86.22  & 66.46  & 56.82 \\
$1$                      & \texttimes & \dmath{}& 56.04 & 84.89   & 58.62 & 85.31  & 77.71  & 51.79 \\
\rowcolor{silver!50} $1$ & \checkmark & \dmath{}& 37.71 & 75.58   & 56.25 & 86.20  & 66.88  & 54.45 \\
$5$                      & \texttimes & \dmath{}& 34.38 & 78.97   & 55.87 & 83.82  & 78.12  & 50.86 \\
\rowcolor{silver!50} $5$ & \checkmark & \dmath{}& 8.13  & 69.50   & 40.53 & 80.95  & 34.38  & 41.76 \\
$0$                      & \texttimes & \dchem{}& 55.63 & 79.25   & 61.62 & 87.00  & 74.79  & 51.94 \\
\rowcolor{silver!50} $0$ & \checkmark & \dchem{}& 42.08 & 76.25   & 55.05 & 87.08  & 73.12  & 57.08 \\
$1$                      & \texttimes & \dchem{}& 55.84 & 79.20   & 62.31 & 86.92  & 75.21  & 51.68 \\
\rowcolor{silver!50} $1$ & \checkmark & \dchem{}& 37.29 & 75.60   & 55.11 & 86.65  & 71.67  & 60.46 \\
$5$                      & \texttimes & \dchem{}& 53.34 & 78.42   & 57.39 & 84.10  & 72.29  & 56.11 \\
\rowcolor{silver!50} $5$ & \checkmark & \dchem{}& 15.63 & 71.50   & 39.96 & 81.85  & 60.00  & 47.77 \\
$0$                      & \texttimes & \dtool{}& 47.71 & 78.88   & 60.73 & 87.82  & 76.04  & 58.48 \\
\rowcolor{silver!50} $0$ & \checkmark & \dtool{}& 53.34 & 82.28   & 56.69 & 84.38  & 80.62  & 56.66 \\
$1$                      & \texttimes & \dtool{}& 41.04 & 77.18   & 58.78 & 87.18  & 76.04  & 49.69 \\
\rowcolor{silver!50} $1$ & \checkmark & \dtool{}& 53.55 & 82.90   & 56.38 & 84.55  & 80.42  & 59.11 \\
$5$                      & \texttimes & \dtool{}& 47.71 & 78.57   & 57.39 & 86.25  & 77.71  & 55.90 \\
\rowcolor{silver!50} $5$ & \checkmark & \dtool{}& 53.34 & 83.00   & 50.06 & 83.30  & 77.08  & 52.50 \\
\bottomrule
\end{tabular}
\end{table}
\Cref{appdx:tab:sdpoema} presents the performance scores used to plot \Cref{fig:sdpoema:result},
and displays a configuration matrix of EMA teachers $\alpha\in\{0, 0.5\%, 1\%, 2\%, 5\%\}$ and
training datasets \dmath{}, \dtool{}, and \dchem{}.
\Cref{appdx:tab:sdpothink} extends \Cref{fig:sdpothink}
for SDPO versus $\sdpocot$ under more teacher EMA settings and evaluation benchmarks.
The large degradation at $\alpha=1\%, 5\%$ further verifies \emph{confirmation bias}. 

\begin{table}[ht]
\centering\small \setlength{\tabcolsep}{4pt}
\caption{Performance (Acc@8) trained on a single domain. 
Rows indicate training domain and method; columns indicate evaluation benchmarks. GRPO outperforms SDPO at generalization.}
\label{appdx:tab:singledomain}
% \documentclass[]{article}
% \input{section/preamble}

% \begin{document}

% \begin{table}[ht]
% \small\setlength{\tabcolsep}{3pt}

\begin{NiceTabular}{llcccccccc}
\CodeBefore
  % \cellcolor{gray!15}{2-2} % Colors cell in row 2, column 2
\Body
\toprule
% &&\multicolumn{2}{c}{\dmath}&\multicolumn{2}{c}{\dchem{}}&\multicolumn{2}{c}{\dtool{}}&\dataset{Logic}&\dataset{Know.}\\
% \cmidrule(lr){3-4}\cmidrule(lr){5-6}\cmidrule(lr){7-8}\cmidrule(lr){9-9}\cmidrule(lr){10-10}
&Method&{AIME}&{Math500}&{\dchem{}}&{ GPQA  }&{ \dtool{} }&{ BFCLv4}&{ ZLogic }&{ MMLU-R} \\
\midrule
\multicolumn{10}{c}{Olmo-3-7B-Instruct-DPO}\\
\midrule
&--&19.17&61.30&32.50&30.24&42.46&76.46&60.48&69.92\\
\midrule\multirow{2}{*}{\rotatebox[origin=c]{90}{{\scriptsize\dmath}}}
&GRPO&18.96&64.82&35.24&35.23&40.62&71.88&64.95&72.00\\
&\sdpoz&13.75&51.24&21.93&24.94&34.38&57.08&34.31&58.96\\
\midrule
\multicolumn{10}{c}{Olmo-3-7B-Think-DPO}\\
\midrule
&--&46.04&71.53&30.89&36.30&9.19&3.96&59.18&79.35\\
\midrule\multirow{4}{*}{\rotatebox[origin=c]{90}{{\dmath}}}
&GRPO&51.67&74.95&37.56&44.32&9.38&3.33&61.58&79.47\\
&\sdpof   &41.67&72.90&28.15&33.71&8.09&0.00&43.88&72.20\\
&\sdpoz   &51.25&75.95&32.98&38.45&8.46&0.62&54.40&79.00\\
&\sdpozcot&54.59&74.52&41.79&43.81&7.17&0.00&61.78&78.53\\
\bottomrule
\end{NiceTabular}

% \caption{Caption}
% \label{tab:placeholder}
% \end{table}

% \end{document}
\end{table}

\begin{table}[ht]
\centering\small
\caption{Representative gain, retention, and forgetting cases. 
Intermediate-distance targets exhibit the clearest forgetting, while source-like and very distant targets are often stable or improve.}
\label{tab:intermediate_cases}
\begin{tabular}{lllcccccc}
\toprule
& Train & Test
& MMD$\cdot10^{-2}$ & $\Delta_{\sdpof{}}$ & $\Delta_{\sdpoz{}}$ & Gain & Retention & Forgetting  \\
\midrule
\multirow{4}{*}{\rotatebox[origin=c]{90}{{Near}}}
& \dmath{}& AIME    & $0.64$ & \dfcell{1.67 }& \dfcell{23.71} & \checkmark  \\
& \dmath{}& Math500 & $1.38$ & \dfcell{5.12 }& \dfcell{10.44} & \checkmark  \\
& \dcode{}& LCBv6   & $3.74$ & \dfcell{11.74}& \dfcell{13.39} & \checkmark  \\
& \dtool{}& BFCLv4  & $4.75$ & \dfcell{1.46 }& \dfcell{-0.21 }&&\checkmark  \\
\midrule
\multirow{6}{*}{\rotatebox[origin=c]{90}{{Intermediate}}}
& \dmath{}& GPQA    & $6.74 $ & \dfcell{-7.07} & \dfcell{-4.22} &&&\checkmark  \\
& \dchem{}& GPQA    & $11.42$ & \dfcell{-5.55} & \dfcell{-1.32} &&&\checkmark  \\
& \dtool{}& GPQA    & $18.89$ & \dfcell{-5.55} & \dfcell{-2.21} &&&\checkmark  \\
& \dcode{}& GPQA    & $11.90$ & \dfcell{-4.61} & \dfcell{-2.65} &&&\checkmark  \\
& \dmath{}& ZLogic  & $21.98$ & \dfcell{-7.94} & \dfcell{-7.12} &&&\checkmark  \\
& \dchem{}& BFCLv4  & $25.16$ & \dfcell{-3.96} & \dfcell{-1.46} &&&\checkmark  \\
\midrule
\multirow{3}{*}{\rotatebox[origin=c]{90}{{Distant}}}
& \dchem{}& ZLogic & $32.66$ & \dfcell{9.06}  & \dfcell{8.18} & \checkmark  \\
& \dtool{}& ZLogic & $43.33$ & \dfcell{2.40}  & \dfcell{4.86} & \checkmark  \\
& \dcode{}& ZLogic & $31.66$ & \dfcell{-1.09} & \dfcell{0.96} &&\checkmark  \\
\bottomrule
\end{tabular}
\end{table}
\begin{table}[H]
\centering\small \setlength{\tabcolsep}{4pt}
\caption{\sdpof{} performance (Acc@8) continually trained with various data at 2nd \& 3rd stage. Altering order only partially mitigates collapse; SDPO still degrades significantly under data shifts.}
\label{appdx:tab:cl_collapse}
\begin{tabular}{clcccccccccccc}
\toprule
&&\multicolumn{2}{c}{\dmath{}}&\multicolumn{2}{c}{\dchem{}}&\multicolumn{2}{c}{\dtool{}}&\multicolumn{2}{c}{\dcode{}}&\dataset{Logic}&\dataset{Know.}\\
\cmidrule(lr){3-4}\cmidrule(lr){5-6}\cmidrule(lr){7-8}\cmidrule(lr){9-10}\cmidrule(lr){11-11}\cmidrule(lr){12-12}
&Trained on&{\scriptsize AIME}&{\scriptsize Math500}&{\scriptsize \dchem{}}&{\scriptsize  GPQA  }&{\scriptsize  \dtool{} }&{\scriptsize  BFCLv4}&{\scriptsize  \dcode{} }&{\scriptsize  LCBv6 }&{\scriptsize  ZLogic }&{\scriptsize  MMLU-R} \\
\midrule
&\multicolumn{1}{c}{--}&32.71&73.85&50.54&62.94&48.90&76.25&56.39&46.73&81.62&86.02 \\
\midrule
\multirow{8}{*}{\rotatebox[origin=c]{90}{\sdpof}}
&1. \dmath{}&34.38&78.97&45.00&55.87&53.31&78.12&52.37&50.86&73.68&83.82 \\
\cmidrule(lr){2-12}
&2. \dchem{}&21.25&73.95&76.37&47.98&30.15&51.25&59.33&48.82&61.88&82.12 \\
&3. \dtool{}& 0.00& 0.00& 0.18& 0.13& 0.00& 0.00& 0.00& 0.00& 0.10& 0.12 \\
\cmidrule(lr){2-12}
&2. \dchem{}&21.25&73.95&76.37&47.98&30.15&51.25&59.33&48.82&61.88&82.12 \\
&3. \dcode{}&17.09&69.53&74.52&42.93& 0.74& 0.00&54.79&54.01&58.00&80.50 \\
\cmidrule(lr){2-12}
&2. \dtool{}&22.29&73.82&38.63&42.36&12.68& 9.79&   --&   --&56.95&77.35 \\
&3. \dchem{}&10.42&66.50&72.56&36.55&13.60&20.62&   --&   --&49.20&77.75 \\
\bottomrule
\end{tabular}
\end{table}

For single-domain training, \Cref{appdx:tab:singledomain} examines SDPO and GRPO for Olmo-3 series.
For continual training, \Cref{appdx:tab:cl} presents detailed scores for \Cref{fig:cl} across each stage, with a data-ordering ablation on collapsed \sdpof{} in \Cref{appdx:tab:cl_collapse}.
We supplement diagnostics (\Cref{fig:analysis}) with additional matrices (\Cref{appdx:fig:svd:cl}), more training (\Cref{appdx:fig:svd:more}), and more benchmarks (\Cref{appdx:fig:rollout}).

\begin{wrapfigure}{r}{.48\textwidth}
\centering
\includegraphics[width=.47\textwidth]{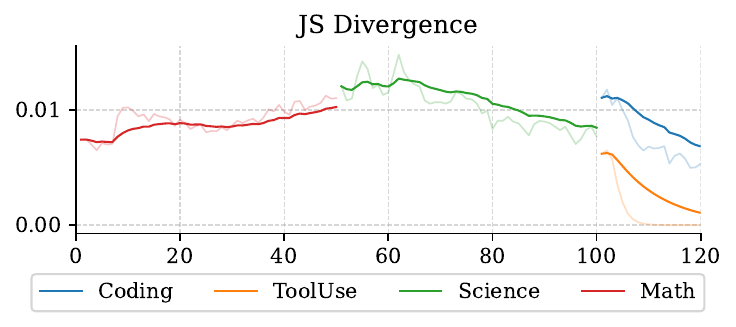}
\caption{JS divergence during continual-\sdpof{} with \dtool{} or \dcode{} as third stage.}
\label{fig:cl:training}
\end{wrapfigure}
We also identify one extreme failure mode as artifact amplification in self-distillation.
\sdpof{} collapses with a visible generation artifact: endlessly repeated ``{\textbackslash{}boxed}'' tokens.
This behavior coincides with a change in final-answer formatting requirement at the third stage, where \dtool{} differs from earlier \dmath{} and \dchem{} stages (data samples in \Cref{appdx:sec:data}). The collapsed model still over-amplifies the high-frequency formatting pattern from earlier training stages. More experiments show that this collapse can be partially mitigated by reordering data or masking tokens (\Cref{appdx:tab:cl,appdx:tab:cl_collapse}).

This failure can be attributed to the token supervision: frequent tokens can receive repeated reinforcement. 
As the teacher is derived from the student, artifacts can enter a self-reinforcing loop (\emph{confirmation bias}): the student emits the artifact token, the teacher assigns it token-level support, and the next student update makes it more likely, explaining the JS divergence plummeting in the third stage (\Cref{fig:cl:training}).
A frequently refreshed teacher or denser supervision can worsen this process and, in severe cases, can lead to near-complete collapse.

\begin{table}[t]\centering\small\setlength{\tabcolsep}{3pt}
\caption{Performance (Acc@8) comparison of linear and spectral weight editing for mitigating forgetting in \sdpoz. $@s$ denotes the scale of update between the post-trained model and pre-update baseline. $r$ denotes the top-$r$ singular components to retain, which takes $r=64$ unless otherwise specified.}
\label{appdx:tab:merging}
\begin{tabular}{lccccccccccc}
\toprule
&\multicolumn{2}{c}{\dmath{}}&\multicolumn{2}{c}{\dchem{}}&\multicolumn{2}{c}{\dtool{}}&\multicolumn{2}{c}{\dcode{}}&\dataset{Logic}&\dataset{Know.}&\multirow{2}{*}{Avg}\\
\cmidrule(lr){2-3}\cmidrule(lr){4-5}\cmidrule(lr){6-7}\cmidrule(lr){8-9}\cmidrule(lr){10-10}\cmidrule(lr){11-11}
Method&{\scriptsize AIME}&{\scriptsize Math500}&{\scriptsize \dchem{}}&{\scriptsize GPQA  }&{\scriptsize\dtool{} }&{\scriptsize BFCLv4}&{\scriptsize \dcode{} }&{\scriptsize LCBv6 }&{\scriptsize ZLogic }&{\scriptsize MMLU-R} \\
\midrule
Baseline&32.71&73.85&50.54&62.94&48.90&76.25&56.39&46.73&81.62&86.02&61.60\\
\dmath{}-\grpo&44.67&76.84&51.33&62.83&50.48&78.33&60.73&52.93&80.82&85.74&64.47 \\
\midrule
\dmath{}-\sdpoz&56.42&84.29&47.36&58.72&52.06&77.71&57.58&48.52&74.50&85.47&64.26\\
$+\lin{}{@0.3}$&48.34&81.42&50.71&63.38&49.45&75.00&56.71&44.01&79.20&85.65&63.39\\
$+\lin{}{@0.5}_{\lambda=3e-5}$&44.80&80.82&49.64&61.93&49.26&77.71&56.55&44.94&79.97&85.72&63.13\\
$+\lin{}{@0.5}$&51.67&84.25&48.81&62.25&50.92&77.29&56.47&46.65&76.42&85.65&64.04\\
$+\lin{}{@0.7}$&55.00&84.82&46.07&59.41&52.57&76.67&57.59&48.36&72.10&84.12&63.67\\
$+\lin{}{@1}_{\lambda=3e-5}$&55.63&84.90&47.86&60.10&50.55&77.08&56.81&45.98&75.32&85.60&63.98\\
$+\lin{}{@1}_{\lambda=5e-5}$&56.67&84.58&44.23&57.20&51.84&77.92&56.75&47.88&72.62&84.35&63.40\\
$+\spe{}{@0.3}$&41.46&79.10&50.77&62.18&48.71&74.58&55.79&45.25&80.88&85.95&62.47\\
$+\spe{}{@0.5}_{r=16}$&50.21&82.05&52.26&61.62&48.90&75.21&57.95&46.61&79.38&86.40&64.06\\
$+\spe{}{@0.5}_{r=256}$&50.42&82.65&49.64&61.87&49.45&77.50&56.24&44.25&78.70&85.98&63.67\\
$+\spe{}{@0.5}$&49.59&81.88&48.81&60.54&49.63&76.25&57.44&45.35&78.72&85.65&63.39\\
$+\spe{}{@0.7}$&54.17&83.15&49.70&60.10&50.74&76.67&57.45&47.72&76.60&86.18&64.25\\
$+\spe{}{@1}_{\lambda=3e-5}$&49.38&82.18&49.58&62.88&49.45&75.42&56.46&47.01&77.88&86.38&63.66\\
\bottomrule
\end{tabular}
\end{table}

% Method paragraph for src/postprocess/merge/core.py (drift-anchored task-vector merging).
In weight editing, we merge a post-trained checkpoint back into its base model through a task-vector formulation designed to mitigate forgetting. For each weight tensor, let $W_0$ and $W_1$ denote the base and post-trained weights and $\Delta = W_1 - W_0$ the task vector; the merged weight takes the form $W = W_0 + \lambda \Delta$, where the effective per-tensor scale $\lambda$ composes three mechanisms. First, a global scale $s \in [0,1]$ interpolates between the base ($s{=}0$) and post-trained ($s{=}1$) endpoints in the spirit of WiSE-FT, trading in-domain gains against out-of-distribution (OOD) retention. Second, in the \emph{spectral} variant, the task vector is truncated to its top-$r$ singular components, $\Delta_r = U_{:r} S_{:r} V_{:r}^{\top}$ with $\Delta = U S V^{\top}$, before scaling: the discarded spectral tail is hypothesized to carry noisy, overfit directions that drive OOD forgetting, whereas the principal subspace preserves most in-domain gains. Third, each task vector is optionally anchored to a spectral-drift budget: with
\begin{equation*}
  d_t = \frac{\bigl\|\boldsymbol{\sigma}(W_1) - \boldsymbol{\sigma}(W_0)\bigr\|_2}{\bigl\|\boldsymbol{\sigma}(W_0)\bigr\|_2}
\end{equation*}
the normalized drift of the singular-value spectrum, a runtime-measurable proxy for posttraining-induced forgetting, the per-tensor scale $\lambda_t = \min(1, \tau / d_t)$ shrinks over-drifted tensors to the tolerated level $\tau$ while leaving within-budget tensors untouched; in the first-order regime of observed drifts ($d_t \sim 10^{-5}$), $d(W_0,\, W_0 + \lambda\Delta) \approx \lambda\, d_t$, so the anchor is exact up to $O(d_t^2)$. An additional relative-norm guard $\|\lambda \Delta\|_F \le g \,\|W_0\|_F$ caps the per-tensor update magnitude. All computations are carried out in float32 per tensor with tied weights merged once; one-dimensional tensors (biases, normalization parameters), for which the drift metric is undefined, are exempt from anchoring, and the choice $r \ge \min(m, n)$ recovers plain linear interpolation.

\begin{figure}[t]
\centering
\begin{subfigure}[b]{0.33\textwidth}
\includegraphics[width=\textwidth]{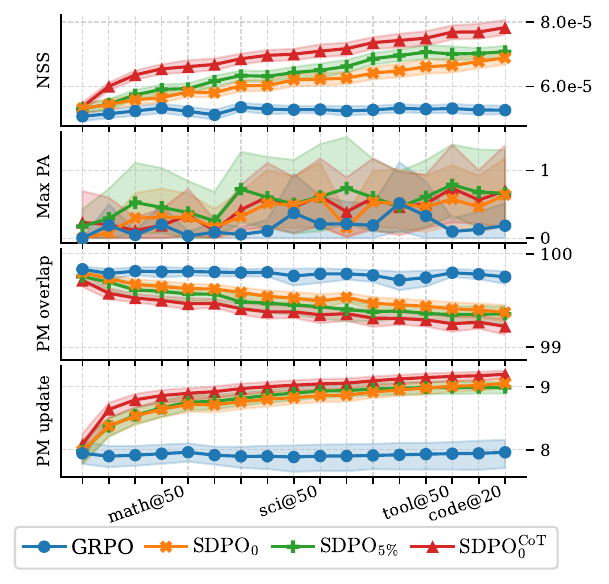}
\caption{MLP Down Projections}
\end{subfigure}\hfill%
\begin{subfigure}[b]{0.33\textwidth}
\includegraphics[width=\textwidth]{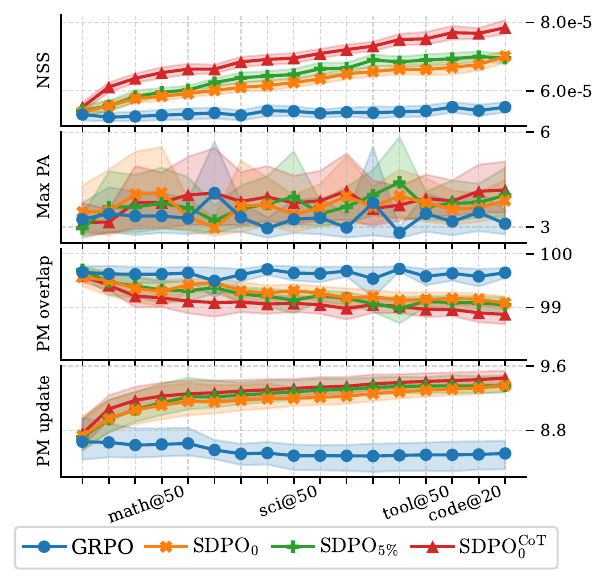}
\caption{MLP Up Projections}
\end{subfigure}\hfill%
\begin{subfigure}[b]{0.33\textwidth}
\includegraphics[width=\textwidth]{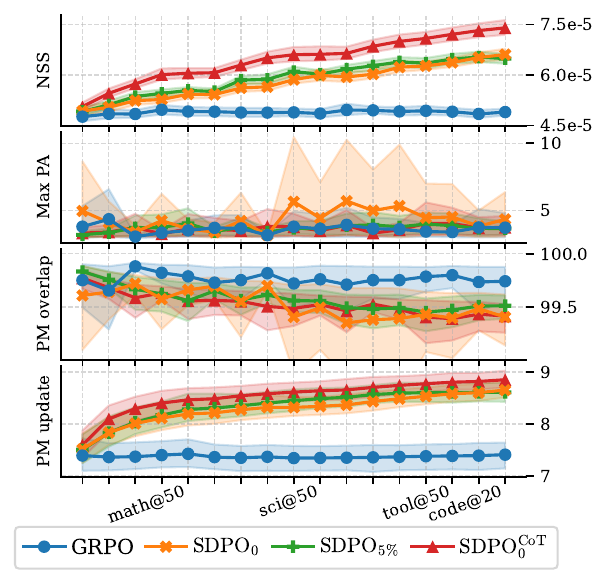}
\caption{MLP Gate Projections}
\end{subfigure}\\
\begin{subfigure}[b]{0.33\textwidth}
\includegraphics[width=\textwidth]{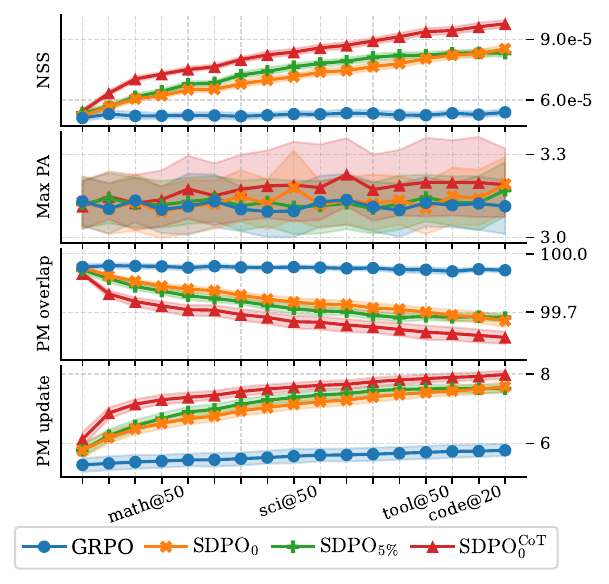}
\caption{Self-Attention $Q$s}
\end{subfigure}\hfill%
\begin{subfigure}[b]{0.33\textwidth}
\includegraphics[width=\textwidth]{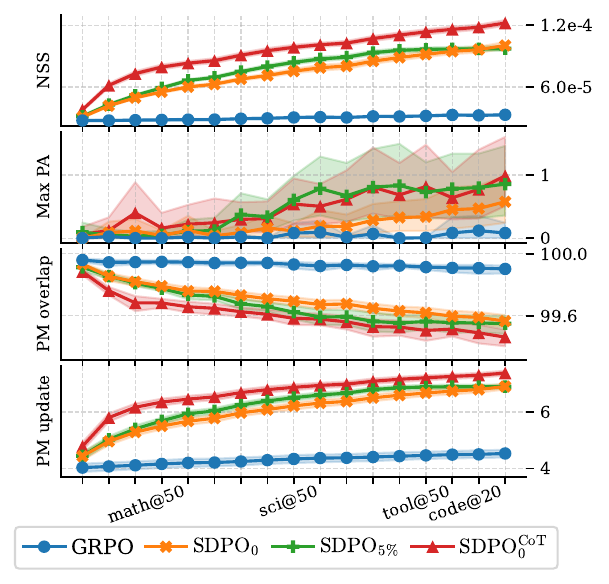}
\caption{Self-Attention $K$s}
\end{subfigure}\hfill%
\begin{subfigure}[b]{0.33\textwidth}
\includegraphics[width=\textwidth]{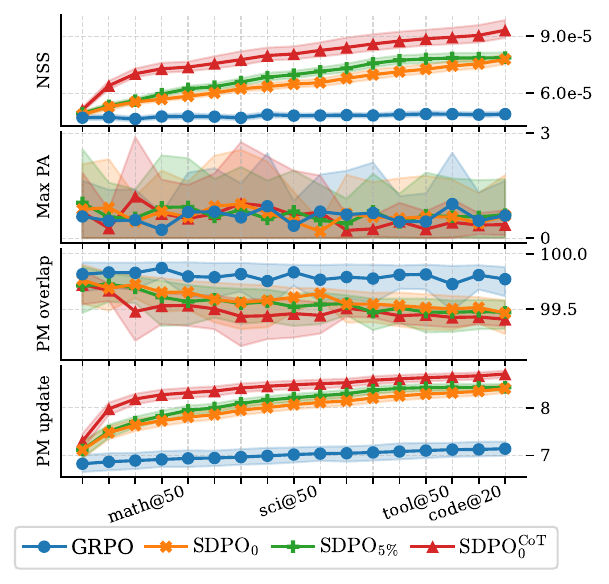}
\caption{Self-Attention $O$s}
\end{subfigure}\\
\caption{Parameter drift ($y$-axis) against training steps ($x$-axis), with metrics averaged across all layers with a 95\% confidence interval. \textbf{(a-f)} \dmath{} training. \textbf{(g-l)} continual training.}
\label{appdx:fig:svd:cl}
\end{figure}

\begin{table}[p]
\centering \small 
\setlength{\tabcolsep}{3.8pt} \renewcommand{\arraystretch}{0.8}
\caption{SDPO versus GRPO performance (Acc@8) continually trained.
\sdpoz{} and \sdpof{} denote SDPO with teacher EMA $0$ and $5\%$, respectively.
\sdpozcot{} denotes \sdpoz{} with CoT distilled.
\sdpomask{25} denotes \sdpof{} randomly masked out 25\% tokens.
S-\sdpoz{} denotes SDPO with restart-and-freeze teacher strategy ($\gamma=1$).}
\label{appdx:tab:cl}
\begin{tabular}{clccccccccccc}
\toprule
&&\multicolumn{2}{c}{\dmath{}}&\multicolumn{2}{c}{\dchem{}}&\multicolumn{2}{c}{\dtool{}}&\multicolumn{2}{c}{\dcode{}}&\dataset{Logic}&\dataset{Know.}&\multirow{2}{*}{Avg}\\
\cmidrule(lr){3-4}\cmidrule(lr){5-6}\cmidrule(lr){7-8}\cmidrule(lr){9-10}\cmidrule(lr){11-11}\cmidrule(lr){12-12}
&Trained on&{\scriptsize AIME}&{\scriptsize Math500}&{\scriptsize \dchem{}}&{\scriptsize  GPQA  }&{\scriptsize  \dtool{} }&{\scriptsize  BFCLv4}&{\scriptsize  \dcode{} }&{\scriptsize  LCBv6 }&{\scriptsize  ZLogic }&{\scriptsize  MMLU-R} \\
\midrule

\multicolumn{13}{c}{Qwen3-4B-Thinking-2507}\\
\midrule
&\multicolumn{1}{c}{--}&32.71&73.85&50.54&62.94&48.90&76.25&56.39&46.73&81.62&86.02&61.60 \\
\midrule
\multirow{4}{*}{\rotatebox[origin=c]{90}{{\grpo}}}
&1. \dmath{}&44.67&76.84&51.33&62.83&50.48&78.33&60.73&52.93&80.82&85.74&64.47 \\
&2. \dchem{}&54.88&78.28&56.39&64.23&52.06&76.67&60.06&50.17&88.23&87.91&66.89 \\
&3. \dtool{}&56.21&79.49&60.64&63.99&58.71&80.83&60.49&52.23&90.10&88.54&69.12 \\
&4. \dcode{}&57.29&79.88&61.85&64.58&57.72&81.25&64.68&53.56&89.52&88.70&69.90 \\
\midrule
\multirow{4}{*}{\rotatebox[origin=c]{90}{\sdpoz}}
&1. \dmath{}&56.42&84.29&47.36&58.72&52.06&77.71&57.58&48.52&74.50&85.47&64.26 \\
&2. \dchem{}&49.29&81.75&65.70&59.03&48.49&76.67&58.35&56.88&87.17&86.38&66.97 \\
&3. \dtool{}&41.13&80.02&70.37&54.71&37.35&30.83&56.39&54.25&80.03&84.92&59.00 \\
&4. \dcode{}&34.38&77.40&70.71&51.83& 9.93&11.25&61.77&51.89&72.15&84.20&52.55 \\
\midrule
\multirow{4}{*}{\rotatebox[origin=c]{90}{\sdpozcot}}
&1. \dmath{}&38.13&75.58&49.88&58.59&48.35&66.46&52.70&56.82&77.00&86.22&60.97 \\
&2. \dchem{}&27.71&74.65&70.89&54.10&42.83&61.67&47.28&46.28&76.65&85.95&58.80 \\
&3. \dtool{}&21.04&71.88&70.71&49.37&38.60&36.04&41.94&44.88&66.22&84.38&52.51 \\
&4. \dcode{}&14.17&69.53&68.69&43.24&23.35&13.54&30.01&24.32&59.95&81.62&42.84 \\
% \midrule
% \multirow{4}{*}{\rotatebox[origin=c]{90}{S-\sdpoz{}}}
% &1. \dmath{}&57.50&85.05&49.11&60.86&51.47&76.67&--&--&74.68&86.60&--\\
% &2. \dchem{}&52.92&79.77&61.96&59.91&35.48&50.83&--&--&86.25&86.75&--\\
% &3. \dtool{}&27.50&74.88&59.76&49.62&6.80&13.12&47.17&45.75&64.40&81.05&47.00\\
% &4. \dcode{}&\\
\midrule
\multirow{4}{*}{\rotatebox[origin=c]{90}{\sdpof}}
&1. \dmath{}&34.38&78.97&45.00&55.87&53.31&78.12&52.37&50.86&73.68&83.82&60.64 \\
&2. \dchem{}&21.25&73.95&76.37&47.98&30.15&51.25&59.33&48.82&61.88&82.12&55.31 \\
&3. \dtool{}& 0.00& 0.00& 0.18& 0.13& 0.00& 0.00& 0.00& 0.00& 0.10& 0.12& 0.05 \\
&4. \dcode{}& 0.00& 0.00& 0.00& 0.00& 0.00& 0.00& 0.00& 0.00& 0.02& 0.18& 0.02 \\
\midrule
\multirow{3}{*}{\rotatebox[origin=c]{90}{\dataset{\sdpomask{25}}}}
&1. \dmath{}&37.09&79.88&42.32&56.12&52.39&74.58&54.71&51.51&71.80&84.35&60.48 \\
&2. \dchem{}&18.54&72.38&72.32&45.08&23.16&23.96&57.73&48.75&54.12&81.90&49.79 \\
&3. \dtool{}& 0.00& 0.00& 5.42& 1.96& 0.00& 0.00& 0.00& 0.00& 1.50& 9.65& 1.85 \\
\midrule
\multirow{3}{*}{\rotatebox[origin=c]{90}{\dataset{\sdpomask{50}}}}
&1. \dmath{}&39.17&78.53&43.99&55.81&50.00&76.67&54.01&51.75&75.20&84.72&60.98 \\
&2. \dchem{}&25.84&76.10&77.98&51.52&35.48&44.58&48.18&49.56&72.05&82.98&56.43 \\
&3. \dtool{}& 0.00& 0.12&34.17&18.62& 0.00& 0.00& 0.09& 0.00&55.47&43.38&15.18 \\
\midrule
\multirow{3}{*}{\rotatebox[origin=c]{90}{\dataset{\sdpomask{75}}}}
&1. \dmath{}&37.50&79.05&43.15&56.12&53.68&80.00&53.53&51.93&75.22&84.85&61.50 \\
&2. \dchem{}&18.55&73.05&70.95&44.00&42.83&63.75&59.02&50.39&56.53&81.73&55.31 \\
&3. \dtool{}& 0.63&28.48&54.70&28.22& 0.00& 0.00& 8.18& 7.31&45.55&62.02&23.51 \\
\midrule
\multicolumn{13}{c}{Qwen3-4B-Instruct-2507}\\
\midrule
&\multicolumn{1}{c}{--}&50.00&77.10&24.23&14.71&59.01&83.12&--&--&48.00&56.75&51.61\\
\midrule
\multirow{3}{*}{\rotatebox[origin=c]{90}{\grpo}}
&1. \dmath{}&56.00&79.94&25.94&17.23&58.97&83.33&--&--&49.40&57.60&53.55 \\
&2. \dchem{}&56.88&81.12&59.35&58.84&57.72&82.92&--&--&84.22&86.52&70.95 \\
&3. \dtool{}&58.34&81.88&60.54&59.91&65.99&75.42&--&--&84.80&87.85&71.84 \\
\midrule
\multirow{3}{*}{\rotatebox[origin=c]{90}{{\sdpoz}}}
&1. \dmath{}&35.84&71.15&12.20& 5.49&59.74&82.29&--&--& 9.15& 2.97&34.85 \\
&2. \dchem{}&27.08&64.53&11.79& 8.90&57.90&82.92&--&--&11.70& 9.48&34.29 \\
&3. \dtool{}&20.63&54.25& 7.14& 4.17& 0.00& 0.00&--&--& 1.48& 5.35&11.63 \\
\midrule
\multirow{3}{*}{\rotatebox[origin=c]{90}{\dataset{SFT}}}
&1. \dmath{}& 0.42&20.78&27.38&23.99&52.94&78.12&--&--& 2.58&66.55&34.09\\
&2. \dchem{}& 0.00& 0.00&72.74&30.74& 0.00& 7.92&--&--&30.22&69.30&26.36\\
&3. \dtool{}& 0.00& 0.00& 0.00& 0.00& 0.00& 0.00&--&--& 0.00& 0.00& 0.00\\
\midrule
\multicolumn{13}{c}{Olmo-3-7B-Think}\\
\midrule
&\multicolumn{1}{c}{--}&42.29&72.55&31.90&34.09&18.01&27.71&50.79&39.64&63.10&80.20&46.03\\
\midrule
\multirow{4}{*}{\rotatebox[origin=c]{90}{\grpo}}
&1. \dmath{}&52.08&74.50&36.01&39.96&14.71&20.62&55.39&43.25&63.18&80.53&48.02 \\
&2. \dchem{}&52.30&75.38&49.88&42.42&13.05&23.12&56.92&45.66&62.22&81.25&50.22\\
&3. \dtool{}&51.88&74.85&53.04&43.75&52.57&74.79&59.45&46.10&63.82&80.70&60.09\\
&4. \dcode{}&54.17&75.60&53.51&44.26&50.18&71.67&60.88&45.16&63.92&79.50&59.88\\
\midrule
\multirow{4}{*}{\rotatebox[origin=c]{90}{\sdpoz{}}}
&1. \dmath{}&52.08&75.12&34.76&37.50&10.29&25.42&52.51&41.90&56.92&78.25&46.48\\
&2. \dchem{}&43.96&74.50&57.38&35.54&9.38&11.67&56.96&45.57&57.18&76.40&46.85 \\
&3. \dtool{}&41.25&72.38&52.56&34.47&16.73&30.83&59.53&46.08&52.20&74.68&48.07 \\
&4. \dcode{}&37.29&71.60&53.39&32.83&14.89&20.21&60.95&44.98&46.62&73.40&45.62 \\
% \midrule
% \multicolumn{13}{c}{Olmo-3-7B-Instruct}\\
% \midrule
% &\multicolumn{1}{c}{--}&47.71&72.15&40.71&44.38&44.85&76.04&&&73.35&75.50\\
% \midrule
% \multirow{4}{*}{\rotatebox[origin=c]{90}{\grpo{}}}
% &1. \dmath{}&46.46&73.82&39.40&45.64&45.22&&&&73.88&75.32\\
% &2. \dchem{}&46.88&74.15&51.67&45.01&44.30&77.71&7.63&48.22&73.25&76.62\\
% &3. \dtool{}&46.67&74.48&54.52&43.69&56.43&73.33&6.93&49.89&73.60&76.68\\
% &4. \dcode{}&46.88&74.50&56.96&44.38&56.25&77.08&24.55&47.88&74.72&76.38\\
% \midrule
% \multirow{4}{*}{\rotatebox[origin=c]{90}{\sdpoz{}}}
% &1. \dmath{}& \\
% &2. \dchem{}&18.96&62.08&73.27&38.01&27.94&49.58&&&58.48&68.40\\
% &3. \dtool{}&12.08&55.95&72.80&37.37&0.00&0.00&&&43.62&63.98\\
% &4. \dcode{}&&&&&&0.00&&&&\\
\bottomrule
\end{tabular}
\end{table}

% we dont really care about these 'additional figures'
\clearpage

\begin{figure}[p]
\centering
\begin{subfigure}[b]{0.25\textwidth}
\includegraphics[width=\textwidth]{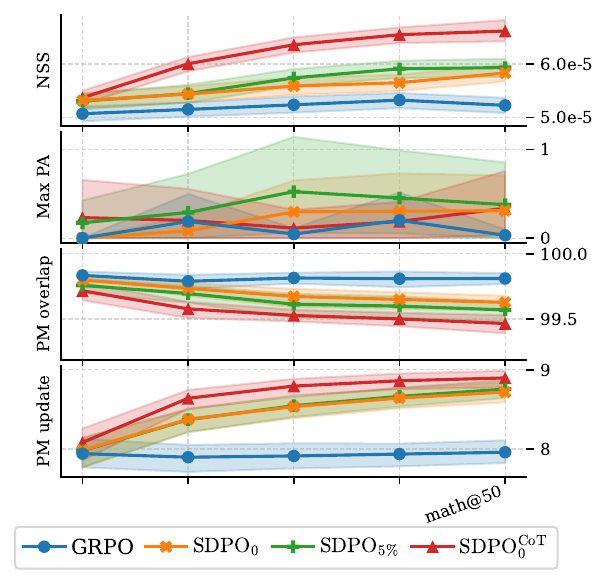}
\caption{MLP Down Projections}
\end{subfigure}\hfill%
\begin{subfigure}[b]{0.25\textwidth}
\includegraphics[width=\textwidth]{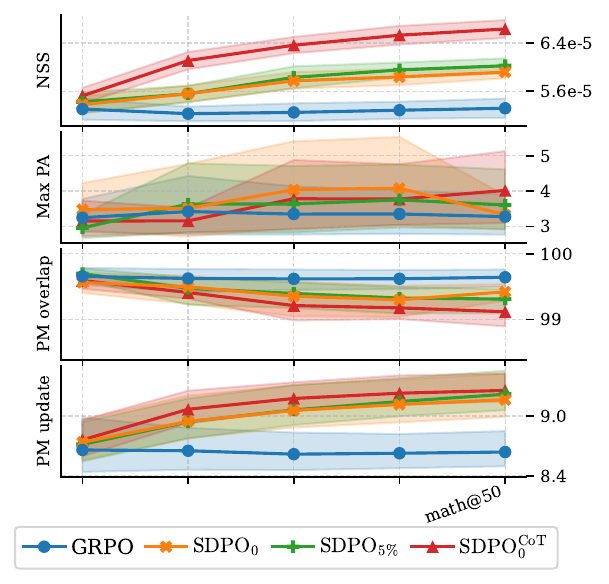}
\caption{MLP Up Projections}
\end{subfigure}\hfill%
\begin{subfigure}[b]{0.25\textwidth}
\includegraphics[width=\textwidth]{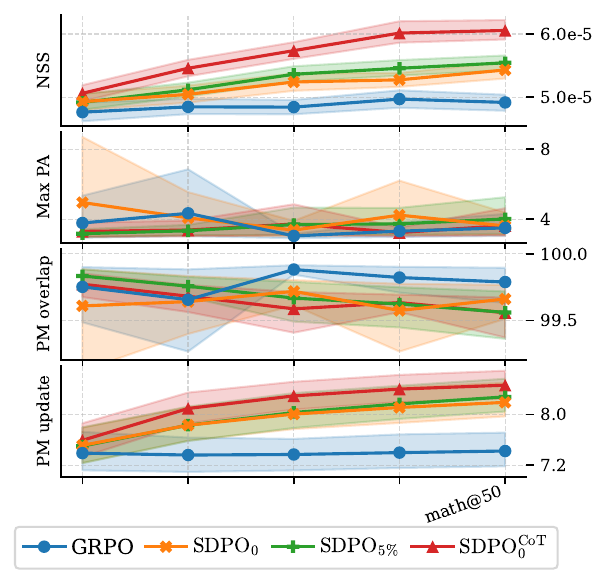}
\caption{MLP Gate Projections}
\end{subfigure}\hfill%
\begin{subfigure}[b]{0.25\textwidth}
\includegraphics[width=\textwidth]{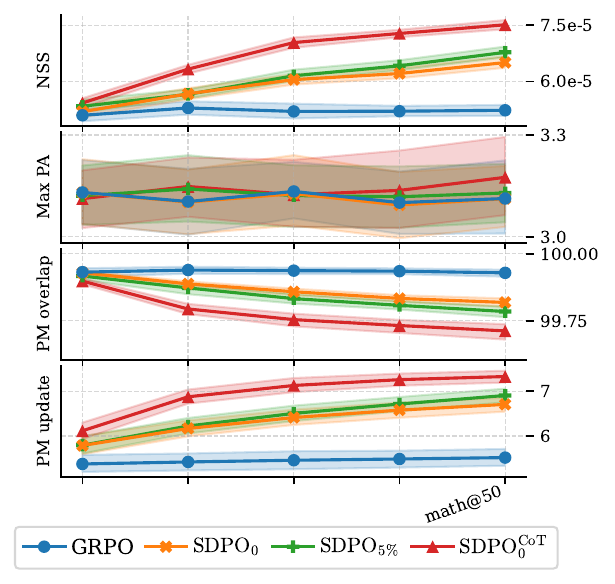}
\caption{Self-Attention $Q$s}
\end{subfigure}\hfill%
\begin{subfigure}[b]{0.25\textwidth}
\includegraphics[width=\textwidth]{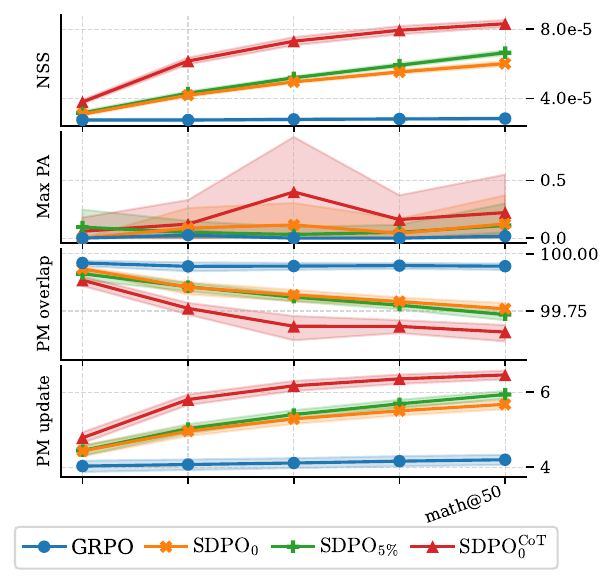}
\caption{Self-Attention $K$s}
\end{subfigure}\hfill%
\begin{subfigure}[b]{0.25\textwidth}
\includegraphics[width=\textwidth]{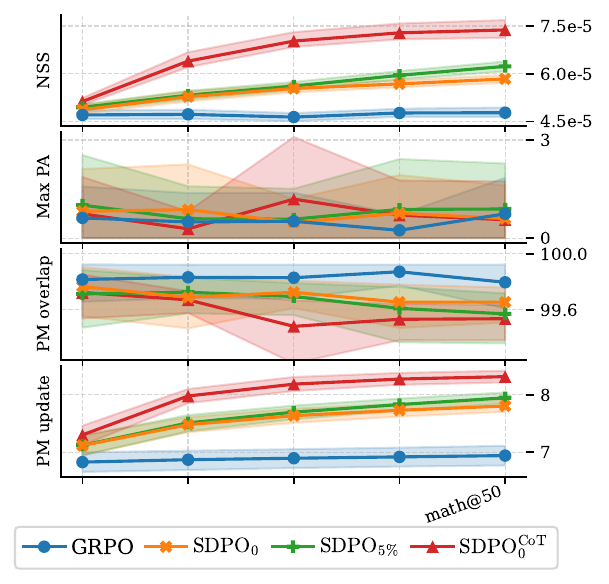}
\caption{Self-Attention $O$s}
\end{subfigure}\hfill%
\begin{subfigure}[b]{0.25\textwidth}
\includegraphics[width=\textwidth]{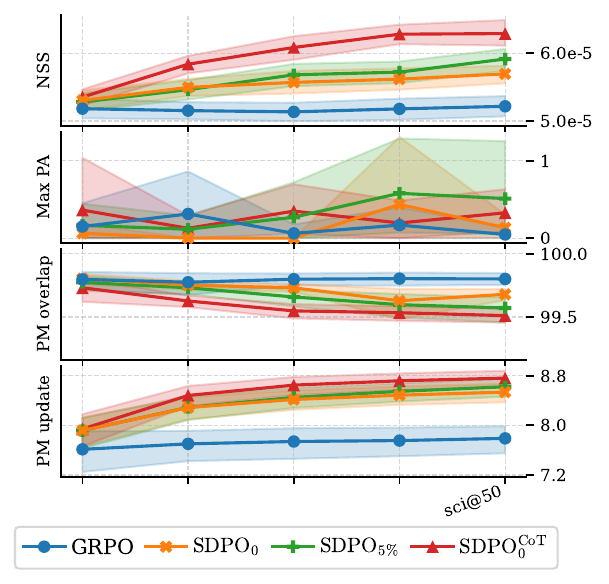}
\caption{\dchem{} MLP Down}
\end{subfigure}\hfill%
\begin{subfigure}[b]{0.25\textwidth}
\includegraphics[width=\textwidth]{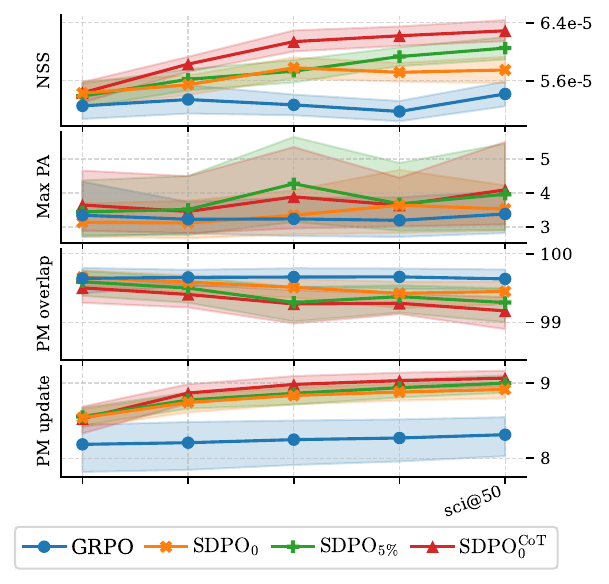}
\caption{\dchem{} MLP Up}
\end{subfigure}\hfill%
\begin{subfigure}[b]{0.25\textwidth}
\includegraphics[width=\textwidth]{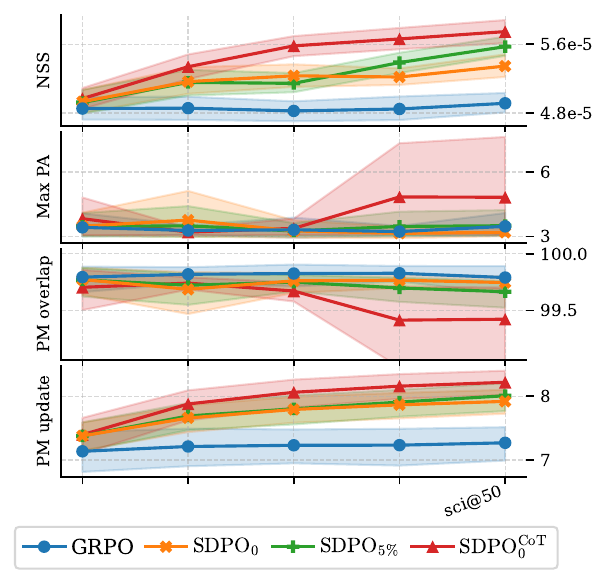}
\caption{\dchem{} MLP Gate}
\end{subfigure}\hfill%
\begin{subfigure}[b]{0.25\textwidth}
\includegraphics[width=\textwidth]{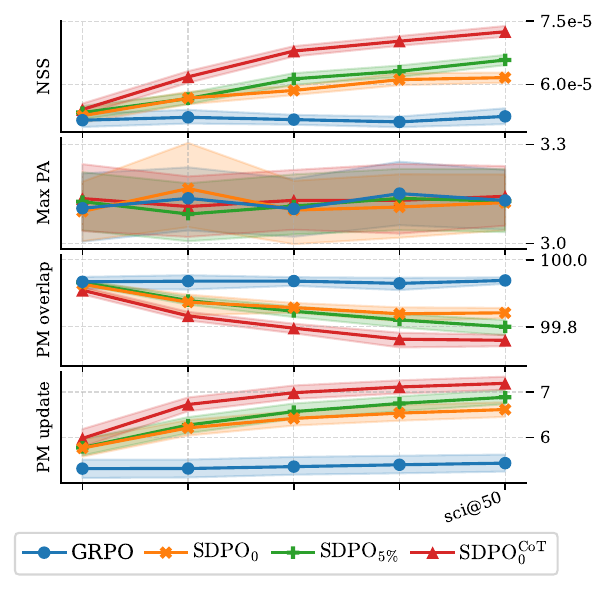}
\caption{\dchem{} Self-Attn $Q$s}
\end{subfigure}\hfill%
\begin{subfigure}[b]{0.25\textwidth}
\includegraphics[width=\textwidth]{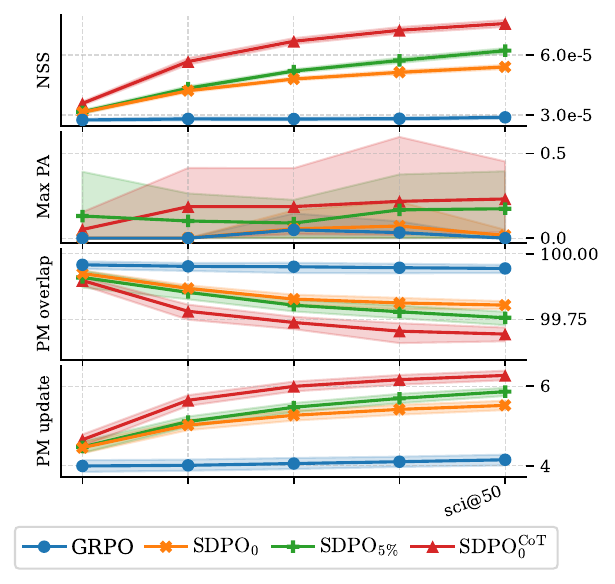}
\caption{\dchem{} Self-Attn $K$s}
\end{subfigure}\hfill%
\begin{subfigure}[b]{0.25\textwidth}
\includegraphics[width=\textwidth]{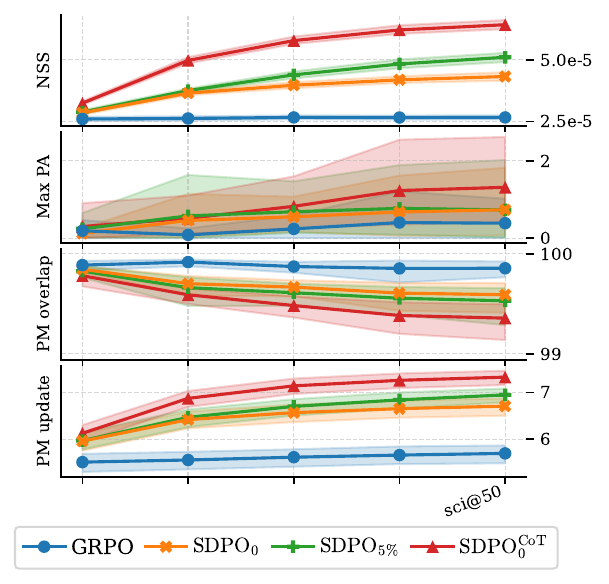}
\caption{\dchem{} Self-Attn $V$s}
\end{subfigure}\hfill%
\begin{subfigure}[b]{0.25\textwidth}
\includegraphics[width=\textwidth]{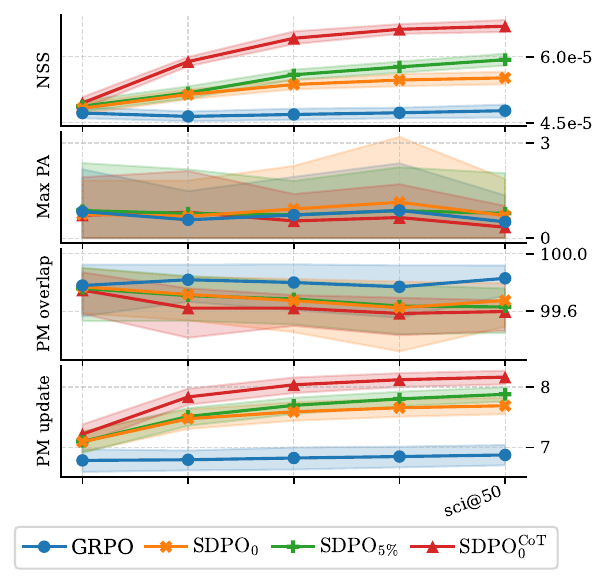}
\caption{\dchem{} Self-Attn $O$s}
\end{subfigure}\hfill%
\begin{subfigure}[b]{0.25\textwidth}
\includegraphics[width=\textwidth]{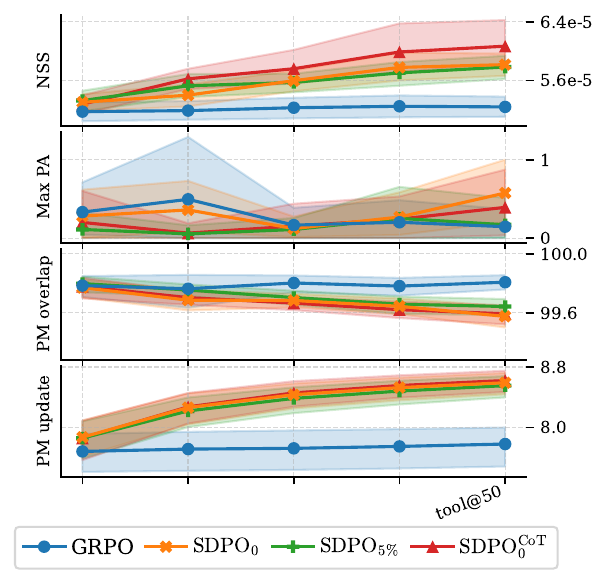}
\caption{\dtool{} MLP Down}
\end{subfigure}\hfill%
\begin{subfigure}[b]{0.25\textwidth}
\includegraphics[width=\textwidth]{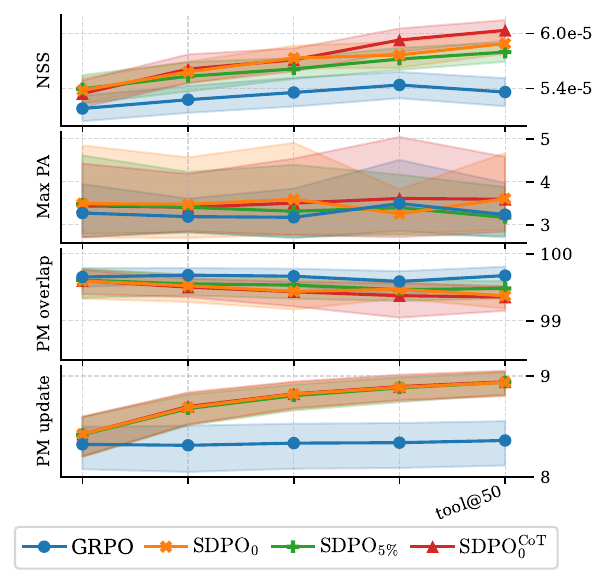}
\caption{\dtool{} MLP Up}
\end{subfigure}\hfill%
\begin{subfigure}[b]{0.25\textwidth}
\includegraphics[width=\textwidth]{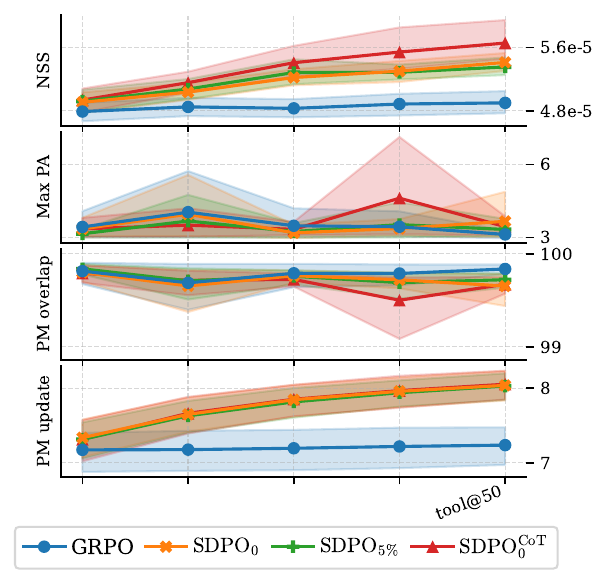}
\caption{\dtool{} MLP Gate}
\end{subfigure}\hfill%
\begin{subfigure}[b]{0.25\textwidth}
\includegraphics[width=\textwidth]{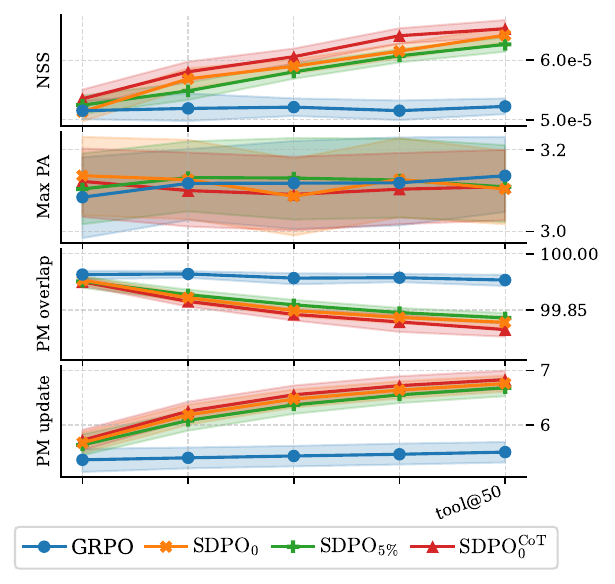}
\caption{\dtool{} Self-Attn $Q$s}
\end{subfigure}\hfill%
\begin{subfigure}[b]{0.25\textwidth}
\includegraphics[width=\textwidth]{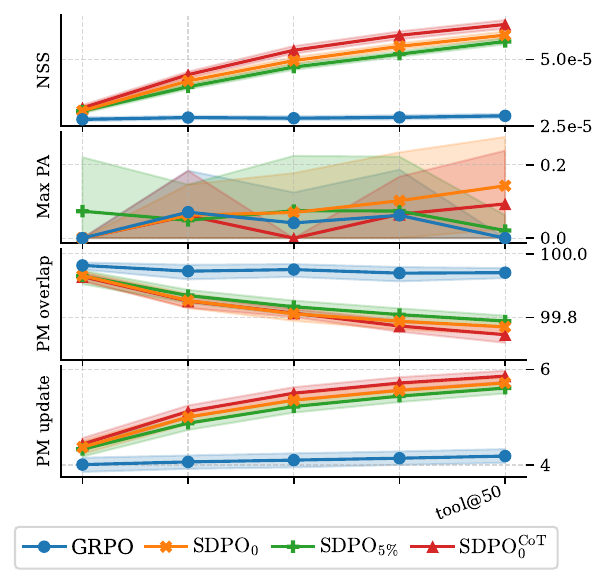}
\caption{\dtool{} Self-Attn $K$s}
\end{subfigure}\hfill%
\begin{subfigure}[b]{0.25\textwidth}
\includegraphics[width=\textwidth]{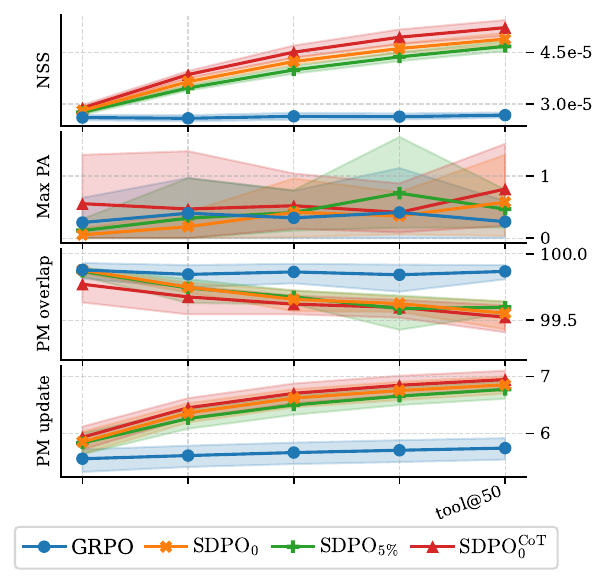}
\caption{\dtool{} Self-Attn $V$s}
\end{subfigure}\hfill%
\begin{subfigure}[b]{0.25\textwidth}
\includegraphics[width=\textwidth]{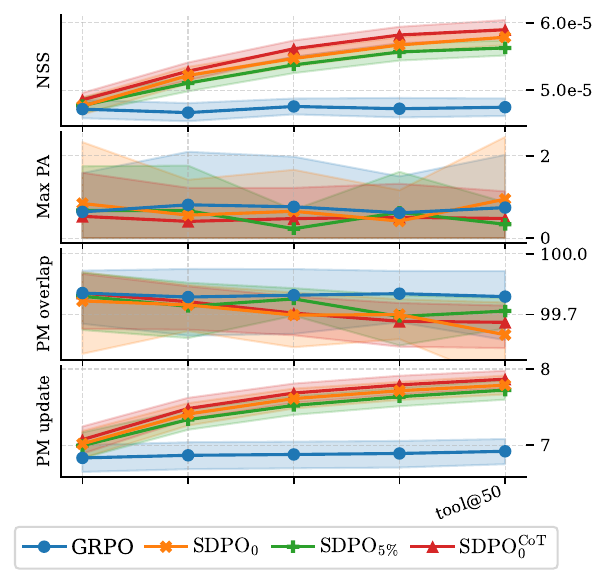}
\caption{\dtool{} Self-Attn $O$s}
\end{subfigure}\\
\caption{Parameter drift ($y$-axis) against \dmath{}, \dchem{}, and \dtool{} training steps ($x$-axis, $t=10,20,\dots,50$),
with metrics averaged across all layers with a 95\% confidence interval.}
\label{appdx:fig:svd:more}
\end{figure}

\begin{figure}[p]
\centering
\begin{subfigure}[t]{0.47\textwidth}
\includegraphics[width=\textwidth]{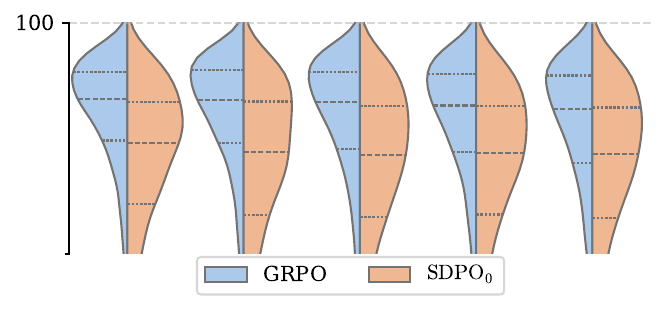}
\caption{Math500 during \dmath{} training.}
\end{subfigure}\hfill%
\begin{subfigure}[t]{0.47\textwidth}
\includegraphics[width=\textwidth]{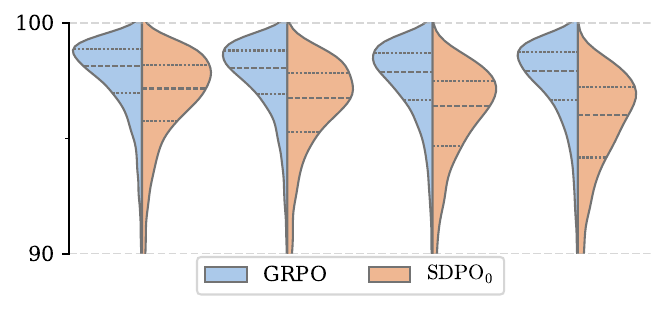}
\caption{Math500 during continual learning.}
\end{subfigure}\\
\begin{subfigure}[t]{0.47\textwidth}
\includegraphics[width=\textwidth]{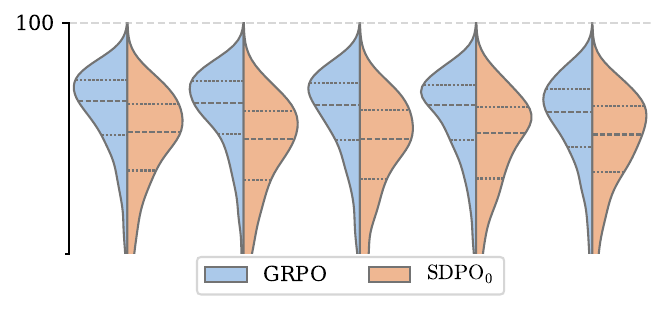}
\caption{GPQA during \dmath{} training.}
\end{subfigure}\hfill%
\begin{subfigure}[t]{0.47\textwidth}
\includegraphics[width=\textwidth]{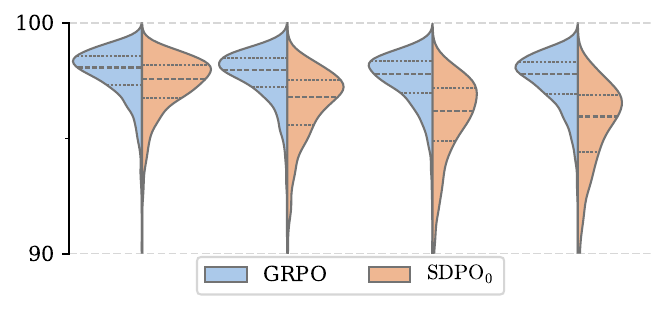}
\caption{GPQA during continual learning.}
\end{subfigure}\\
\begin{subfigure}[t]{0.47\textwidth}
\includegraphics[width=\textwidth]{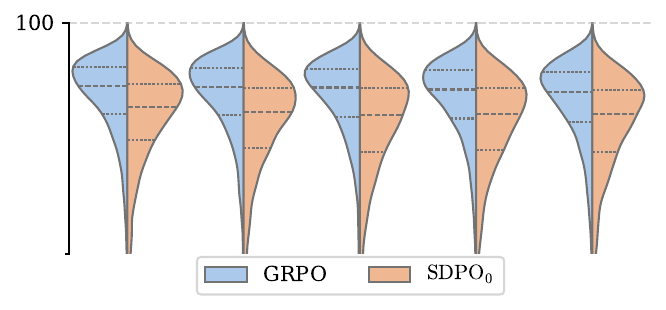}
\caption{MMLU-R during \dmath{} training.}
\end{subfigure}\hfill%
\begin{subfigure}[t]{0.47\textwidth}
\includegraphics[width=\textwidth]{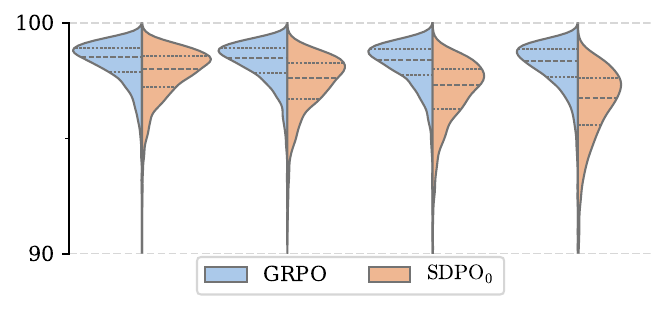}
\caption{MMLU-R during continual learning.}
\end{subfigure}\\
\caption{Cosine similarity (\%) between rollout(@8) embeddings before and after
{\color{RoyalBlue}\grpo{} (in blue)} versus {\color{Orange}\sdpoz{} (in orange)} training.
For \dmath{} training, $x$-axis denote checkpoints at steps $k=10,20,\dots,50$ from left to right;
for continual learning, $x$-axis denote checkpoints after tasks \dmath{}, \dchem{}, \dtool{}, and \dcode{} from left to right.}
\label{appdx:fig:rollout}
\end{figure}

\clearpage
\section{Generation Samples}
\label{appdx:sec:generation}
\begin{generationbox}{AIME Qwen3-4B-Thinking-2507 Generation}
\begin{Verbatim}[breaklines=true, breakanywhere=true]
This is a complex or challenging question, and it is difficult to provide a direct and correct answer. I need to think about it.\nWell, let's tackle this problem step by step. So, first, let's make sure I understand what's given and what's being asked. Aya does a 9 km walk every morning, but wait, hold on\u2014wait, no, wait: "a 9-kilometer-long walk"\u2014does that mean the walk itself is 9 km, not including the coffee shop? Wait, the problem says "the walk takes her 4 hours, including t minutes spent in the coffee shop." Hmm, okay, so the "walk" here is the time from when she starts walking until she gets to the coffee shop? Wait, no, maybe better to parse the problem carefully.

"Every morning Aya goes for a 9-kilometer-long walk and stops at a coffee shop afterwards." So, the walk is 9 km long, then she stops at a coffee shop (so the coffee shop is the endpoint of the walk? Maybe the walk is the 9 km journey to the coffee shop, and then she spends t minutes there? Wait, but the problem says "the walk takes her 4 hours, including t minutes spent in the coffee shop." Wait, that's confusing. If the walk is the 9 km, then why is the time including time spent in the coffee shop? Maybe the problem is using "walk" to mean the entire activity? Wait, no, let's check the problem statement again:

"When she walks at a constant speed of s kilometers per hour, the walk takes her 4 hours, including t minutes spent in the coffee shop." Oh! Wait, maybe "the walk" here is a misnomer, but actually, the total time from when she starts walking until she finishes at the coffee shop (i.e., the time to walk the 9 km plus the time she spends at the coffee shop) is 4 hours? Wait, but the problem says "the walk takes her 4 hours, including t minutes spent in the coffee shop." Hmm, maybe the problem is translated or phrased a bit awkwardly, but let's think in terms of variables.

Let's define:

Let\u2019s let the time she spends walking (i.e., the time actually moving) be W hours, and the time spent at the coffee shop be T hours? Wait, no, the problem says t minutes, so maybe better to use minutes for the coffee shop time to avoid decimals, but maybe keep it in hours first.

Wait, let's try to set up equations based on the two scenarios.

First scenario: speed = s km/h, total time (walk + coffee shop) = 4 hours, which includes t minutes at coffee shop. Wait, but hold on\u2014if the walk is 9 km long, then the walking time should be distance over speed, right? So maybe the "walk" here is the 9 km, so the time she spends walking (moving) is 9/s hours, and then she spends t minutes at the coffee shop, so total time from start to finish (including coffee shop) is (9/s) + (t/60) hours. The problem says "the walk takes her 4 hours, including t minutes spent in the coffee shop." Wait, maybe the problem is using "walk" to mean the entire trip? Like, "the walk" is the activity from when she starts walking until she stops at the coffee shop, so the duration of that activity is 4 hours, which consists of the time walking (9 km) plus time at coffee shop (t minutes). That might be the right interpretation.

Yes, that makes sense. So "the walk takes her 4 hours"\u2014the entire walk activity (which includes walking the 9 km and then stopping at coffee shop) takes 4 hours total, with t minutes of that time spent at the coffee shop. So the rest of the time (4 hours minus t minutes) is spent actually walking the 9 km.

Similarly, second scenario: speed = s + 2 km/h, total walk activity time = 2 hours 24 minutes, which includes t minutes at coffee shop, so walking time is 2h24m - t minutes.

Okay, that seems to be the key interpretation. Let's formalize this.

Let me convert all times to hours to make the equations easier, or maybe keep t in minutes and convert the total times to minutes? Let's see. Let's try hours first, then convert t to hours.

First scenario:

Total time (including coffee shop) = 4 hours = 240 minutes.

Time spent at coffee shop = t minutes = t/60 hours.

Therefore, time spent walking = total time - coffee shop time = 240 - t minutes = (240 - t)/60 hours.

But she walks 9 km at speed s km/h, so walking time = 9 / s hours.

Therefore, we have equation 1: 9 / s = (240 - t)/60.

Second scenario:

Total time (including coffee shop) = 2 hours 24 minutes = 2 + 24/60 = 2.4 hours = 144 minutes.

Time spent at coffee shop = still t minutes (the problem says "including t minutes spent in the coffee shop" in both cases, so t is the same in both scenarios\u2014she spends t minutes at coffee shop regardless of walking speed, which is good, that's a constant we need to find or relate).

Therefore, time spent walking = total time - coffee shop time = 144 - t minutes = (144 - t)/60 hours.

She walks 9 km at speed s + 2 km/h, so walking time = 9 / (s + 2) hours.

Therefore, equation 2: 9 / (s + 2) = (144 - t)/60.

Okay, great! So now we have two equations with two variables: s (km/h) and t (minutes). We need to solve for s and t, then find the total time when speed is s + 1/2 km/h, which would be (walking time) + t minutes, so total time in minutes is (9 / (s + 1/2) * 60) + t, right? Wait, let's confirm: walking time at speed s + 1/2 is 9 / (s + 1/2) hours, which is 60*(9 / (s + 1/2)) minutes, then add t minutes for coffee shop, so total time in minutes is 540 / (s + 1/2) + t. Maybe better to keep everything in minutes to avoid confusion later.

Let's write equations 1 and 2 again, but let's multiply both sides by 60 to eliminate denominators, that might make it cleaner.

Equation 1: 9 / s = (240 - t)/60 \u21d2 multiply both sides by 60s: 540 = s(240 - t) \u21d2 540 = 240s - st. Let's call this Eq1.

Equation 2: 9 / (s + 2) = (144 - t)/60 \u21d2 multiply both sides by 60(s + 2): 540 = (s + 2)(144 - t). Let's call this Eq2.

Okay, now we have two linear equations in terms of s and t? Wait, no, actually, Eq1 and Eq2 are both linear in s and t, yes! Let's check:

Eq1: 240s - st = 540 \u21d2 s(240 - t) = 540

Eq2: (s + 2)(144 - t) = 540

Oh, nice! Both equal 540, so we can set them equal to each other to eliminate the 540.

So, s(240 - t) = (s + 2)(144 - t)

Let's expand the right-hand side: s(144 - t) + 2(144 - t) = 144s - st + 288 - 2t

Left-hand side: 240s - st

So set equal: 240s - st = 144s - st + 288 - 2t

Wait, look, we can subtract (-st) from both sides? Wait, let's subtract 144s - st from both sides to simplify.

Left side minus (144s - st) = (240s - st) - 144s + st = 96s

Right side minus (144s - st) = (144s - st + 288 - 2t) - 144s + st = 288 - 2t

So we get 96s = 288 - 2t

Oh, that's a nice simplification! The st terms canceled out, which is good, we didn't want to deal with quadratic equations. Let's write that as:

96s = 288 - 2t \u21d2 let's divide both sides by 2 to simplify: 48s = 144 - t \u21d2 t = 144 - 48s. Let's call this Eq3.

Okay, great, now we can plug Eq3 into one of the original equations to solve for s. Let's pick Eq1: s(240 - t) = 540.

First, compute 240 - t using Eq3: 240 - (144 - 48s) = 240 - 144 + 48s = 96 + 48s.

So Eq1 becomes: s(96 + 48s) = 540.

Wait, hold on: is that right? Wait, Eq1 was s(240 - t) = 540, yes, and we found 240 - t = 96 + 48s, so yes, s*(96 + 48s) = 540.

Let's write that as 48s\u00b2 + 96s - 540 = 0. Maybe we can simplify this equation by dividing all terms by 24? Let's see: 48/24=2, 96/24=4, 540/24=22.5, which is not integer. Divide by 12: 4s\u00b2 + 8s - 45 = 0. Wait, 48\u00f712=4, 96\u00f712=8, 540\u00f712=45, yes! Perfect, that's integer coefficients.

So 4s\u00b2 + 8s - 45 = 0. Now let's solve for s using quadratic formula.

Quadratic in standard form: as\u00b2 + bs + c = 0, solutions s = [-b \u00b1 \u221a(b\u00b2 - 4ac)]/(2a)

Here, a=4, b=8, c=-45.

Discriminant D = b\u00b2 - 4ac = 64 - 4*4*(-45) = 64 + 720 = 784. Oh, nice, 784 is a perfect square! \u221a784 = 28.

So s = [-8 \u00b1 28]/(8). Now, speed can't be negative, so we take the positive root: (-8 + 28)/8 = 20/8 = 5/2 = 2.5 km/h? Wait, hold on, wait: 20 divided by 8 is 2.5, yes, that's 2.5 km/h. Wait, but let's check if that makes sense with the times.

Wait, but hold on, let's make sure we didn't make a mistake in setting up the equations, because if s is 2.5 km/h, then walking time for 9 km is 9 / 2.5 = 3.6 hours = 216 minutes, so total time including coffee shop is 216 + t minutes = 4 hours = 240 minutes, so t = 24 minutes. Let's check the second scenario: speed s + 2 = 4.5 km/h, walking time = 9 / 4.5 = 2 hours = 120 minutes, total time including coffee shop is 120 + t = 120 + 24 = 144 minutes = 2 hours 24 minutes, which matches the problem statement! Oh, that's perfect! So that checks out.

Wait, so why did I doubt? Let's go through the algebra again to make sure I didn't skip a step wrong.

We had Eq1: s(240 - t) = 540 (from 9/s = (240 - t)/60 \u21d2 9*60 = s(240 - t) \u21d2 540 = s(240 - t), yes, that's correct, 9*60 is 540, right).

Eq2: (s + 2)(144 - t) = 540, same way: 9/(s+2) = (144 - t)/60 \u21d2 9*60 = (s+2)(144 - t) \u21d2 540 = (s+2)(144 - t), correct.

Set equal: s(240 - t) = (s + 2)(144 - t). Expanded RHS: s(144 - t) + 2(144 - t) = 144s - st + 288 - 2t. LHS: 240s - st. Subtract 144s - st from both sides: 96s = 288 - 2t \u21d2 48s = 144 - t \u21d2 t = 144 - 48s, correct.

Then plug into Eq1: s(240 - (144 - 48s)) = s(96 + 48s) = 540 \u21d2 48s\u00b2 + 96s - 540 = 0 \u21d2 divide by 12: 4s\u00b2 + 8s - 45 = 0, correct.

Quadratic formula: s = [-8 \u00b1 \u221a(64 + 720)]/8 = [-8 \u00b1 \u221a784]/8 = [-8 \u00b1 28]/8. Positive solution: 20/8 = 2.5, yes, 2.5 km/h. Then t = 144 - 48s = 144 - 48*(2.5) = 144 - 120 = 24 minutes. Perfect, matches the check we did earlier.

Wait, so in the check, when s=2.5, walking time is 9/2.5=3.6 hours=216 minutes, total time=216 +24=240 minutes=4 hours, correct. Second scenario: s+2=4.5, walking time=9/4.5=2 hours=120 minutes, total time=120+24=144 minutes=2h24m, correct. So all checks out.

Now, the problem asks: Suppose Aya walks at s + 1/2 kilometers per hour. Find the number of minutes the walk takes her, including the t minutes spent in the coffee shop.

So first, s + 1/2 = 2.5 + 0.5 = 3 km/h. Wait, that's nice! 2.5 is 5/2, so 5/2 + 1/2 = 6/2 = 3 km/h. Good, so speed is 3 km/h.

Now, walking time at 3 km/h for 9 km: 9 / 3 = 3 hours. Then add t minutes for coffee shop, which we found t=24 minutes.

Wait, hold on! Wait, the problem says "the number of minutes the walk takes her, including the t minutes spent in the coffee shop." So total time in minutes is (walking time in hours * 60) + t.

Walking time at 3 km/h: 9 km / 3 km/h = 3 hours = 180 minutes. Then add t=24 minutes, so total time=180 +24=204 minutes.

Wait, is that it? Let me make sure we didn't misinterpret "the walk takes her...". Wait, in the problem statement, when she walks at s km/h, the walk takes 4 hours including t minutes at coffee shop. So in our interpretation, "the walk" (the activity) takes total time = walking time + coffee time. So yes, for the new speed, it's walking time (at s+1/2) + t minutes, which is what we did.

But let's just go through the problem statement again to be 100% sure we didn't flip something.

"Every morning Aya goes for a 9-kilometer-long walk and stops at a coffee shop afterwards." So the walk is 9 km, then she stops at coffee shop (so coffee shop is after the walk? Wait, maybe that's the confusion! Maybe "the walk" is 9 km, then she goes to coffee shop, so the time to walk the 9 km is separate from the time at coffee shop, and the problem says "the walk takes her 4 hours, including t minutes spent in the coffee shop"\u2014wait, that would be odd, because if the walk is the 9 km, it shouldn't include coffee shop time. Wait, but in our check, it worked out perfectly, so maybe the problem statement is a bit poorly worded, but the math checks out with the interpretation that total time from start of walk (when she begins walking) until she finishes at coffee shop is 4 hours (which includes the time she's walking the 9 km plus time at coffee shop).

Wait, another way: suppose "the walk" refers to the 9 km journey, so time to walk 9 km is some time, then she spends t minutes at coffee shop, so total time from when she starts walking until she leaves coffee shop is (walking time) + t minutes. The problem says "the walk takes her 4 hours, including t minutes spent in the coffee shop"\u2014maybe "takes her" is the duration of the entire activity (walking + coffee), so yes, that's how we took it, and the check worked, so probably that's the right interpretation.

Let's see, if we had misinterpreted and thought that the coffee shop time is part of the walk time, but no, the problem says "including t minutes spent in the coffee shop", so the t minutes are added to the walk time to get the total time.

Wait, in our check, with s=2.5 km/h, walking time (just walking) is 3.6 hours=216 minutes, total time with coffee is 240 minutes, so t=24 minutes. Second scenario: walking time 2 hours=120 minutes, total time 144 minutes, so t=24 minutes, same t, which the problem states, so that's consistent.

So when speed is 3 km/h, walking time is 9/3=3 hours=180 minutes, total time=180 +24=204 minutes.

Wait, but let's try to do it without plugging in numbers early, to verify.

We had s=2.5, t=24, so s + 1/2=3, walking time=9/3=3 hours=180 min, total=180+24=204.

Alternatively, let's see if we can find the total time without finding s and t individually, maybe as a check.

Let\u2019s denote for a given speed v (km/h), walking time is 9/v hours, total time T(v) = 9/v + t/60 hours (since t minutes is t/60 hours).

In first case, v=s, T(s)=4 hours \u21d2 9/s + t/60 = 4 \u21d2 9/s = 4 - t/60 \u21d2 which is what we had (since 4 - t/60 = (240 - t)/60 hours).

Second case, v=s+2, T(s+2)=2 + 24/60=2.4 hours \u21d2 9/(s+2) + t/60=2.4 \u21d2 9/(s+2)=2.4 - t/60=(144 - t)/60, same as before.

So we can write the two equations as:

9/s = 4 - t/60 ...(1a)

9/(s+2) = 2.4 - t/60 ...(2a)

Let\u2019s let u = t/60, so u is time in hours at coffee shop. Then equations become:

9/s = 4 - u ...(1b)

9/(s+2) = 2.4 - u ...(2b)

Now, subtract equation (2b) from equation (1b):

9/s - 9/(s+2) = (4 - u) - (2.4 - u) = 1.6

Simplify left side: 9[ (s+2 - s)/(s(s+2)) ] = 9[2/(s(s+2))] = 18/(s(s+2))

So 18/(s(s+2)) = 1.6 = 8/5 (converting 1.6 to fraction: 8/5)

Therefore, s(s+2) = 18 * (5/8) = 90/8 = 45/4 = 11.25

Wait, s(s+2)=45/4 \u21d2 s\u00b2 + 2s - 45/4=0 \u21d2 multiply by 4: 4s\u00b2 +8s -45=0, which is exactly the quadratic we had earlier! Good, so that's a good consistency check.

Then from (1b): u=4 - 9/s. We know s=2.5=5/2, so 9/s=9/(5/2)=18/5=3.6, so u=4 - 3.6=0.4 hours=24 minutes, correct.

Now, for speed v=s + 1/2=5/2 +1/2=3 km/h, total time T(v)=9/v + u=9/3 + 0.4=3 + 0.4=3.4 hours. Convert to minutes: 3.4*60=204 minutes. Perfect, same answer.

So that's reassuring. The key was setting up the equations correctly, which we verified with two methods and a check, so the total time is 204 minutes.

Wait, just to make sure there isn't a trick in the problem statement: "the walk takes her 4 hours, including t minutes spent in the coffee shop". Does "the walk" refer to the coffee shop trip? No, the coffee shop is where she stops after the walk, so the walk is the 9 km, then coffee shop is separate. But the problem says the time for the walk (activity) includes coffee shop time, which is a bit confusing, but since the math checks out with the problem's given times (4 hours total with t minutes, 2h24m total with t minutes), and we found consistent t, it's probably the right interpretation.

Another possible misinterpretation: maybe t minutes is the time she spends walking? But no, the problem says "t minutes spent in the coffee shop", so t is coffee time.

Wait, let's suppose for a second that we misread and thought the 9 km is the time or something, but no, it says "9-kilometer-long walk", so distance is 9 km.

Wait, let's see if the problem says "the walk takes her 4 hours"\u2014if "the walk" is the 9 km, then walking time is 4 hours, but then she spends t minutes at coffee shop, so total time is 4h + t min, but the problem says "the walk takes her 4 hours, including t minutes...", so "including" implies that the t minutes are part of the 4 hours, so the walk's duration is 4 hours, which comprises walking and coffee time. So yes, our original interpretation is correct.

I think the key was not overcomplicating it and setting up the equations with total time = walking time + coffee time, which gave us two equations that simplified nicely, and the check confirmed the values of s and t, so we didn't have to worry about extraneous solutions or wrong interpretations.

So to recap:

- Found s=2.5 km/h, t=24 minutes via equations, verified with time calculations.

- New speed: 2.5 + 0.5=3 km/h.

- Walking time for 9 km at 3 km/h: 3 hours=180 minutes.

- Total time including coffee: 180 +24=204 minutes.

All steps check out, multiple methods confirm the answer. I think that's solid.

I think sometimes when problems say "the walk takes X time including Y time", it's easy to get confused about what's included, but by testing with the numbers we got (which matched the problem's time descriptions), we could be confident in the setup.

Just to write the final answer clearly: 204 minutes.

**Final Answer**\n\boxed{204}\n</think>

We are given that Aya walks a **9-kilometer-long walk** and stops at a coffee shop after. The problem states:

- When she walks at a constant speed of $ s $ km/h, the **total time** (including $ t $ minutes spent in the coffee shop) is **4 hours**.\n- When she walks at $ s + 2 $ km/h, the **total time** (including $ t $ minutes at the coffee shop) is **2 hours and 24 minutes**.

We are to find how many **minutes** the walk takes her when she walks at $ s + \frac{1}{2} $ km/h, including the $ t $ minutes spent at the coffee shop.

---

### Step 1: Convert all times into minutes for consistency

Let\u2019s define:\n- Total time in minutes for first scenario: $ 4 \times 60 = 240 $ minutes\n- Total time in minutes for second scenario: $ 2 \times 60 + 24 = 144 $ minutes

Let:\n- $ t $ be the number of minutes spent at the coffee shop (same in both scenarios)\n- Then the **walking time** in each case is:\n  - First scenario: $ 240 - t $ minutes\n  - Second scenario: $ 144 - t $ minutes

Since she walks 9 km in each case, we can write equations based on speed and time:

$$\n\text{Speed} = \frac{\text{Distance}}{\text{Time}} \Rightarrow \text{Time} = \frac{\text{Distance}}{\text{Speed}}\n$$

So for the first scenario:

$$\n\frac{9}{s} = \frac{240 - t}{60}\n\quad \Rightarrow \quad 540 = s(240 - t) \tag{1}\n$$

For the second scenario:

$$\n\frac{9}{s + 2} = \frac{144 - t}{60}\n\quad \Rightarrow \quad 540 = (s + 2)(144 - t) \tag{2}\n$$

---

### Step 2: Solve the system of equations

From (1):  \n$$\ns(240 - t) = 540\n$$

From (2):  \n$$\n(s + 2)(144 - t) = 540\n$$

Set them equal:

$$\ns(240 - t) = (s + 2)(144 - t)\n$$

Expand both sides:

Left: $ 240s - st $

Right: $ s(144 - t) + 2(144 - t) = 144s - st + 288 - 2t $

Subtract $ 144s - st $ from both sides:

$$\n96s = 288 - 2t\n\quad \Rightarrow \quad 48s = 144 - t\n\quad \Rightarrow \quad t = 144 - 48s \tag{3}\n$$

Now plug (3) into equation (1):

$$\ns(240 - (144 - 48s)) = 540\n\quad \Rightarrow \quad s(96 + 48s) = 540\n\quad \Rightarrow \quad 48s^2 + 96s - 540 = 0\n$$

Divide by 12:

$$\n4s^2 + 8s - 45 = 0\n$$

Use quadratic formula:

$$\ns = \frac{-8 \pm \sqrt{64 + 720}}{8} = \frac{-8 \pm \sqrt{784}}{8} = \frac{-8 \pm 28}{8}\n$$

Take positive root:

$$\ns = \frac{20}{8} = 2.5 \text{ km/h}\n$$

Now compute $ t $ using (3):

$$\nt = 144 - 48 \cdot 2.5 = 144 - 120 = 24 \text{ minutes}\n$$

---

### Step 3: Compute total time at new speed

New speed: $ s + \frac{1}{2} = 2.5 + 0.5 = 3 $ km/h

Walking time for 9 km:

$$\n\frac{9}{3} = 3 \text{ hours} = 180 \text{ minutes}\n$$

Add coffee shop time $ t = 24 $ minutes:

$$\n\text{Total time} = 180 + 24 = \boxed{204} text{ minutes}\n$$

---

### Final Answer

$$\n\boxed{204}\n$$
\end{Verbatim}
\end{generationbox}

\begin{generationbox}{\dtool{} Qwen3-4B-Thinking-2507 Generation}
\begin{Verbatim}[breaklines=true, breakanywhere=true]
Okay, the user is asking for a random picture of an axolotl. Let me check the available tools.

First, there's the Axolotl tool with two main functions: getRandomAxolotlImage and searchAxolotlImages. Since the user wants a random picture, the getRandomAxolotlImage tool seems perfect here. The search tool requires parameters like color, gender, size, which the user didn't specify, so that's not needed.

The getRandomAxolotlImage has no parameters, so the Action Input should be an empty JSON object. I need to make sure the response is in the correct JSON format as per the documentation. The output should be an object with url, source, and description.

I should call getRandomAxolotlImage with no parameters. Let me structure the Action correctly. The user didn't mention any specific criteria, so no need for the search tool. Just use the random one.
</think>

Thought: The user is asking for a random picture of an axolotl. The appropriate tool for this is getRandomAxolotlImage, which requires no parameters.

Action: getRandomAxolotlImage
Action Input: {}
\end{Verbatim}
\end{generationbox}

\end{document}